\documentclass[final]{siamart0216}

\usepackage{epsfig}
\usepackage{tikz}
\usepackage{booktabs}
\usepackage{tikz-cd}
\usepackage{graphicx}
\usepackage{makeidx}
\usepackage{multicol}
\usepackage{verbatim}
\usepackage{color}
\usepackage{chngcntr}
\usepackage{crop}
\usepackage{subfigure}
\usepackage{multicol}
\usepackage{algpseudocode}
\usepackage{algorithm}
\usepackage{amssymb}
\usepackage{amsmath}
\usepackage{amscd}
\usepackage{setspace}
\usepackage{enumerate}

\newtheorem{remark}[theorem]{Remark}
\newtheorem{eg}[theorem]{Example}

\DeclareMathOperator*{\rank}{rank}
\DeclareMathOperator*{\card}{card}
\DeclareMathOperator*{\Span}{Span}
\DeclareMathOperator*{\codim}{codim}

\newcommand{\ra}[1]{\renewcommand{\arraystretch}{#1}}

\def\A{\mathcal{A}}
\def\B{\boldsymbol{B}}
\def\C{\boldsymbol{C}}
\def\H{\mathcal{H}}
\def\Q{\boldsymbol{Q}}
\def\S{\mathcal{S}}
\def\V{\mathcal{V}}
\def\I{\mathcal{I}}
\def\J{\mathcal{J}}
\def\X{\mathcal{X}}
\def\Y{\mathcal{Y}}

\def\Z{\mathcal{Z}}

\def\b{\boldsymbol{b}}
\def\c{\boldsymbol{c}}
\def\e{\boldsymbol{e}}
\def\f{\boldsymbol{f}}
\def\h{\boldsymbol{h}}
\def\v{\boldsymbol{v}}
\def\y{\boldsymbol{y}}
\def\x{\boldsymbol{x}}

\def\bzeta{\boldsymbol{\zeta}}

\def\Re{\mathbb{R}}
\def\transpose{\top}

\def\0{\boldsymbol{0}}

\graphicspath{{./figures/}}

\title{Filtrated Algebraic Subspace Clustering\thanks{This work was partially supported by grants NSF 1447822 and 1218709, and ONR N000141310116.} }

\author{Manolis C. Tsakiris \and Ren\'{e} Vidal	
	\thanks{The authors are with the Center for Imaging Science of The Johns Hopkins University. 
		(\email{m.tsakiris,rvidal}@jhu.edu).}}

\begin{document}

\maketitle

\begin{abstract}
Subspace clustering is the problem of clustering data that lie close to a union of linear subspaces. Existing algebraic subspace clustering methods are based on fitting the data with an algebraic variety and decomposing this variety into its constituent subspaces. Such methods are well suited to the case of a known number of subspaces of known and equal dimensions, where a single polynomial vanishing in the variety is sufficient to identify the subspaces. While subspaces of unknown and arbitrary dimensions can be handled using multiple vanishing polynomials, current approaches are not robust to corrupted data due to the difficulty of estimating the number of polynomials. As a consequence, the current practice is to use a single polynomial to fit the data with a union of hyperplanes containing the union of subspaces, an approach that works well only when the dimensions of the subspaces are high enough. In this paper, we propose a new algebraic subspace clustering algorithm, which can identify the subspace $\S$ passing through a point $\x$ by constructing a descending filtration of subspaces containing $\S$. First, a single polynomial vanishing in the variety is identified and used to find a hyperplane containing $\S$. After intersecting this hyperplane with the variety to obtain a sub-variety, a new polynomial vanishing in the sub-variety is found and so on until no non-trivial vanishing polynomial exists. In this case, our algorithm identifies $\S$ as the intersection of the hyperplanes identified thus far. By repeating this procedure for other points, our algorithm eventually identifies all the subspaces. Alternatively, by constructing a filtration at each data point and comparing any two filtrations using a suitable affinity, we propose a spectral version of our algebraic procedure based on spectral clustering, which is suitable for computations with noisy data. We show by experiments on synthetic and real data that the proposed algorithm outperforms state-of-the-art methods on several occasions, thus demonstrating the merit of the idea of filtrations\footnote{Partial results from this paper have been published in the conference paper \cite{Tsakiris:Asilomar14} and in the workshop paper \cite{Tsakiris:FSASCICCV15}.}.
\end{abstract}

\begin{keywords} Generalized Principal Component Analysis, Subspace Clustering, Algebraic Subspace Clustering, Subspace Arrangements, Transversal Subspaces, Spectral Clustering \end{keywords}

\begin{AMS}
	13P05, 13P25, 14M99, 68T10
\end{AMS}

\pagestyle{myheadings}
\thispagestyle{plain}

\section{Introduction}
Given a set of points lying close to a union of linear subspaces, subspace clustering refers to the problem of identifying the number of subspaces, their dimensions, a basis for each subspace, and the clustering of the data points according to their subspace membership. This is an important problem with widespread applications in computer vision \cite{Vidal:CVPR04-multiaffine}, systems theory \cite{Ma-Vidal:HSCC05} and genomics  \cite{Jiang:CAG14}. 


\subsection{Existing work} \label{subsection:Literature}
Over the past $15$ years, various subspace clustering methods have appeared in the literature \cite{Vidal:SPM11-SC}. Early techniques, such as \emph{K-subspaces} \cite{Bradley:JGO00,Tseng:JOTA00} or \emph{Mixtures of Probabilistic PCA} \cite{Tipping-mixtures:99,Gruber-Weiss:CVPR04}, rely on solving a non-convex optimization problem by alternating between assigning points to subspaces and re-estimating a subspace for each group of points. As such, these methods are sensitive to initialization. Moreover, these methods require a-priori knowledge of the number of subspaces and their dimensions. This motivated the development of a family of purely algebraic methods, such as \emph{Generalized Principal Component Analysis} or GPCA \cite{Vidal:PAMI05}, which feature closed form solutions for various subspace configurations, such as hyperplanes \cite{Vidal:CVPR03-gpca,Vidal:CVPR04-gpca}. A little later, ideas from spectral clustering \cite{vonLuxburg:StatComp2007} led to a family of algorithms based on constructing an affinity between pairs of points. Some methods utilize local geometric information to construct the affinities \cite{Yan:ECCV06}. Such methods can estimate the dimension of the subspaces, but cannot handle data near the intersections. Other methods use global geometric information to construct the affinities, such as the \emph{spectral curvature} \cite{Chen:IJCV09}. Such methods can handle intersecting subspaces, but require the subspaces to be low-dimensional and of equal dimensions. In the last five years, methods from sparse representation theory, such as \emph{Sparse Subspace Clustering} \cite{Elhamifar:CVPR09,Elhamifar:ICASSP10,Elhamifar:TPAMI13}, low-rank representation, such as \emph{Low-Rank Subspace Clustering} \cite{Liu:ICML10,Favaro:CVPR11,Liu:TPAMI13,Vidal:PRL14}, and least-squares, such as \emph{Least-Squares-Regression Subspace Clustering} \cite{Lu:ECCV12}, have provided new ways for constructing affinity matrices using convex optimization techniques. Among them, sparse-representation based methods have become extremely attractive because they have been shown to provide affinities with guarantees of correctness as long as the subspaces are sufficiently separated and the data are well distributed inside the subspaces \cite{Elhamifar:TPAMI13,Soltanolkotabi:AS13}. Moreover, they have also been shown to handle noise \cite{Wang-Xu:ICML13} and outliers \cite{Soltanolkotabi:AS14}. However, existing results require the subspace dimensions to be small compared to the dimension of the ambient space. This is in sharp contrast with algebraic methods, which can handle the case of hyperplanes. 

\subsection{Motivation} \label{subsection:Motivation}
This paper is motivated by the highly complementary properties of Sparse Subspace Clustering (SSC) and Algebraic Subspace Clustering (ASC), priorly known as GPCA:\footnote{Following the convention introduced in \cite{Vidal:GPCAbook}, we have taken the liberty to change the name from GPCA to ASC for two reasons. First, to have a consistent naming convention across many subspace clustering algorithms, such as ASC, SSC, LRSC, which is indicative of the its type (algebraic, sparse, low-rank). Second, we believe that GPCA is a more general name that is best suited for the entire family of subspace clustering algorithms, which are all generalizations of PCA.} On the one hand, theoretical results for SSC assume that the subspace dimensions are small compared to the dimension of the ambient space. Furthermore, SSC is known to be very robust in the presence of noise in the data. On the other hand, theoretical results for ASC are valid for subspaces of arbitrary dimensions, with the easiest case being that of hyperplanes, provided that an upper bound on the number of subspaces is known. However, all known implementations of ASC for subspaces of different dimensions, including the recursive algorithm proposed in \cite{Huang:CVPR04-ED}, are very sensitive to noise and are thus considered impractical. As a consequence, our motivation for this work is to develop an algorithm that enjoys the strong theoretical guarantees associated to ASC, but it is also robust to noise.


\subsection{Paper contributions} 
\label{subsection:Contribution}

This paper features two main contributions. 

As a first contribution, we propose a new ASC algorithm, called \emph{Filtrated Algebraic Subspace Clustering} (FASC), which can handle an unknown number of subspaces of possibly high and different dimensions, and give a rigorous proof of its correctness.\footnote{Partial results from the present paper have been presented without proofs in \cite{Tsakiris:Asilomar14}.} Our algorithm solves the following problem:

\begin{definition}[Algebraic subspace clustering problem] \label{dfn:AbstractSC}
Given a finite set of points $\X=\left\{\x_1,\dots,\x_N\right\}$ lying in general position\footnote{We will define formally the notion of points in general position in Definition \ref{dfn:GeneralPosition}.} inside a transversal subspace arrangement\footnote{We will define formally the notion of a transversal subspace arrangement in Definition \ref{dfn:transversal}.} $\A=\bigcup_{i=1}^n \S_i$, decompose $\A$ into its irreducible components, i.e., find the number of subspaces $n$ and a basis for each subspace $\S_i,i=1,\dots,n$. 		
\end{definition}
 
Our algorithm approaches this problem by selecting a suitable polynomial vanishing on the subspace arrangement $\A$. The gradient of this polynomial at a point $\x\in\A$ gives the normal vector to a hyperplane $\V_1$ containing the subspace $\S$ passing through the point. By intersecting the subspace arrangement with the hyperplane, we obtain a  subspace sub-arrangement $\A_1 \subset \A$, which lives in an ambient space $\V_1$ of dimension one less than the original ambient dimension and still contains $\S$. By choosing another suitable polynomial that vanishes on $\A_1$, computing the gradient of this new polynomial at the same point, intersecting again with the new hyperplane $\V_2$, and so on, we obtain a \emph{descending filtration} $\V_1 \supset \V_2 \supset \cdots \supset \S$ of subspace arrangements, which eventually gives us the subspace $\S$ containing the point. This happens precisely after $c$ steps, where $c$ is the codimension of $\S$, when no non-trivial vanishing polynomial exists, and the ambient space $\V_c$, which is the orthogonal complement of the span of all the gradients used in the filtration, can be identified with $\S$. By repeating this procedure at another point not in the first subspace, we can identify the second subspace and so on, until all subspaces have been identified. Using results from algebraic geometry, we rigorously prove that this algorithm correctly identifies the number of subspaces, their dimensions and a basis for each subspace.

As a second contribution, we extend the ideas behind the purely abstract FASC algorithm to a working algorithm called \emph{Filtrated Spectral Algebraic Subspace Clustering} (FSASC), which is suitable for computations with noisy data.\footnote{A preliminary description of this method appeared in a workshop paper \cite{Tsakiris:FSASCICCV15}.} The first modification is that intersections with hyperplanes are replaced by projections onto them. In this way, points in the subspace contained by the hyperplane are preserved by the projection, while other points are generally shrank. The second modification is that we compute a filtration at each data point and use the norm of point $x_j$ at the end of the filtration associated to point $x_i$ to define an affinity between these two points. The intuition is that the filtration associated to point $x_i$ will in theory preserve the norms of all points lying in the same subspace as $x_i$. This process leads to an affinity matrix of high intra-class and low cross-class connectivity, upon which spectral clustering is applied to obtain the clustering of the data. By experiments on real and synthetic data we demonstrate that the idea of filtrations leads to affinity matrices of superior quality, i.e., affinities with high intra- and low inter-cluster connectivity, and as a result to better clustering accuracy. In particular, FSASC is shown to be superior to state-of-the-art methods in the problem of motion segmentation using the Hopkins155 dataset \cite{Tron:CVPR07}.

Finally, we have taken the liberty of presenting in an appendix the foundations of the algebraic geometric theory of subspace arrangements relevant to Algebraic Subspace Clustering, in a manner that is both rigorous and accessible to the interested audience outside the algebraic geometry community, thus complementing existing reviews such as \cite{Ma:SIAM08}.

\subsection{Notation} \label{subsection:Notation}

For any positive integer $n$, we define $[n]:=\left\{1,2,\hdots,n\right\}$. We denote by $\Re$ the real numbers. The right null space of a matrix $\B$ is denoted by $\mathcal{N}(\B)$. If $\S$ is a subspace of $\Re^D$, then $\dim(\S)$ denotes the dimension of $\S$ and $\pi_{\S}: \Re^D \rightarrow \S$ is the orthogonal projection of $\Re^D$ onto $\S$. The symbol $\oplus$ denotes direct sum of subspaces. We denote the 
orthogonal complement of a subspace $\S$ in $\Re^D$ by $\S^\perp$. If $\y_1, \dots, \y_s$ are elements of $\Re^D$, we denote by $\Span(\y_1,\dots,\y_s)$ the subspace of $\Re^D$ spanned by these elements. For two vectors $\x,\y \in \Re^D$, the notation $\x \cong \y$ means that $\x$ and $\y$ are colinear.
We let $\Re[x]=\Re[x_1,\hdots,x_D]$ be the polynomial ring over the real numbers in $D$ indeterminates. We use $x$ to denote the vector of indeterminates $x=(x_1,\dots,x_D)$, while we reserve $\x$ to denote a data point $\x=(\chi_1,\dots,\chi_D)$ of $\Re^D$. We denote by $\Re[x]_\ell$ the set of all homogeneous\footnote{A polynomial in many variables is called homogeneous if all monomials appearing in the polynomial have the same degree.} polynomials of degree $\ell$ and similarly $\Re[x]_{\le \ell}$ the set of all homogeneous polynomials of degree less than or equal to $\ell$. $\Re[x]$ is an infinite dimensional real vector space, while $\Re[x]_\ell$\ and $\Re[x]_{\le \ell}$ are finite dimensional subspaces of $\Re[x]$
of dimensions $\mathcal{M}_{\ell}(D):={ \ell+D-1 \choose \ell}$ and ${\ell + D\choose \ell}$, respectively. We denote by $\Re(x)$ the field of all rational functions over $\Re$ and indeterminates $x_1,\hdots,x_D$.  If $\left\{p_1,\dots,p_s\right\}$ is a subset of $\Re[x]$, we denote by $\langle p_1,\dots,p_s \rangle$ the ideal generated by $p_1,\dots,p_s$ (see Definition \ref{dfn:ideal}). If $\A$ is a subset of $\Re^D$, we denote by $\I_{\A}$ the vanishing ideal of $\A$, i.e., the set of all elements of $\Re[x]$ that vanish on $\A$ and similarly $\I_{\A,\ell} :=\I_{\A} \cap \Re[x]_\ell$ and $\I_{\A,\le \ell} :=\I_{\A} \cap \Re[x]_{\le \ell}$. Finally, for a point $\x \in \Re^D$, and a set $\I \subset \Re[x]$ of polynomials, $\nabla \I|_{\x}$ is the set of gradients of all the elements of $\I$ evaluated at $\x$.


\subsection{Paper organization}
\label{subsection:Organization}
The remainder of the paper is organized as follows: section \ref{section:ASC} provides a careful, yet concise review of the state-of-the-art in algebraic subspace clustering. In section \ref{section:geometricAASC} we discuss the FASC algorithm from a geometric viewpoint with as few technicalities as possible. Throughout Sections \ref{section:ASC} and \ref{section:geometricAASC}, we use a running example of two lines and a plane in $\Re^3$ to illustrate various ideas; the reader is encouraged to follow these illustrations. We save the rigorous treatment of FASC for section \ref{section:mfAASC}, which consists of the technical heart of the paper. In particular, the listing of the FASC algorithm can be found in Algorithm \ref{alg:AASC} and the theorem establishing its correctness is Theorem \ref{thm:AASC}. In section \ref{section:FSASC} we describe FSASC, which is the numerical adaptation of FASC, and compare it to other state-of-the-art subspace clustering algorithms using both synthetic and real data. Finally, appendices \ref{appendix:CA}, \ref{appendix:AG} and \ref{appendix:SA} cover basic notions and results from commutative algebra, algebraic geometry and subspace arrangements respectively, mainly used throughout section \ref{section:mfAASC}.

\section{Review of Algebraic Subspace Clustering (ASC)} \label{section:ASC}

This section reviews the main ideas behind ASC. For the sake of simplicity, we first discuss ASC in the case of hyperplanes (section \ref{subsection:Hyperplanes}) and subspaces of equal dimension (section \ref{subsection:Equidimensional}), for which a closed form solution can be found using a single polynomial. In the case of subspaces of arbitrary dimensions, the picture becomes more involved, but a closed form solution from multiple polynomials is still available when the number of subspaces $n$ is known (section \ref{subsection:Known}) or an  upper bound $m$ for $n$ is known (section \ref{subsection:GeneralCase}). In section \ref{subsection:RASC} we discuss one limitation of ASC due to computational complexity and a partial solution based on a recursive ASC algorithm. In section \ref{subsection:SASC-A} we discuss another limitation of ASC due to sensitivity to noise and a practical solution based on spectral clustering. We conclude in section \ref{subsection:challenges} with the main challenge that this paper aims to address.

\subsection{Subspaces of codimension $1$} \label{subsection:Hyperplanes}

The basic principles of ASC can be introduced more smoothly by considering the case where the union of subspaces is the union of $n$ hyperplanes $\A = \bigcup_{i=1}^n \mathcal{H}_i$ in $\Re^D$. Each hyperplane $\mathcal{H}_i$ is uniquely defined by its unit length normal vector $\b_i \in \Re^D$ as $\H_i = \{ \x\in\Re^D : \b_i^\transpose \x = 0\}$. In the language of algebraic geometry this is equivalent to saying that $\mathcal{H}_i$ is the zero set of the polynomial $\b_i^{\transpose} x$ or equivalently $\mathcal{H}_i$ is the algebraic variety defined by the polynomial equation $\b_i^{\transpose} x=0$, where $\b_i^{\transpose} x= b_{i,1}x_1 + \cdots b_{i,D} x_D$ with $\b_i:=(b_{i,1},\hdots,b_{i,D})^{\transpose}, x := (x_1,\hdots,x_D)^{\transpose}$. We write this more succinctly as $\mathcal{H}_i = \mathcal{Z}(\b_i^{\transpose} x)$. We then observe that a point $\x$ of $\Re^D$ belongs to $\bigcup_{i=1}^n \mathcal{H}_i$ if and only if $\x$ is a root of the polynomial $p(x)=(\b_1^{\transpose}x)\cdots (\b_n^{\transpose}x)$, i.e., the union of hyperplanes $\A$ is the \emph{algebraic variety} $\A=\mathcal{Z}(p)$ (the zero set of $p$). Notice the important fact that $p$ is homogeneous of degree equal to the number $n$ of distinct hyperplanes and moreover it is the product of linear homogeneous polynomials $\b_i^{\transpose}x$, i.e., a product of \emph{linear forms}, each of which defines a distinct hyperplane $\mathcal{H}_i$ via the corresponding normal vector $\b_i$. 

Given a set of points $\mathcal{X}=\left\{\x_j\right\}_{j=1}^N \subset \A$ in general position in the union of hyperplanes, the classic 
\emph{polynomial differentiation} algorithm proposed in \cite{Vidal:CVPR04-gpca,Vidal:PAMI05} recovers the correct number of hyperplanes as well as their normal vectors by
\begin{enumerate}	
\item embedding the data into a higher-dimensional space via a polynomial map, 
\item finding the number of subspaces by analyzing the rank of the embedded data matrix, 
\item finding the polynomial $p$ from the null space of the embedded data matrix, 
\item finding the hyperplane normal vectors from the derivatives of $p$ at a nonsingular point $\x$ of $\A$.\footnote{A nonsingular point of a subspace arrangement is a point that lies in one and only one of the subspaces that constitute the arrangement.}
\end{enumerate}
More specifically, observe that the polynomial $p(x)=(\b_1^{\transpose}x)\cdots (\b_n^{\transpose}x)$ can be written as a linear combination of the set of all monomials of degree $n$ in $D$ variables, $\{x_1^{n}, x_1^{n-1} x_2, x_1^{n-1} x_3 \hdots, x_1 x_D^{n-1},\hdots,x_D^n\}$ as: 
\begin{align}
p(x) = \sum_{n_1+n_2+\cdots n_D = n} c_{n_1,n_2,\dots,n_D}x_1^{n_1}x_2^{n_2}\cdots x_D^{n_D} = \c^\top \nu_n(x).
\end{align}
In the above expression, $\c\in\Re^{M_n(D)}$ is the vector of all coefficients $c_{n_1,n_2,\dots,n_D}$, and $\nu_{n}$ is the \emph{Veronese} or \emph{Polynomial embedding} of degree $n$, as it is known in the algebraic geometry and machine learning literature, respectively. It is defined by taking a point of $\Re^D$ to a point of $\Re^{\mathcal{M}_{n}(D)}$ under the rule 
\begin{align}
(x_1,\hdots,x_D)^\transpose \stackrel{\nu_{n}}{\longmapsto} \left(x_1^{n}, x_1^{n-1} x_2, x_1^{n-1} x_3 \hdots, x_1 x_D^{n-1},\hdots,x_D^{n}\right)^{\transpose} ,
\end{align} 
where $\mathcal{M}_{n}(D)$ is the dimension of the space of homogeneous polynomials of degree $n$ in $D$ indeterminates. 
The image of the data set $\X$ under the Veronese embedding is used to form the so-called \emph{embedded data matrix} 
\begin{align}
\nu_{\ell}(\mathcal{X}):=\begin{bmatrix} \nu_{\ell}(\x_1) & \cdots& \nu_{\ell}(\x_N) \end{bmatrix}^{\transpose}. 
\end{align}
It is shown in \cite{Vidal:PAMI05} that when there are sufficiently many data points that are sufficiently well distributed in the subspaces, the correct number of hyperplanes is the smallest degree $\ell$ for which $\nu_{\ell}(\mathcal{X})$ drops rank by 1: $n = \min_{\ell \geq 1} \{ \ell : \rank ( \nu_{\ell} (\X)) = M_{\ell}(D) - 1\}$. Moreover, it is shown in \cite{Vidal:PAMI05} that the polynomial vector of coefficients $\c$ is the unique up to scale vector in the one-dimensional null space of $\nu_n(\mathcal{X})$.

It follows that the task of identifying the normals to the hyperplanes from $p$ is equivalent to extracting the linear factors of $p$.
This is achieved\footnote{A direct factorization has been shown to be possible as well \cite{Vidal:CVPR03-gpca}; however this approach has not been generalized yet to the case of subspaces of different dimensions.}  by observing that if we have a point $\x \in \mathcal{H}_i - \cup_{i' \neq i} \mathcal{H}_{i'}$, then the gradient $\nabla p|_{\x}$ of $p$ evaluated at $\x$ 
\begin{align}
\nabla p |_{\x} = \sum_{j=1}^n \b_j \prod_{j'\neq j} (\b_{j'}^\transpose \x)
\end{align}
is equal to $\b_i$ up to a scale factor because $\b_i^\top \x = 0$ and hence all the terms in the sum vanish except for the $i^{th}$ (see Proposition \ref{prp:Grd} for a more general statement). Having identified the normal vectors, the task of clustering the points in $\mathcal{X}$ is straightforward.

\subsection{Subspaces of equal dimension} \label{subsection:Equidimensional}
Let us now consider a more general case, where we know that the subspaces are of equal and known dimension $d$. Such a case can be reduced to the case of hyperplanes, by noticing that a union of $n$ subspaces of dimension $d$ of $\Re^{D}$ becomes a union of hyperplanes of $\Re^{d+1}$ after a \emph{generic} projection $\pi_{d}:\Re^{D} \rightarrow \Re^{d+1}$. We note that any random orthogonal projection will almost surely preserve the number of subspaces and their dimensions, as the set of projections $\pi_{d}$ that do not have this preserving property is a zero measure subset of the set of orthogonal projections $\left\{\pi_{d} \in \Re^{(d+1) \times D}: \pi_{d} \pi_{d}^\transpose  = I_{{(d+1)}\times {(d+1)}} \right\}$.
	
When the common dimension $d$ is unknown, it can be estimated exactly by analyzing the right null space of the embedded data matrix, after projecting the data generically onto
subspaces of dimension $d'+1$, with $d'= D-1, D-2, \dots$ \cite{Vidal:PhD03}. More specifically, when $d'>d$, we have
that $\dim \mathcal{N}(\nu_{n}(\pi_{d'}(\X)))>1$, while when $d' < d$ we have $\dim \mathcal{N}(\nu_{n}(\pi_{d'}(\X)))=0$. On the other hand, the case $d'=d$ is the only case for which the null space is one-dimensional, and so $d=\left\{d': \dim \mathcal{N}(\nu_{n}(\pi_{d'}(\X)))=1\right\}$. 

Finally, when both $n$ and $d$ are unknown, one can first recover $d$ as the smallest $d'$ such that there exists an $\ell$ for which $\dim \mathcal{N}(\nu_{\ell}(\pi_{d'}(\X)))>0$, and subsequently recover $n$ as the smallest $\ell$ such that $\dim \mathcal{N}(\nu_{\ell}(\pi_{d}(\X)))>0$; see \cite{Vidal:PhD03} for further details.

\subsection{Known number of subspaces of arbitrary dimensions} \label{subsection:Known}

When the dimensions of the subspaces are unknown and arbitrary, the problem becomes much more complicated, even if the number $n$ of subspaces is known, which is the case examined in this subsection. In such a case, a union of subspaces $\A=\S_1 \cup \cdots \cup \S_n$ of $\Re^D$, henceforth called a \emph{subspace arrangement}, is still an algebraic variety.
The main difference with the case of hyperplanes is that, in general, multiple polynomials of degree $n$ are needed to define $\A$, i.e., $\A$ is the zero set of a finite collection of homogeneous polynomials of degree $n$ in $D$ indeterminates.

\begin{eg}\label{eg:setup}
Consider the union $\A$ of a plane $\S_1$ and two lines $\S_2,\S_3$ in general position in $\Re^3$ (Fig. \ref{fig:TwoLinesOnePlane}).
\begin{figure}[!h]
	\centering
	\begin{tikzpicture}
	\filldraw[fill=lightgray] (-2,0,-2) -- (2,0,-2) -- (2,0,2) node[anchor=west]{$\S_1$} -- (-2,0,2) -- (-2,0,-2);
	\draw[->] (-1.5,0,1.5) -- (-1.5,1,1.5) node[anchor=east]{$\b_1$};
	\draw[->] (0,0,0)  -- (1,0,0);
	\draw[densely dotted] (0,0,0) -- (-1,0,0);
	\draw[->]  (0,0,0) -- (0,0,1) ;
	\draw (0,0,0) -- (0,1.5,-1.5) node[anchor=east]{$\S_2$};
	\draw[densely dotted] (0,0,-0.5) -- (0,0.5,-0.5);
	\draw[densely dotted] (0,0,0) -- (0,0,-1);
	\draw (0,0,0) -- (-1.6,1.6,-0.8) node[anchor=west]{$\S_3$};
	\draw[densely dotted] (0,0,0) -- (-1.6,0,-0.8);
	\draw[densely dotted] (-1,0,-0.5) -- (-1,1,-0.5);
	\end{tikzpicture}
	\caption{A union of two lines and one plane in general position in $\mathbb{R}^3$.}
	\label{fig:TwoLinesOnePlane}
\end{figure}
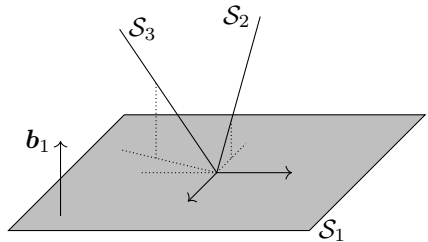
Then 
$\A=\S_1 \cup \S_2 \cup \S_3$ is the zero set of the degree-$3$ homogeneous polynomials 
\begin{align}
p_1 &:= (\b_1^\transpose x) (\b_{2,1}^\transpose x)(\b_{3,1}^\transpose x), &
p_2 &:= (\b_1^\transpose x) (\b_{2,1}^\transpose x)(\b_{3,2}^\transpose x), \\
p_3 &:= (\b_1^\transpose x) (\b_{2,2}^\transpose x)(\b_{3,1}^\transpose x), &
p_4 &:= (\b_1^\transpose x) (\b_{2,2}^\transpose x)(\b_{3,2}^\transpose x),
\end{align} 
where $\b_1$ is the normal vector to the plane $\S_1$ and
$\b_{i,j}, \, j=1,2$, are two linearly independent vectors that are orthogonal
to the line $\S_i, i=2,3$. These polynomials are linearly independent and form a basis
for the vector space $\I_{\A,3}$ of the degree-$3$ homogeneous polynomials that 
vanish on $\A$.\footnote{The interested reader is encouraged to prove this claim.}
\end{eg}

In contrast to the case of hyperplanes, when the subspace dimensions are different, 
there may exist vanishing polynomials of degree strictly less than the number of
subspaces. 

\begin{eg}\label{eg:degree2}
	Consider the setting of Example \ref{eg:setup}. Then there exists a unique up to scale vanishing polynomial of degree $2$, which is the product of two linear forms: one form is $\b_1^\transpose x$, where $\b_1$ is the normal to the plane $\S_1$, and the other linear form is $\f^\transpose x$,
	where $\f$ is the normal to the plane defined by the lines $\S_2$ and $\S_3$ (Fig. \ref{fig:normals-b-f}). 
	\begin{figure}[!h]
		\centering
		\begin{tikzpicture}
		\filldraw[fill=lightgray] (-2,0,-2) -- (2,0,-2) -- (2,0,2) node[anchor=west]{$\S_1$} -- (-2,0,2) -- (-2,0,-2);
		\filldraw[fill=gray] (-1.2,1.2,-0.6) -- (0,1.2,-1.2) -- (0,0,0) -- (-1.2,1.2,-0.6);
		\draw (0.25,1.425,0.5)node[anchor=east]{$\f$}; 	
		\draw (-0.3,1.4,-1)node[anchor=east]{$\H_{23}$}; 	
		\draw[->] (-0.25,0.625,-0.5) -- (0.25,1.625,0.5);
		\draw[->] (0,0,0)  -- (1,0,0);
		\draw[densely dotted] (0,0,0) -- (-1,0,0);
		\draw[->] (-1.5,0,1.5) -- (-1.5,1,1.5) node[anchor=east]{$\b_1$};
		\draw[->]  (0,0,0) -- (0,0,1) ;
		\draw (0,0,0) -- (0,1.5,-1.5) node[anchor=east]{$\S_2$};
		\draw[densely dotted] (0,0,-0.5) -- (0,0.5,-0.5);
		\draw[densely dotted] (0,0,0) -- (0,0,-1);
		\draw (0,0,0) -- (-1.6,1.6,-0.8) node[anchor=west]{$\S_3$};
		\draw[densely dotted] (0,0,0) -- (-1.6,0,-0.8);
		\draw[densely dotted] (-1,0,-0.5) -- (-1,1,-0.5);
		\end{tikzpicture}
		\caption{The geometry of the unique degree-$2$ polynomial 
			$p(x)=(\b_1^\transpose x)(\f^\transpose x)$ that vanishes on 
			$\S_1\cup\S_2\cup \S_3$. $\b_1$ is the normal vector to plane $\S_1$ and $\f$ is the normal vector to the plane $\H_{23}$ spanned by lines $\S_2$ and $\S_3$.}
		\label{fig:normals-b-f}
	\end{figure}
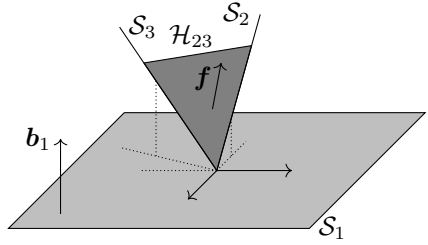
\end{eg}

As Example \ref{eg:setup} shows, all the relevant geometric information is still encoded in the factors
of \emph{some} special basis\footnote{Strictly speaking, this is not always true. However, it is true if the subspace arrangement is general enough, in particular if it is transversal; see Definition \ref{dfn:transversal} and Theorem \ref{thm:I=J}.} of $\I_{\A,n}$, that consists of degree-$n$ homogeneous
polynomials that factorize into the product of linear forms. However,
computing such a basis remains, to the best of our knowledge, an unsolved problem. Instead, one can only rely on computing (or be given) a general basis for the vector space $\I_{\A,n}$. In our example such a basis could be 
\begin{align}
p_1 + p_4, \, \, p_1 - p_4, \, \, p_2 + p_3, \, \, p_2 - p_3\end{align} and it can be seen that none of these polynomials is factorizable
into the product of linear forms. This 
difficulty was not present in the case of hyperplanes, because there was only one 
vanishing polynomial (up to scale) of degree $n$ and it had to be factorizable.

In spite of this difficulty, a solution can still be achieved in an elegant fashion by resorting to polynomial differentiation. The key fact that allows this approach is that any homogeneous polynomial $p$ of degree $n$ that vanishes on the subspace arrangement $\A$ is a  linear combination of vanishing polynomials, each of which is a product of linear forms, with each distinct subspace contributing a vanishing linear form in every product (Theorem \ref{thm:I=J}). As a consequence (Proposition \ref{prp:Grd}), the gradient of $p$ evaluated at some point $\x \in \S_i - \cup_{i' \neq i} \S_{i'}$ lies in $\S_i^{\perp}$ and the linear span of the gradients at $\x$ of all such $p$ is precisely equal to $\S_i^{\perp}$. We can thus recover $\S_i$, remove it from $\A$ and then repeat the procedure to identify all the remaining subspaces.  As stated in Theorem \ref{thm:ASC}, this process is provably correct as long as the subspace arrangement $\A$ is transversal, as defined next.
\begin{definition}[Transversal subspace arrangement \cite{Derksen:JPAA07}]\label{dfn:transversal}
\!\!\! A subspace arrangement $\A = \bigcup_{i=1}^n \S_i \subset \Re^D$ is called transversal, if for any subset $\mathfrak{I}$ of $[n]$, the codimension of $\bigcap_{i \in \mathfrak{I}} \S_i$ is the minimum between $D$ and the sum of the codimensions of all $\S_i, \, i \in \mathfrak{I}$.
\end{definition} 

\begin{remark}Transversality is a geometric condition on the subspaces, which in particular requires the dimensions of all possible intersections among subspaces to be as small as the dimensions of the subspaces allow (see Appendix \ref{appendix:SA} for a discussion).
	\end{remark}

\smallskip
\begin{theorem}[ASC by polynomial differentiation when $n$ is known, \cite{Vidal:PAMI05,Ma:SIAM08}] \label{thm:ASC}
Let $\A = \bigcup_{i=1}^n \mathcal{S}_i$ be a transversal subspace arrangement of $\mathbb{R}^D$, let $\x \in \S_i - \bigcup_{i' \neq i} \S_{i'}$ be a nonsingular point in $\A$, and let $\I_{\A,n}$ be the vector space of all degree-$n$ homogeneous polynomials that vanish on $\A$. Then $\S_i$ is the orthogonal complement of the subspace spanned by all vectors of the form $\nabla p|_{\x}$, where $p \in \I_{\A,n}$, i.e., $\S_i = \Span\left( \nabla \I_{\A,n}|_{\x}\right)^\perp$. 
\end{theorem} 

Theorem \ref{thm:ASC} and its proof are illustrated in the next example.

\begin{eg}\label{eg:ASCtheorem}
	Consider Example \ref{eg:setup} and recall that
	$p_1 = (\b_1^\transpose x) (\b_{2,1}^\transpose x)(\b_{3,1}^\transpose x)$,
	$p_2 = (\b_1^\transpose x) (\b_{2,1}^\transpose x)(\b_{3,2}^\transpose x)$,
	$p_3 = (\b_1^\transpose x) (\b_{2,2}^\transpose x)(\b_{3,1}^\transpose x)$,
	and
	$p_4 = (\b_1^\transpose x) (\b_{2,2}^\transpose x)(\b_{3,2}^\transpose x)$.
	Let $\x_2$ be a generic point in $\S_2 - \S_1 \cup \S_3$. Then
	\begin{align}
	\nabla p_1|_{\x_2} \cong \nabla p_2|_{\x_2} \cong \b_{2,1},\, \, \, \nabla p_3|_{\x_2} \cong \nabla p_4|_{\x_2} \cong \b_{2,2}.
	\end{align} 
	Hence $\b_{2,1},\b_{2,2} \in \Span( \nabla \I_{\A,3}|_{\x_2})$ and so 
	$\S_2 \supset \Span \left( \nabla \I_{\A,3}|_{\x_2}\right)^\perp$. Conversely, let $p \in \I_{\A,3}$. Then there exist $\alpha_i\in\Re,i=1,\dots,4$, such that $p = \sum_{i=1}^{4} \alpha_i p_i$ and so 
	\begin{align}
	\nabla p|_{\x_2} = \sum_{i=1}^{4} \alpha_i \nabla p_i|_{\x_2} \in \Span(\b_{2,1},\b_{2,2})=\S_2^\perp.
	\end{align} Hence $\nabla \I_{\A,3}|_{\x_2} \subset \S_2^\perp$, and so
	$\Span(\nabla \I_{\A,3}|_{\x_2})^\perp \supset \S_2$.
\end{eg}

\subsection{Unknown number of subspaces of arbitrary dimensions} \label{subsection:GeneralCase}

As it turns out, when the number of subspaces $n$ is unknown, but an upper
bound $m \ge n$ is given, one can obtain the decomposition of the subspace
arrangement from the gradients of the vanishing polynomials of degree $m$,
precisely as in Theorem \ref{thm:ASC}, simply by replacing $n$ with $m$.

\begin{theorem}[ASC by polynomial differentiation when an upper bound on $n$ is known, \cite{Vidal:PAMI05,Ma:SIAM08}]  
\label{thm:ubASC}
	Let $\A = \bigcup_{i=1}^n \mathcal{S}_i$ be a transversal subspace arrangement of $\mathbb{R}^D$, let $\x \in \S_i - \bigcup_{i' \neq i} \S_{i'}$ be a nonsingular point in $\A$, and let $\I_{\A,m}$ be the vector space of all degree-$m$ homogeneous polynomials that vanish on $\A$, where $m \ge n$. Then $\S_i$ is the orthogonal complement of the subspace spanned by all vectors of the form $\nabla p|_{\x}$, where $p \in \I_{\A,m}$, i.e., $\S_i = \Span\left( \nabla \I_{\A,m}|_{\x}\right)^\perp$.  
\end{theorem}

\begin{eg}
	Consider the setting of Examples \ref{eg:setup} and \ref{eg:degree2}. Suppose that we have the upper bound $m=4$ on the number of underlying subspaces $(n=3)$. It can be shown that the vector space $\I_{\A,4}$ has\footnote{This can be verified by applying the dimension formula of Corollary 3.4 in \cite{Derksen:JPAA07}.} dimension $8$ and is spanned by the polynomials	 
	\begin{align}
	q_1 &:= (\b_1^\transpose x) (\f^\transpose x)^3,&
	q_5 &:= (\b_1^\transpose x) (\f^\transpose x)(\b_{3}^\transpose x)^2, \\
	q_2 &:= (\b_1^\transpose x)^2 (\f^\transpose x)^2&
	q_6 &:= (\b_1^\transpose x) (\b_2^\transpose x)^2(\f^\transpose x), \\
	q_3 &:= (\b_1^\transpose x)^3 (\f^\transpose x), &
	q_7 &:=(\b_1^\transpose x) (\b_{2}^\transpose x)^2(\b_3^\transpose x), \\
	q_4 &:= (\b_1^\transpose x) (\f^\transpose x)^2(\b_{3}^\transpose x), &
	q_8 &:= (\b_1^\transpose x) (\b_{2}^\transpose x)(\b_{3}^\transpose x)^2,
	\end{align} where $\b_1$ is the normal to $\S_1$, $\f$ is the normal to the plane defined by lines $\S_2$ and $\S_3$, and $\b_i$ is a normal to line $\S_i$ that is linearly independent from $\f$, for $i=2,3$. Hence $\S_1 = \Span(\b_1)^\perp$ and $\S_i=\Span(\f,\b_i)^\perp, i=2,3$. Then for a generic point $\x_2 \in \S_2 - \S_1 \cup \S_3$, we have that 
	\begin{align}
	& \nabla q_1|_{\x_2}=\nabla q_2|_{\x_2}=\nabla q_4|_{\x_2}=\nabla q_6|_{\x_2}=\nabla q_7|_{\x_2}=0,\\ 
	& \nabla q_3|_{\x_2}\cong\nabla q_5|_{\x_2}\cong \f, \, \, \, \nabla q_8|_{\x_2} \cong \b_2.
	\end{align} 
	Hence $\f,\b_2 \in \Span(\nabla \I_{\A,4}|_{\x_2})$ and so 
	$\S_2 \supset \Span(\nabla\I_{\A,4}|_{\x_2})^\perp$. Similarly to 
	Example \ref{eg:ASCtheorem}, since every element of $\I_{\A,4}$ is a 
	linear combination of the $q_\ell,\ell=1,\dots,8$, we have 
	$\S_2 = \Span(\nabla\I_{\A,4}|_{\x_2})^\perp$.	
\end{eg}

\begin{remark}
Notice that both Theorems \ref{thm:ASC} and \ref{thm:ubASC} are statements about the abstract subspace arrangement $\A$, i.e., no finite subset $\X$ of $\A$ is explicitly considered. To pass from $\A$ to $\X$ and get similar Theorems, we need to require $\X$ to be \emph{in general position} in $\A$, in some suitable sense. As one may suspect, this notion of general position must entail that 
polynomials of degree $n$ for Theorem \ref{thm:ASC}, or of degree $m$ for Theorem \ref{thm:ubASC}, that vanish on $\X$ must also vanish on $\A$ and vice versa. In that case, we can compute the required basis for $\I_{\A,n}$, simply by computing a basis for $\I_{\X,n}$, by means of the Veronese embedding described in section \ref{subsection:Hyperplanes}, and similarly for $\I_{\A,m}$. We will make the notion of general position precise in Definition \ref{dfn:GeneralPosition}.	
\end{remark}

\subsection{Computational complexity and recursive ASC} \label{subsection:RASC}
Although Theorem \ref{thm:ubASC} is quite satisfactory from a theoretical point of view, using an upper bound $m\geq n$ for the number of subspaces comes with the practical disadvantage that the dimension of the Veronese embedding, $M_m(D)$, grows exponentially with $m$. In addition, increasing $m$ also increases the number of polynomials in the null space of $\nu_m(\X)$, some which will eventually, as $m$ becomes large, be polynomials that simply fit the data $\X$ but do not vanish on $\A$. To reduce the computational complexity of the polynomial differentiation algorithm, one can consider vanishing polynomials of smaller degree, $m < n$, as suggested by Example \ref{eg:degree2}. While such vanishing polynomials may not be sufficient to cluster the data into $n$ subspaces, they still provide a clustering of the data into $m' \le n$ subspaces. We can then look at each of these $m'$ clusters and see if they can be partitioned further. For instance, in Example \ref{eg:degree2}, we can first cluster the data into two planes, the plane $\S_1$ and the plane $\H_{23}$ containing the two lines $\S_2$ and $\S_3$, and then partition the data lying in $\H_{23}$ into the two lines $\S_2$ and $\S_3$. This leads to the recursive ASC algorithm proposed in \cite{Huang:CVPR04-ED,Vidal:PAMI05}, which is based on finding the polynomials of the smallest possible degree $m$ that vanish on the data, computing the gradients of these vanishing polynomials to cluster the data into $m'\le n$ groups, and then repeating the procedure for each group until the data from each group can be fit by polynomials of degree $1$, in which case each group lies in single linear subspace. While this recursive ASC algorithm is very intuitive, no rigorous proof of its correctness has appeared in the literature. In fact, there are examples where this recursive method provably fails in the sense of producing \emph{ghost subspaces} in the decomposition of $\A$. For instance, when partitioning the data from Example \ref{eg:degree2} into two planes $\S_1$ and $\H_{23}$, we may assign the data from the intersection of the two planes to $\H_{23}$. If this is the case, when trying to partition further the data of $\H_{23}$, we will obtain three lines: $\S_2$, $\S_3$ and the ghost line $\S_4=\S_1 \cap \H_{23}$ (see Fig. \ref{fig:ghost}).


\subsection{Instability in the presence of noise and spectral ASC}\label{subsection:SASC-A}
Another important issue with Theorem \ref{thm:ubASC} from a practical standpoint is its sensitivity to noise. More precisely, when implementing Theorem \ref{thm:ubASC} algorithmically, one is required to estimate the dimension of the null space of $\nu_m(\X)$, which is an extremely challenging problem in the presence of noise. Moreover, small errors in the estimation of $\dim \mathcal{N}(\nu_m(\X))$ have been observed to have dramatic effects in the quality of the clustering, thus rendering algorithms that are directly based on Theorem \ref{thm:ubASC} unstable. While the recursive ASC algorithm of \cite{Huang:CVPR04-ED,Vidal:PAMI05} is more robust than such algorithms, it is still sensitive to noise, as considerable errors may occur in the partitioning process. Moreover, the performance of the recursive algorithm is always subject to degradation due to the potential occurrence of ghost subspaces.

To enhance the robustness of ASC in the presence of noise and obtain a stable working algebraic algorithm, the standard practice has been to apply a variation of the polynomial differentiation algorithm based on spectral clustering \cite{Vidal:PhD03}. More specifically, given noisy data $\X$ lying close to a union of $n$ subspaces $\A$, one computes an approximate vanishing polynomial $p$ whose coefficients are given by the right singular vector of $\nu_n(\X)$ corresponding to its smallest singular value. Given $p$, one computes the gradient of $p$ at each point in $\X$ (which gives a normal vector associated with each point in $\X)$, and builds an affinity matrix between points $\x_j$ and $\x_{j'}$ as the cosine of the angle between their corresponding normal vectors, i.e.,
\begin{align}
\C_{jj',\text{angle}} = \Big | \Big \langle \frac{\nabla p|_{\x_j}}{||\nabla p|_{\x_j}||},  \frac{\nabla p|_{\x_{j'}}}{||\nabla p|_{\x_{j'}}||} \Big \rangle \Big |. \label{eq:ABA}
\end{align} 
This affinity is then used as input to any spectral clustering algorithm (see \cite{vonLuxburg:StatComp2007} for a tutorial on spectral clustering) to obtain a clustering $\X = \bigcup_{i=1}^n\X_i$. We call this Spectral ASC method with \emph{angle-based affinity} as SASC-A.

To gain some intuition about $\C$, suppose that $\A$ is a union of $n$ hyperplanes and that there is no noise in the data. Then $p$ must be of the form $p(x)=(\b_1^\transpose x)\cdots (\b_n^\transpose x)$. In this case $\C_{jj'}$ is simply the cosine of the angle between the normals to the hyperplanes that are associated with points $\x_j$ and $\x_{j'}$. If both points lie in the same hyperplane, their normals must be equal, and hence $\C_{jj'} = 1$. Otherwise, $\C_{jj'} < 1$ is the cosine of the angles between the hyperplanes. Thus, assuming that the smallest angle between any two hyperplanes is sufficiently large and that the points are well distributed on the union of the hyperplanes, applying spectral clustering to the affinity matrix $\C$ will in general yield the correct clustering.
 
Even though SASC-A is much more robust in the presence of noise than purely algebraic methods for the case of a union of hyperplanes, it is fundamentally limited by the fact that, theoretically, it applies only to unions of hyperplanes. Indeed, if the orthogonal complement of a subspace $\S$ has dimension greater than $1$, there may be points $\x, \x' $ inside $\S$ such that the angle between $\nabla p|_{\x}$ and $\nabla p|_{\x'}$ is as large as $90^\circ$. In such instances, points associated to the same subspace may be weakly connected and thus there is no guarantee for the success of spectral clustering.

\subsection{The challenge} \label{subsection:challenges}
As the discussion so far suggests, the state of the art in ASC can be summarized as follows:
\begin{enumerate}
	\item A complete closed form solution to the abstract subspace clustering problem (Problem   \ref{dfn:AbstractSC}) exists and can be found using the polynomial differentiation algorithm implied by Theorem \ref{thm:ubASC}.	
	\item All known algorithmic variants of the polynomial differentiation algorithm are sensitive to noise, especially for subspaces of arbitrary dimensions.
	\item The recursive ASC algorithm described in section \ref{subsection:RASC} does not in 
	general solve the abstract subspace clustering problem (Problem \ref{dfn:AbstractSC}), and is in addition sensitive to noise.
	\item The spectral algebraic algorithm described in section \ref{subsection:SASC-A} is less sensitive to noise, but is theoretically justified only for unions of hyperplanes. 
\end{enumerate}  

The above list reveals the challenge that we will be addressing in the rest of this paper: Develop an ASC algorithm, that solves the abstract subspace clustering problem for perfect data, while at the same time it is robust to noisy data.

\section{Filtrated Algebraic Subspace Clustering - Overview}\label{section:geometricAASC}
This section provides an overview of our proposed \emph{Filtrated Algebraic Subspace Clustering} (FASC) algorithm, which conveys the geometry of the key idea of this paper while keeping technicalities at a minimum. To that end, let us pretend for a moment that we have access to the entire set $\A$, so that we can manipulate it via set operations such as taking its intersection with some other set.
Then the idea behind FASC is to construct a \emph{descending filtration} of the given subspace arrangement $\A \subset \Re^D$, i.e., a sequence of inclusions of subspace arrangements, that starts with
$\A$ and terminates after a finite number of $c$ steps with one of the irreducible components $\S$ of $\A$:\footnote{We will also be using the notation
$\A =: \A_0 \leftarrow \A_1 \leftarrow \A_2 \leftarrow \cdots$,
 where the arrows denote embeddings.}
\begin{align}
\A =: \A_0 \supset \A_1 \supset \A_2 \supset \cdots \supset \A_c=\S.
\end{align} 
The mechanism for generating such a filtration is to construct a strictly descending filtration of
intermediate ambient 
spaces, i.e., 
\begin{align}
\V_0 \supset \V_1 \supset \V_2 \supset \cdots, \label{eq:AmbientSpaces}
\end{align} 
such that $\V_0=\Re^D$, $\dim (\V_{s+1}) = \dim (\V_{s})-1$, and
each $\V_s$ contains the same fixed irreducible component $\S$ of $\A$. 
Then the filtration of
subspace arrangements is obtained by intersecting $\A$ with the filtration of ambient spaces,
i.e., 
\begin{align}
\A_0 := \A \supset \A_1 : = \A \cap \V_1 \supset \A_2 := \A \cap \V_2 \supset \cdots.
\end{align} This can be seen equivalently as constructing a descending filtration of pairs
$(\V_s, \A_s)$, where $\A_s$ is a subspace arrangement of $\V_s$:
\begin{align}
(\Re^{D},\A) \leftarrow (\V_1 \cong \Re^{D-1}, \A_1) \leftarrow (\V_2 \cong \Re^{D-2}, \A_2) \leftarrow \cdots .
\end{align} 

But how can we construct a filtration of ambient spaces \eqref{eq:AmbientSpaces}, that satisfies
the apparently strong condition $\V_s \supset \S, \, \forall s$? The answer lies at the heart of ASC: to construct
$\V_1$ pick a suitable polynomial $p_1$ vanishing on $\A$ and evaluate its gradient at a nonsingular point $\x$ of $\A$. Notice that $\x$ will lie in some irreducible component $\S_{\x}$ of $\A$. Then take $\V_1$ to be the hyperplane of $\Re^D$ defined by the gradient of $p_1$ at $\x$. We know from Proposition \ref{prp:Grd} that $\V_1$ must contain $\S_{\x}$. To construct $\V_2$ we apply essentially the same procedure on the pair $(\V_1,\A_1)$: take a suitable polynomial $p_2$ that vanishes on $\A_1$, but does not vanish on $\V_1$, and take $\V_2$ to be the hyperplane of $\V_1$ defined by $\pi_{\V_1}\left(\nabla p_2|_{\x}\right)$. As we will show in section \ref{section:mfAASC}, it is always the case that $\pi_{\V_1}\left(\nabla p_2|_{\x}\right) \perp \S_{\x}$ and so $\V_2 \supset \S_{\x}$. Now notice, that after precisely $c$ such steps, where $c$ is the codimension of $\S_{\x}$, $\V_c$ will be a $(D-c)$-dimensional linear subspace of $\Re^D$ that by construction contains $\S_{\x}$. But $\S_{\x}$ is also a $(D-c)$-dimensional subspace and the only possibility is that $\V_c = \S_{\x}$. Observe also that this is precisely the step where the filtration naturally terminates, since there is no polynomial that vanishes on $\S_{\x}$ but does not vanish on $\V_c$.  The relations between the intermediate ambient spaces and subspace arrangements are illustrated in the  commutative diagram of \eqref{eq:commutative-diagram}. The filtration in \eqref{eq:commutative-diagram} will yield the irreducible component $\S:=\S_{\x}$ of $\A$ that contains the nonsingular point $\x \in \A$ that we started with. We will be referring to such a point as \emph{the reference point}. We can also take without loss of generality $\S_{\x}=\S_1$. Having identified $\S_1$, we can pick a nonsingular point $\x' \in \A-\S_{\x}$ and construct a filtration of $\A$ as above with reference point $\x'$. Such a filtration will terminate with the irreducible component $\S_{\x'}$ of $\A$ containing $\x'$, which without loss of generality we take to be $\S_2$. Picking a new reference point $\x'' \in \A-\S_{\x}\cup\S_{\x'}$ and so on, we can identify the entire list of irreducible components of $\A$, as described in Algorithm \ref{alg:geometricAASC}.

\begin{equation}	
\begin{tikzcd} \label{eq:commutative-diagram}
	\Re^D \arrow[r,leftarrow,"\cong"]
	& \V_0 \arrow[r,leftarrow] \arrow[d,leftarrow] & \A_0 \arrow[r,leftarrow] \arrow[d,leftarrow] & \S_{\x} \\
		\Re^{D-1} \arrow[r,leftarrow,"\cong"]
		& \V_1 \arrow[r,leftarrow]\arrow[d,leftarrow]  & \A_1 \arrow[r,leftarrow]\arrow[d,leftarrow] & \S_{\x} \\
		\Re^{D-2} \arrow[r,leftarrow,"\cong"]
		& \V_2 \arrow[r,leftarrow]\arrow[d,leftarrow]  & \A_2 \arrow[r,leftarrow]\arrow[d,leftarrow] & \S_{\x} \\
		& \vdots \arrow[d,leftarrow] & \vdots \arrow[d,leftarrow] & 	\\
		\Re^{D-c+1} \arrow[r,leftarrow,"\cong"]
		& \V_{c-1} \arrow[r,leftarrow]\arrow[d,leftarrow]  & \A_{c-1} \arrow[r,leftarrow]\arrow[d,leftarrow] & \S_{\x} \\
		\Re^{D-c} \arrow[r,leftarrow,"\cong"]
		& \V_c \arrow[r,leftarrow,"\cong"]  & \A_c \arrow[r,leftarrow,"\cong"] & \S_{\x}	  
\end{tikzcd}
\end{equation}

\begin{algorithm} \caption{Filtrated Algebraic Subspace Clustering (FASC) - Geometric Version}\label{alg:geometricAASC} \begin{algorithmic}[1] 
		\Procedure{FASC}{$\A$}		
		\State $\mathfrak{L} \gets \emptyset$; $\mathcal{L} \gets \emptyset$;
		\While{$\A - \mathcal{L} \neq \emptyset$}		
		\State pick a nonsingular point $\x$ in $\A - \mathcal{L}$;	
		\State $\V \gets \Re^D$;	
		\While{$\V \cap \A \subsetneq \V$}
		\State find polynomial $p$ that vanishes on $\A\cap \V$ but not on $\V$, s.t. $\nabla p|_{\x} \neq \0$;		
		\State let $\V$ be the orthogonal complement of $\pi_{\V}(\nabla p|_{\x})$ in $\V$;				 			
		\EndWhile		
		\State $\mathfrak{L} \gets \mathfrak{L} \cup 
		\left\{\V\right\}$;	$\mathcal{L} \gets \mathcal{L} \cup \V$;	
		\EndWhile							
		\State \Return $\mathfrak{L}$;
		\EndProcedure 				
	\end{algorithmic} 
\end{algorithm}

\begin{eg}
	Consider the setting of Examples \ref{eg:setup} and \ref{eg:degree2}. 
	Suppose that in the first filtration the algorithm picks as reference point 
	 $\x \in \S_2 - \S_1 \cup \S_3$. Suppose further that the algorithm 
	picks the polynomial $p(x) = (\b_1^\transpose x) (\f^\transpose x)$, which 
	vanishes on $\A$ but certainly not on $\Re^3$. Then the first ambient space
	$\V_1$ of the filtration associated to $\x$ is constructed as $\V_1 = \Span(\nabla p|_{\x})^\perp$.
	Since $\nabla p|_{\x} \cong \f$, this gives that $\V_1$ is precisely the plane of
	$\Re^3$ with normal vector $\f$. Then $\A_1$ is constructed as $\A_1 = \A \cap \V_1$, which consists of the union of three lines $\S_2 \cup \S_3 \cup \S_4$,
	where $\S_4$ is the intersection of $\V_1$ with $\S_1$ (see Figs. \ref{fig:ghost} and \ref{fig:FiltrationStep}).		
	\begin{figure}[t]
		\subfigure[]{\label{fig:ghost}	
			\centering						
			\begin{tikzpicture}[scale=0.8]				
			\filldraw[fill=lightgray] (-2,0,-2) -- (2,0,-2) -- (2,0,2) node[anchor=west]{$\S_1$} -- (-2,0,2) -- (-2,0,-2);
			\filldraw[fill=gray] (-1.2,1.2,-0.6) -- (0,1.2,-1.2) -- (0,0,0) -- (-1.2,1.2,-0.6);
			\draw (0.25,1.425,0.5)node[anchor=east]{$\f$}; 	
			\draw (-0.15,1.45,-1)node[anchor=east]{$\H_{23}$};	
			\draw[->] (-0.25,0.625,-0.5) -- (0.25,1.625,0.5);			
			\draw[->] (-1.5,0,1.5) -- (-1.5,1,1.5) node[anchor=east]{$\b_1$};
			\draw (0,0,0) -- (0,1.5,-1.5) node[anchor=east]{$\S_2$};
			\draw[densely dotted] (0,0,-0.5) -- (0,0.5,-0.5);
			\draw[densely dotted] (0,0,0) -- (0,0,-1);
			\draw (0,0,0) -- (-1.6,1.6,-0.8) node[anchor=west]{$\S_3$};
			\draw[densely dotted] (0,0,0) -- (-1.6,0,-0.8);
			\draw[densely dotted] (-1,0,-0.5) -- (-1,1,-0.5);
			\draw (-1.6,0,0.8) -- (1.2,0,-0.6) node[anchor=west]{$\S_4$};						
			\end{tikzpicture}}			
		\subfigure[]{\label{fig:FiltrationStep}
			\centering
		\begin{tikzpicture}[scale=0.8]		
		\filldraw[fill=gray] (-1.2,1.2,-0.6) -- (0,1.2,-1.2) -- (0,0,0) -- (-1.2,1.2,-0.6);	
		\filldraw[fill=gray] (-1.2,1.2,-0.6) -- (0,-1.2,1.2) -- (0,0,0) -- (-1.2,1.2,-0.6);
		\filldraw[fill=gray] (1.2,-1.2,0.6) -- (0,1.2,-1.2) -- (0,0,0) -- (-1.2,1.2,-0.6);
		\filldraw[fill=gray] (1.2,-1.2,0.6) -- (0,-1.2,1.2) -- (0,0,0) -- (1.2,-1.2,0.6);	
		\draw (0.25,-0.625,0.5) node[anchor=north]{$\V_1^{(1)}$};
		\draw (0,-1.5,1.5) -- (0,1.5,-1.5) node[anchor=east]{$\S_2$};	
		\draw (1.6,-1.6,0.8) -- (-1.6,1.6,-0.8) node[anchor=west]{$\S_3$};	
		\draw (-1.6,0,0.8) -- (1.2,0,-0.6) node[anchor=west]{$\S_4$};		
		\end{tikzpicture}}								
	\subfigure[]{\label{fig:FiltrationStepNormals}
		\begin{tikzpicture}[scale=0.8]	
		\centering	
		\filldraw[fill=gray] (-1.2,1.2,-0.6) -- (0,1.2,-1.2) -- (0,0,0) -- (-1.2,1.2,-0.6);	
		\filldraw[fill=gray] (-1.2,1.2,-0.6) -- (0,-1.2,1.2) -- (0,0,0) -- (-1.2,1.2,-0.6);
		\filldraw[fill=gray] (1.2,-1.2,0.6) -- (0,1.2,-1.2) -- (0,0,0) -- (-1.2,1.2,-0.6);
		\filldraw[fill=gray] (1.2,-1.2,0.6) -- (0,-1.2,1.2) -- (0,0,0) -- (1.2,-1.2,0.6);	
		\draw[->] (0,0.5,-0.5) -- (1,0.25,-0.75) node[anchor=south]{$\b_2$};
		\draw[->] (1.2,-1.2,0.6) -- (1.7,-0.95,0.1) node[anchor=north]{$\b_3$};
		\draw[->] (1,0,-0.5) -- (1.25,-0.625,0) node[anchor=west]{$\b_4$};
		\draw (0.25,-0.625,0.5) node[anchor=north]{$\V_1^{(1)}$};
		\draw (0,-1.5,1.5) -- (0,1.5,-1.5) node[anchor=east]{$\S_2$};	
		\draw (1.6,-1.6,0.8) -- (-1.6,1.6,-0.8) node[anchor=west]{$\S_3$};	
		\draw (-1.6,0,0.8) -- (1.2,0,-0.6) node[anchor=west]{$\S_4$};		
		\end{tikzpicture}}		
		\caption{\protect\subref{fig:ghost}: The plane spanned by lines $\S_2$ and $\S_3$ intersects the plane $\S_1$ at the line $\S_4$. \subref{fig:FiltrationStep}: Intersection of the original subspace arrangement $\A=\S_1 \cup \S_2 \cup \S_3$ with the intermediate ambient space $\V_1^{(1)}$, giving rise to the intermediate subspace arrangement $\A_1^{(1)}=\S_2 \cup \S_3 \cup \S_4$. \subref{fig:FiltrationStepNormals}: Geometry of the unique degree-$3$ polynomial $p(x)=(\b_2^\transpose x)(\b_3^\transpose x)(\b_4^\transpose x)$ that vanishes on $\S_2 \cup \S_3 \cup \S_4$ as a variety of the intermediate ambient space $\V_1^{(1)}$. $\b_i \perp \S_i, i=2,3,4$.}								
	\end{figure}
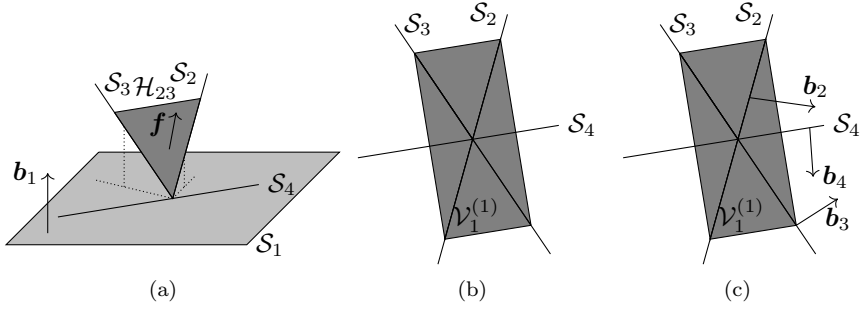
	
	Since $\A_1 \subsetneq \V_1$, the algorithm takes one more step in the filtration. Suppose that the algorithm picks the polynomial $q(x)=(\b_2^\transpose x) (\b_3^\transpose x) (\b_4^\transpose x)$, where $\b_i$ is the unique normal vector of $\V_1$ that is orthogonal to $\S_i$, for $i=2,3,4$ (see Fig \ref{fig:FiltrationStepNormals}). 
	Because of the general position assumption, none of the lines 
	$\S_2, \S_3, \S_4$ is orthogonal to another. Consequently, 
	$\nabla q|_{\x}= (\b_3^\transpose \x)(\b_4^\transpose \x) \b_2 \neq 0$. Moreover, since $\b_2 \in \V_1$, we have that $\pi_{\V_1}\left(\nabla q|_{\x}\right)=\nabla q|_{\x} \cong \b_2$ defines a line in $\V_1$ that must contain $\S_2$. Intersecting $\A_1$ with $\V_2$ we obtain $\A_2 = \A_1 \cap \V_2 = \V_2$ and the filtration terminates with output the irreducible component $\S_{\x} = \S_2 = \V_2$ of $\A$ associated to reference point $\x$.	 		
	
Continuing, the algorithm now picks a new reference point $\x' \in \A - \S_{\x}$, say $\x' \in \S_1$. A similar process as above will identify $\S_1$ as the intermediate ambient space $\V_1=\S_{\x'}$ of the filtration associated to $\x'$ that arises after one step. Then a third reference point will be chosen as $\x'' \in  \A - \S_{\x} \cup \S_{\x'}$ and $\S_3$ will be identified as the intermediate ambient space $\V_2=\S_{\x''}$ of the filtration associated to $\x''$ that arises after two steps. Since the set $\A - \S_{\x} \cup \S_{\x'} \cup \S_{\x''}$ is empty, the algorithm will terminate and return  $\{\S_{\x}, \S_{\x'},\S_{\x''}\}$, which is up to a permutation a decomposition of the original subspace arrangement into its constituent subspaces.	
\end{eg}

Strictly speaking, Algorithm \ref{alg:geometricAASC} is not a valid algorithm in the 
computer-science theoretic sense, since it takes as input an infinite set $\A$, and
it involves operations such as checking equality of the infinite sets $\V$ and $\A \cap \V$. Moreover, the reader may reasonably ask:
\begin{enumerate}
	\item Why is it the case that through the entire filtration associated with reference point $\x$ we can always find polynomials $p$ such that $\nabla p|_{\x} \neq 0$?
	\item Why is it true that even if $\nabla p|_{\x} \neq 0$ then $\pi_{\V}(\nabla p|_{\x}) \neq 0$?
\end{enumerate} We address all issues above and beyond in the next section, which is devoted to rigorously establishing the theory of the FASC algorithm.\footnote{At this point the reader unfamiliar with algebraic geometry is encouraged to read the appendices before proceeding.} 

\section{Filtrated Algebraic Subspace Clustering - Theory} \label{section:mfAASC}
This section formalizes the concepts outlined in section \ref{section:geometricAASC}. section \ref{subsection:NatureInput} formalizes the notion of a set $\X$ being in \emph{general position inside a subspace arrangement $\A$}. Sections \ref{subsection:FirstStepFiltration}-\ref{subsection:MultipleSteps} establish the theory of a single filtration of a finite subset $\X$ lying in general position inside a transversal subspace arrangement $\A$, and culminate with the \emph{Algebraic Descending Filtration (ADF)} algorithm for identifying a single irreducible component of $\A$ (Algorithm \ref{alg:ADF}) and the theorem establishing its correctness (Theorem \ref{thm:ADF}). The ADF algorithm naturally leads us to the core contribution of this paper in section \ref{subsection:All}, which is the FASC algorithm for identifying all irreducible components of $\A$ (Algorithm \ref{alg:AASC}) and the theorem establishing its correctness (Theorem \ref{thm:AASC}). 

\subsection{Data in general position in a subspace arrangement}\label{subsection:NatureInput}
From an algebraic geometric point of view, a union $\A$ of linear subspaces is the same as the set $\I_\A$ of polynomial functions that vanish on $\A$. However, from a computer-science-theoretic point of view, $\A$ and $\I_{\A}$ are quite different: $\A$ is an infinite set and hence it can not be given as input to any algorithm. On the other hand, even though $\I_{\A}$ is also an infinite set, it is generated as an \emph{ideal} by a finite set of polynomials, which can certainly serve as input to an algorithm.That said, from a machine-learning point of view, both $\A$ and $\I_{\A}$ are often unknown, and one is usually given only a finite set of points $\X$ in $\A$, from which we wish to compute its irreducible components $\S_1, \dots, \S_n$.

To lend ourselves the power of the algebraic-geometric machinery, while providing an algorithm of interest to the machine learning and computer science communities,
we adopt the following setting. The input to our algorithm will be the pair $(\X,m)$, where $\X$ is a finite subset of an unknown union of linear subspaces $\A:=\bigcup_{i=1}^n\S_i$ of $\Re^D$, and $m$ is an upper bound on $n$. To make the problem of recovering the decomposition $\A=\bigcup_{i=1}^n \S_i$ from $\X$ well-defined, it is necessary that $\A$ be uniquely identifiable form $\X$. In other words, $\X$ must be in general position inside $\A$, as defined next.
\begin{definition}[Points in general position] \label{dfn:GeneralPosition}
	Let $\X=\left\{\x_1,\dots,\x_N\right\}$ be a finite subset of a subspace arrangement $\A=\S_1 \cup \cdots \cup \S_n$. We say that $\X$ is in general position in $\A$ with respect to degree $m$, if $m \ge n$ and $\A = \Z(\I_{\X,m})$, i.e., if $\A$ is precisely the zero locus of all homogeneous polynomials of degree $m$ that vanish~on~$\X$.
\end{definition}

The intuitive geometric condition
$\A = \Z(\I_{\X,m})$ of Definition \ref{dfn:GeneralPosition} guarantees that there are no \emph{spurious} polynomials of degree less or equal to $m$ that vanish on $\X$.

\begin{proposition} \label{prp:GeneralPositionLowDegrees}
	Let $\X$ be a finite subset of an arrangement $\A$ of n linear subspaces of $\Re^D$. Then $\X$ lies in general position inside  $\A$  with respect to degree $m$ if and only if  
	$\I_{\A,k} = \I_{\X,k}, \, \, \, \forall k \le m$.
\end{proposition}
\begin{proof}	
$(\Rightarrow)$ We first show that $\I_{\A,m} = \I_{\X,m}$. Since $\A \supset \X$, every homogeneous polynomial of degree $m$ that vanishes on $\A$ must vanish on $\X$, i.e., $\I_{\A,m}  \subset \I_{\X,m}$. Conversely, the hypothesis $\A = \Z(\I_{\X,m})$ implies that every polynomial of $\I_{\X,m}$ must vanish on $\A$, i.e., $\I_{\A,m} \supset \I_{\X,m}$. 

Now let $k < m$. As before, since $\A \supset \X$, we must have $\I_{\A,k} \subset \I_{\X,k}$. For the converse direction, suppose for the sake of contradiction that there exists some $p \in \I_{\X,k}$ that does not vanish on $\A$. This means that there must exist an irreducible component of $\A$, say $\S_1$, such that $p$ does not vanish on $\S_1$. Let $\bzeta$ be a vector of $\Re^D$ non-orthogonal to $\S_1$, i.e., the linear form $g(x) = \bzeta^\transpose x$ does not vanish on $\S_1$. Since $p$ vanishes on $\X$ so will the degree $m$ polynomial $g^{m-k} p$, i.e., $g^{m-k} p \in \I_{\X,m}$. But we have already shown that $\I_{\X,m} = \I_{\A,m}$, and so it must be the case that $g^{m-k} p \in \I_{\A,m}$. Since $g^{m-k} p$ vanishes on $\A$, it must vanish on $\S_1$, i.e., $g^{m-k} p \in \I_{\S_1}$. Since by hypothesis $p \not\in \I_{\S_1}$, and since $\I_{\S_1}$ is a prime ideal (see \ref{prp:ISprime}), it must be the case that $g^{m-k} \in \I_{\S_1}$. But again because $\I_{\S_1}$ is a prime ideal, we must have that $g \in \I_{\S_1}$. But this is true if and only if $\bzeta \in \S_1^\perp$, which contradicts the definition of $\bzeta$.
	
	$(\Leftarrow)$ Suppose $\I_{\A,k} = \I_{\X,k}, \, \, \, \forall k \le m$. We will show that $\A = \Z(\I_{\X,m})$. But this is the same as showing that $\A = \Z(\I_{\A,m})$, which is true, by Proposition \ref{prp:Regularity}.
\end{proof}

The next Proposition ensures the existence of points in general position with respect to any degree $m \ge n$.

\begin{proposition}
Let $\A$ be an arrangement of $n$ linear subspaces of $\Re^D$ and let $m$ be any integer $\ge n$. Then there exists a finite subset $\X \subset \A$ that is in general position inside $\A$ with respect to degree $m$.
\end{proposition}
\begin{proof}
By Proposition \ref{prp:Regularity} $\I_{\A}$ is generated by polynomials of degree $\le m$. Then by Theorem 2.9 in \cite{Ma:SIAM08}, there exists a finite set $\X \subset \A$ such that $\I_{\A,k} = \I_{\X,k}, \, \, \, \forall k \le m$, which concludes the proof in view of Proposition \ref{prp:GeneralPositionLowDegrees}.
\end{proof}

 Notice that there is a price to be paid by requiring $\X$ to be in general position, which is that we need the cardinality of $\X$ to be artificially large, especially when $m-n$ is large. In particular, since the dimension of $\I_{\X, m}$ must match the dimension of $\I_{\A,  m}$, the cardinality of $\X$ must be at least $M_m(D) - \dim (\I_{\A, m}) $.

The next result will be useful in the sequel.
\begin{lemma}\label{lem:GeneralPosition}
	Suppose that $\X$ is in general position inside $\A$ with respect to degree $m$. Let $n'<n$. Then the set $\X^{(n')}:=\X-\bigcup_{i=1}^{n'}\X_i$ lies in general position inside the subspace arrangement $\A^{(n')}:=\S_{n'+1} \cup \cdots \cup \S_n$ with respect to degree $m-n'$.
\end{lemma}
\begin{proof}
	We begin by noting that $m-n'$ is an upper bound on the number of subspaces
	of the arrrangement $\A^{(n')}$. According
	to Proposition \ref{prp:GeneralPositionLowDegrees}, it is enough to prove that a homogeneous
	polynomial $p$ of degree less or equal than $m-n'$ vanishes on $\X^{(n')}$ if and only if it vanishes on $\A^{(n')}$. 
	So let $p$ be a homogeneous polynomial of degree less or equal than $m-n'$. If 
	$p$ vanishes on $\A^{(n')}$, then it certainly vanishes on $\X^{(n')}$. It remains to prove the converse. So suppose that $p$ vanishes on
	$\X^{(n')}$. Suppose that for each $i=1,\dots,n'$ we have a vector $\bzeta_i \perp \S_i$, such that $\bzeta_i \not\perp \S_{n'+1},\dots,\S_n$. 
	Next, define the polynomial $r(x)=(\bzeta_1^\transpose x)\cdots (\bzeta_{n'}^\transpose x) p(x)$.
	Then $r$ has degree $\le m$ and vanishes on $\X$. Since $\X$ is in general position inside $\A$, $r$ must vanish on $\A$. For the sake of contradiction suppose that $p$ does not 
	vanish on $\A^{(n')}$. Then $p$ does not vanish say on $\S_{n}$.
	On the other hand $r$ does vanish on $\S_{n}$, hence $r \in \I_{\S_n}$ or equivalently $(\bzeta_1^\transpose x)\cdots (\bzeta_{n'}^\transpose x) p(x) \in \I_{\S_n}$. Since $\I_{\S_n}$ is a prime ideal we must have either $\bzeta_i^\transpose x \in \I_{\S_n}$ for some $i \in [n']$ or 
	$p \in \I_{\S_n}$. Now, the latter can not be true by hypothesis, thus we must 
	have $\bzeta_i^\transpose x \in \I_{\S_n}$ for some $i \in [n']$. But this implies that $\bzeta_i \perp \S_n$, which contradicts the hypothesis on $\bzeta_i$. Hence it must be
	the case that $p$ vanishes on $\A^{(n')}$.
	
	To complete the proof we show that such vectors $\bzeta_i,i=1,\dots,n'$ always exist. It is enough to prove the existence of $\bzeta_1$. If every vector of $\Re^D$ orthogonal to $\S_1$ were orthogonal to, say $\S_{n'+1}$, then we would have that $\S_1^{\perp} \subset \S_{n'+1}^\perp$, or equivalently, $\S_1 \supset \S_{n'+1}$.
\end{proof}

	\begin{remark}
		Notice that the notion of points $\X$ lying in general position inside a subspace arrangement $\A$ is independent of the notion of transversality of $\A$ (Definition \ref{dfn:transversal}). Nevertheless, to facilitate the technical analysis by avoiding degenerate cases of subspace arrangements, in the rest of section \ref{section:mfAASC} we will assume that $\A$ is transversal. For a geometric interpretation of transversality as well as examples, the reader is encouraged to consult Appendix \ref{appendix:SA}.
	\end{remark}

\subsection{Constructing the first step of a filtration} \label{subsection:FirstStepFiltration}
We will now show how to construct the first step of a descending filtration associated with a single irreducible component of $\A$, as in \eqref{eq:commutative-diagram}. Once again, we are given the pair $(\X, m)$, where $\X$ is a finite set in general position inside $\A$ with respect to degree $m$, $\A$ is transversal, and $m$ is an upper bound on the number $n$ of irreducible components of $\A$ (section \ref{subsection:NatureInput}).

To construct the first step of the filtration, we need to find a first hyperplane $\V_{1}$ of $\Re^D$ that contains some irreducible component $\S_i$ of $\A$. According to Proposition \ref{prp:Grd}, it would be enough to have a polynomial $p_1$ that vanishes on the irreducible component $\S_i$ together with a point $\x \in \S_i$. Then
$\nabla p_1|_{\x}$ would be the normal to a hyperplane $\V_{1}$ containing $\S_i$.
Since every polynomial that vanishes on $\A$ necessarily vanishes on $\S_i, \forall i=1,\dots,n$, a reasonable choice is a vanishing polynomial of \emph{minimal degree} $k$, i.e., some $0\neq p_1 \in \I_{\A,k}$, where $k$ is the smallest degree at which $\I_{\A}$ is non-zero. Since $\X$ is assumed in general position in $\A$ with respect to degree $m$, by Proposition \ref{prp:GeneralPositionLowDegrees} we will have $\I_{\A,k} = \I_{\X,k}$, and so our $p_1$ can be computed as an element of the right null space of the embedded data matrix $\nu_k(\X)$. The next Lemma ensures that given any such $p_1$, there is always a point $\x$ in $\X$ such that $\nabla p_1|_{\x} \neq 0$.
 
\begin{lemma} \label{lem:existsx}
Let $0 \neq p_1 \in \I_{\X,k}$ be a vanishing polynomial of minimal degree. Then there exists
$0 \neq \x \in \X $ 
such that $\nabla p_1 |_{\x}  \neq 0$, and moreover, without loss of generality $\x \in \S_1 - \bigcup_{i > 1} \S_{i}$. 
\end{lemma}

\begin{proof}
We first establish the existence of a point $\x \in \X$ such that $\nabla p_1|_{\x} \neq \0$. For the sake of contradiction, suppose that no such $\x \in \X$ exists. Since $0 \neq p_1 \in \I_{\X,k}$, $p_1$ can not be a constant polynomial, and so there exists some $j \in [D]$ such that the degree $k-1$ polynomial $\frac{\partial p_1}{\partial x_{j}}$ is not the zero polynomial. Now, by hypothesis $\nabla p_1\big|_{\x}=\0, \, \forall \x \in \X$, hence $\frac{\partial p_1}{\partial x_j}\big|_{\x}=0, \, \forall \x \in \X$. But then, $0 \neq \frac{\partial p_1}{\partial x_j} \in \I_{\X,k-1}$ and this would contradict the hypothesis that $k$ is the smallest index such that $\I_{\X,k} \neq 0$. Hence there exists $\x \in \X$ such that $\nabla p_1|_{\x} \neq \0$. To show that $\x$ can be chosen to be non-zero, note that if $k=1$, then $\nabla p_1$ is a constant vector and we can take $\x$
to be any non-zero element of $\X$. If $k>1$ then $\nabla p_1|_{\0}=\0$ and so $\x$ must necessarily be different from zero. 

Next, we establish that $\x \in \S_1 - \bigcup_{i > 1} \S_{i}$.
Without loss of generality we can assume that $\x \in \X_1:=\X \cap \S_1$. For the sake of contradiction, suppose that $\x \in \S_{1} \cap \S_{i}$ for some $i > 1$. Since $\x \neq \0$, there is some index $j \in [D]$ such that the $j^{th}$ coordinate of $\x$, denoted by $\chi_j$, is different from zero. Define $g(x) := x_j^{n-k} p_1(x)$. Then $g \in \I_{\X,n}$ and by the general position assumption we also have that $g \in \I_{\A,n}$. Since $\A$ is assumed transversal, by Theorem \ref{thm:I=J}, $g$ can be written in the form 
\begin{align}
g = \sum_{r_i \in [c_i], \,  i \in [n]} c_{r_1,\dots,r_n} l_{r_1,1} \cdots l_{r_n,n}, \label{eq:p_m}
\end{align} 
where $c_{r_1,\dots,r_n} \in \Re$ is a scalar coefficient, $l_{r_i,i}$ is a linear form vanishing on $\S_i$, and the summation runs over all multi-indices $(r_1,\dots,r_n) \in [c_1]\times \cdots \times [c_n]$. Then evaluating the gradient of the expression on the right of \eqref{eq:p_m} at $\x$, and using the hypothesis that $\x \in \S_{1} \cap \S_{i}$ for some $i>1$, we see that $\nabla g|_{\x}=\0$. However, evaluating the gradient of $g$ at $\x$ from the formula $g(x) := x_j^{n-\ell} p_1(x)$, we get $\nabla g|_{\x}  = \chi_j^{n-k} \nabla p_1|_{\x} \neq \0$. This contradiction implies that the hypothesis $\x \in \S_1 \cap \S_i$ for some $i>1$ can not be true, i.e., $\x$ lies only in the irreducible component $\S_1$.
\end{proof}

Using the notation established so far and setting $\b_1=\nabla p_1|_{\x}$, the hyperplane of $\Re^D$ given by $\V_1 = \Span(\b_1)^\perp=\mathcal{Z}(\b_1^\transpose x)$ contains the irreducible component of $\A$ associated with the reference point $\x$, i.e., $\V_1 \supset \S_1$. Then we can define a subspace sub-arrangement $\A_{1}$ of $\A$ by
\begin{align}
\A_{1} := \A \cap \V_{1} = \S_1 \cup (\S_2 \cap \V_{1}) \cup \cdots \cup  (\S_n \cap \V_{1}). 
\end{align} Observe that $\A_{1}$ can be viewed as a subspace arrangement of $\V_{1}$, since $\A_{1} \subset \V_{1}$ (see also the commutative diagram of eq. \eqref{eq:commutative-diagram}). Certainly, our algorithm can not manipulate directly the infinite sets $\A$ and $\V_1$. Nevertheless, these sets are algebraic varieties and as a consequence we can perform their intersection in the algebraic domain. That is, we can obtain a set of polynomials defining $\A \cap \V_1$, as shown next.\footnote{Lemma \ref{lem:GeneratorsIntersection} is a special case of Proposition \ref{prp:VarietiesIntersection}.}

\begin{lemma} \label{lem:GeneratorsIntersection}
	$\A_1:=\A \cap \V_1$ is the zero set of the ideal generated by $\I_{\X, m}$ and $\b_1^\transpose x$, i.e., 
	\begin{align}
	\A_1=\mathcal{Z} \left(\mathfrak{a}_1 \right), \, \, \, \mathfrak{a}_1:=\langle \I_{\X, m}\rangle+\langle \b_1^\transpose x \rangle.
	\end{align}
\end{lemma}
\begin{proof}
	$(\Rightarrow):$ We will show that $\A_1 \subset \mathcal{Z} \left(\mathfrak{a}_1 \right)$. Let $w$ be a polynomial of $\mathfrak{a}_1$. Then by definition of $\mathfrak{a}_1$, $w$ can be written as $w=w_1+w_2$, where $w_1 \in \langle\I_{\X, m}\rangle$ and $w_2 \in \langle\b_1^\transpose x \rangle$. Now take any point $\y \in \A_1$. Since $\y \in \A$, and $\I_{\X,m} = \I_{\A,m}$, we must have $w_1(\y)=0$. Since $\y \in \V_1$, we must have that $w_2(\y)=0$. Hence $w(\y)=0$, i.e., every point of $\A_1$ is inside the zero set of $\mathfrak{a}_1$.	
	$(\Leftarrow):$ We will show that $\A_1 \supset \mathcal{Z} \left(\mathfrak{a}_1 \right)$. Let $\y \in \mathcal{Z} \left(\mathfrak{a}_1 \right)$, i.e., every element of $\mathfrak{a}_1$ vanishes on $\y$. Hence every element of $\I_{\X, m}$ vanishes on $\y$, i.e., $\y \in \Z(\I_{\X, m})=\A$. In addition, every element of $\langle\b_1^\transpose x \rangle$ vanishes on $\y$, in particular $\b_1^\transpose \y=0$, i.e., $\y \in \V_1$.
\end{proof}

In summary, the computation of the vector $\b_1 \perp \S_1$ completes algebraically the first step of the filtration, which gives us the hyperplane $\V_1$ and the sub-variety $\A_1$. Then, there are two possibilities: $\A_1= \V_1$ or $\A_1 \subsetneq \V_1$. In the first case, we need to terminate the filtration, as explained in section \ref{subsection:SecondStep}, while in the second case we need to take one more step in the filtration, as explained in section \ref{subsection:MultipleSteps}.

\subsection{Deciding whether to take a second step in a filtration} \label{subsection:SecondStep}

If $\A_1 = \V_1$, we should terminate the filtration because in this case $\V_1 = \S_1$, as Lemma \ref{lem:V=A-geometry} shows, and so we have already identified one of the subspaces. Lemma \ref{lem:V=A-algebra} will give us an algebraic procedure for checking if the condition $\A_1 = \V_1$ holds true, while Lemma~\ref{lem:SinglePolynomialCheck} will give us a computationally more friendly procedure for checking the same condition.
\begin{lemma} \label{lem:V=A-geometry}
	$\V_1 = \A_1$ if and only if $\V_1 = \S_1$.
\end{lemma}	
\begin{proof}
	$(\Rightarrow):$ Suppose $\V_1 = \A_1 \doteq \S_1 \cup (\S_2 \cap \V_{1}) \cup \cdots \cup  (\S_n \cap \V_{1})$. 
	Taking the vanishing-ideal operator on both sides, we obtain
	\begin{align}	\label{eq:IA1}
	\I_{\V_1} = \I_{\S_1} \cap \I_{\S_2 \cap \V_{1}} \cap \cdots \cap \I_{\S_n \cap \V_{1}}.
	\end{align} Since $\V_1$ is a linear subspace, $\I_{\V_1}$ is a prime ideal by Proposition \ref{prp:ISprime}, and so by Proposition \ref{prp:ideals-intersection} $\I_{\V_1}$ must contain one of the ideals
	$\I_{\S_1},\I_{\S_2 \cap \V_1},\dots,\I_{\S_n \cap \V_1}$. Suppose that $\I_{\V_1} \supset \I_{\S_i \cap \V_1}$ for some $i >1$. Taking the zero-set operator on both sides, and using Proposition \ref{prp:closure} and the fact that linear subspaces are closed in the Zariski topology, we obtain $\V_1 \subset \S_i \cap \V_1$, which implies that $\V_1 \subset \S_i$. Since $\S_1 \subset \V_1$, we must have that $\S_1 \subset \S_i$, which contradicts the assumption of transversality on $\A$. Hence it must be the case that $\I_{\V_1} \supset \I_{\S_1}$. Taking the zero-set operator on both sides we get $\V_1 \subset \S_1$, which implies that $\V_1 = \S_1$, since $\S_1 \subset \V_1$.
	$(\Leftarrow):$ Suppose $\V_1 = \S_1$. Then $\V_1=\S_1 \subset \A_1 \subset \V_1 = \S_1$ and so $\A_1 = \V_1$.
\end{proof}	

Knowing that a filtration terminates if $\A_1 = \V_1$, we need a mechanism for checking this condition. The next lemma shows how this can be done in the algebraic domain.
 
\begin{lemma} \label{lem:V=A-algebra}
$\V_1 = \A_1$ if and only if $\I_{\X, m} \subset \langle\b_1^\transpose x \rangle_{ m}$.
\end{lemma} 

\begin{proof}
	$(\Rightarrow):$ Suppose $\A_1=\V_1$. Then $\A \supset \V_1$ and by taking vanishing ideals on both sides we get $\I_{\A} \subset \I_{\V_1}=\langle\b_1^\transpose x \rangle$. Since $\I_{\X, m} = \I_{\A, m} \subset \I_{\A}$, it follows that $\I_{\X, m} \subset \langle\b_1^\transpose x \rangle_m$.
	$(\Leftarrow):$ Suppose $\I_{\X, m} \subset \langle\b_1^\transpose x \rangle_{ m}$ and for the sake of contradiction suppose that $\A_1 \subsetneq \V_1$. In particular, from Lemma \ref{lem:V=A-geometry} we have that $\S_1 \subsetneq \V_1$. Hence, there exists a vector $\bzeta_1$ linearly independent from $\b_1$ such that $\bzeta_1 \perp \S_1$. Now for any $i > 1$, there exists $\bzeta_i$ linearly independent from $\b_1$ such that $\bzeta_i \perp \S_i$. For if not, then $\I_{\S_i} \subset \I_{\V_1}$ and so $\S_i \supset \V_1$, which leads to the contradiction $\S_i \supset \S_1$. Then the polynomial $(\bzeta_1^\transpose x) \cdots (\bzeta_n^\transpose x)$ is an element of $\I_{\A,n}=\I_{\X,n}$ and by the hypothesis that $\I_{\X, m} \subset \langle\b_1^\transpose x \rangle_m$ we must have that
	$(\bzeta_1^\transpose x)^{m-n+1} \cdots (\bzeta_n^\transpose x) \in \langle\b_1^\transpose x\rangle$. But $\langle\b_1^\transpose x\rangle$ is a prime ideal and so one of the factors of $(\bzeta_1^\transpose x) \cdots (\bzeta_n^\transpose x)$ must lie in $\langle\b_1^\transpose x\rangle$. So suppose $\bzeta_j^\transpose x \in \langle\b_1^\transpose x\rangle$, for some $j \in [n]$. This implies that there must exist a polynomial $h$ such that
	$\bzeta_j^\transpose x = h \, (\b_1^\transpose x)$. By degree considerations, we conclude that $h$ must be a constant, in which case the above equality implies $\bzeta_j \cong \b_1$. But this is a contradiction on the definition of $\bzeta_j$. Hence it can not be the case that $\A_1 \subsetneq \V_1$.
\end{proof}

Notice that checking the condition $\I_{\X, m} \subset \langle\b_1^\transpose x \rangle_{ m}$ in Lemma \ref{lem:V=A-algebra}, requires computing a basis of $\I_{\X,m}$ and checking whether each element of the basis is divisible by the linear form $\b_1^\transpose \x$. Equivalently, to check the inclusion of finite dimensional vector spaces $\I_{\X, m} \subset \langle\b_1^\transpose x \rangle_{ m}$ we need to compute a basis $\B_{\X, m}$ of $\I_{\X,m}$ as well as a basis $\B$ of $\langle\b_1^\transpose x \rangle_{ m}$ and check whether the rank equality $\rank([\B_{\X,m} \, \, \B]) = \rank(\B)$ holds true. Note that a basis of $\langle\b_1^\transpose x \rangle_{ m}$ can be obtained in a straightforward manner by multiplying all monomials of degree $m-1$ with the linear form $\b_1^\transpose x$. On the other hand, computing a basis of $\I_{\X,m}$ by computing a basis for the right nullspace of $\nu_m(\X)$ can be computationally expensive, particularly when $m$ is large. If however, the points $\X \cap \S_1$ are in general position in $\S_1$ with respect to degree $m$, then checking the condition $\I_{\X,m} \subset \langle\b_1^\transpose x \rangle_{ m}$ can be done more efficiently, as we now explain. Let $\boldsymbol{V}_1=[\v_1,\dots,\v_{D-1}]$ be a basis for the vector space $\V_1$. Then $\V_1$ is isomorphic to $\Re^{D-1}$ under the linear map $\sigma_{\boldsymbol{V}_1}: \V_1 \rightarrow \Re^{D-1}$ that takes a vector $\v = \alpha_1 \v_1 + \cdots + \alpha_{D-1} \v_{D-1}$ to its coordinate representation $(\alpha_1,\dots,\alpha_{D-1})^\transpose$. Then the next result says that checking the condition $\V_1 = \A_1$ is equivalent to checking the rank-deficiency of the embedded data matrix $\nu_{m}(\sigma_{\boldsymbol{V}_1}(\X \cap \V_1))$, which is computationally a simpler task than computing the right nullspace of $\nu_m(\X)$.
\begin{lemma} \label{lem:SinglePolynomialCheck}
	Suppose that $\X_1$ is in general position inside $\S_1$ with respect to degree $m$. Then 
	$\V_1 = \A_1$ if and only if the embedded data matrix $\nu_{m}(\sigma_{\boldsymbol{V}_1}(\X \cap \V_1))$ is full rank.
\end{lemma}
\begin{proof}
	The statement is equivalent to the statement ``$\V_1=\A_1$ if and only if
	$\I_{\X \cap \V_1, m}=\langle \b_1^\transpose x\rangle_{m}$", which we now prove.
	$(\Rightarrow):$ Suppose $\V_1 = \A_1$. Then by Lemma \ref{lem:V=A-geometry} $\V_1=\S_1$, which implies that $\I_{\S_1}=\langle \b_1^\transpose x \rangle$. This in turn implies that $\I_{\S_1, m}=\langle \b_1^\transpose x\rangle_{ m}$. Now
	$\I_{\X \cap \V_1, m} = \I_{\X \cap \S_1, m} = \I_{\X_1, m}$. By the general position hypothesis on $\X_1$ we have $\I_{\S_1, m} = \I_{\X_1, m}$. Hence $\I_{\X \cap \V_1, m} = \langle \b_1^\transpose x\rangle_{ m}$. 
	$(\Leftarrow):$ Suppose that $\I_{\X \cap \V_1, m} = \langle \b_1^\transpose x\rangle_{ m}$.
	For the sake of contradiction, suppose that $\A_1 \subsetneq \V_1$. Since $\A_1$ is an arrangement of at most $m$ subspaces, there exists a homogeneous polynomial $p$ of degree at most $m$ that vanishes on $\A_1$ but does not vanish on $\V_1$. Since $\X \cap \V_1 \subset \A_1$, $p$ will vanish on $\X \cap \V_1$, i.e., $p \in \I_{\X \cap \V_1, m}$ or equivalently $p \in \langle \b_1^\transpose x\rangle_{ m}$ by hypothesis. But then $p$ vanishes on $\V_1$, which is a contradiction; hence it must be the case that $\V_1 = \A_1$.	 
\end{proof}

\subsection{Taking multiple steps in a filtration and terminating}\label{subsection:MultipleSteps}
If $\A_1 \subsetneq \V_1$, 
then it follows from Lemma \ref{lem:V=A-geometry} that $\S_1 \subsetneq \V_1$. Therefore, subspace $\S_1$ has not yet been identified in the first step of the filtration and we should take a second step. As before, we can start constructing the second step of our filtration by choosing a suitable vanishing polynomial $p_2$, such that its gradient at the reference point $\x$ is not colinear with $\b_1$. The next Lemma shows that such a $p_2$ always exists.

\begin{lemma}  \label{lem:exists-q}
$\X$ admits a homogeneous vanishing polynomial $p_2$ of degree $\ell \le n$, such that $p_2 \not\in \I_{\V_1}$ and $\nabla p_2|_{\x} \not\in \Span(\b_{1})$.
\end{lemma}
\begin{proof} Since $\A_1 \subsetneq \V_1$, Lemma \ref{lem:V=A-geometry} implies that $\S_1 \subsetneq \V_1$. Then there exists a vector $\bzeta_1$ that is orthogonal to $\S_1$ and is linearly independent from $\b_1$. Since $\x \in \S_1 - \bigcup_{i>1}\S_i$, for each $i>1$ we can find a vector $\bzeta_i$ such that $\bzeta_i \not\perp \x$ and $\bzeta_i \perp \S_i$. Notice that the pairs $\b_1,\bzeta_i$ are linearly independent for $i>1$, since $\b_1 \perp \x$ but $\bzeta_i \not\perp \x$. Now, the polynomial $p_2:=(\bzeta_1^\transpose x)\cdots (\bzeta_n^\transpose x)$ has degree $n$ and vanishes on $\A$, hence $p_2 \in \I_{\X,\le m}$. Moreover, 
	$\nabla  p_2|_{\x} =
	(\bzeta_2^\transpose \x)\cdots (\bzeta_n^\transpose \x) \bzeta_1 \neq 0$, since by hypothesis $\bzeta_i^\transpose \x \neq 0, \forall i>1$. Since $\bzeta_1$ is linearly independent from $\b_1$, we have $\nabla p_2|_{\x} \not\in \Span(\b_1)$. Finally, $p_2$ does not vanish on $\V_1$, by a similar argument to the one used in the proof of Lemma \ref{lem:V=A-algebra}.
\end{proof}

\begin{remark}
Note that if $\ell$ is the degree of $p_2$ as in Lemma \ref{lem:exists-q}, and if $q_1,\dots,q_s$ is a basis for $\I_{\X,\ell}$, then at least one of the $q_i$ satisfies the conditions of the Lemma. This is important algorithmically, because it implies that the search for our $p_2$ can be done sequentially. We can start by first computing a minimal-degree polynomial in $\I_{\A,k}$, and see if it satisfies our requirements. If not, then we can compute a second linearly independent polynomial and check again. We can continue in that fashion until we have computed a full basis for $\I_{\X,k}$. If no suitable polynomial has been found, we can repeat the process for degree $k+1$, and so on, until we have reached degree $n$, if necessary.
\end{remark}

By using a polynomial $p_2$ as in Lemma \ref{lem:exists-q}, Proposition \ref{prp:Grd} guarantees that $\nabla p_2|_{\x}$ will be
orthogonal to $\S_1$. Recall though that for the purpose of the filtration we are interested in constructing a hyperplane $\V_2$ of $\V_1$. Since there is no guarantee that $\nabla p_2|_{\x}$ is inside $\V_1$ (thus defining a hyperplane of $\V_1$), 
we must project $\nabla p_2|_{\x}$ onto $\V_{1}$ and guarantee that this projection is still orthogonal to $\S_1$. The next Lemma ensures that this is always the case.

\begin{lemma} \label{lem:project-q}
	Let $0 \neq p_2 \in \I_{\X,\le m} - \I_{\V_1}$
	such that $\nabla p_2|_{\x} \not\in \Span(\b_{1})$. Then $\0 \neq
	\pi_{\V_{1}}(\nabla p_2|_{\x}) \perp \S_1$.
\end{lemma}
\begin{proof}
	For the sake of contradiction, suppose that $\pi_{\V_{1}}(\nabla p_2|_{\x})=0$.
	Setting $\b_{11}:=\b_1$,
	let us augment $\b_{11}$ to a basis $\b_{11},\b_{12}\dots,\b_{1c}$ for the orthogonal complement of $\S_1$ in $\Re^D$. In fact, we can choose the vectors
	$\b_{12},\dots,\b_{1c}$ to be a basis for the orthogonal complement of $\S_1$ inside $\V_{1}$. By proposition
	\ref{prp:VIS}, $p_2$ must have the form
	\begin{align}
	p_2(x)=q_1(x)(\b_{11}^\transpose x)+q_2(x) (\b_{12}^\transpose x)+\cdots +q_c(x) (\b_{1c}^\transpose x),
	\end{align} where $q_1,\dots,q_c$ are homogeneous polynomials of 
	degree $\deg(p_2)-1$. Then
	\begin{align}
	\nabla p_2|_{\x} = q_1(\x)\b_{11}+ q_2(\x)\b_{12}+\cdots + q_c(\x)\b_{1c}. \label{eq:NablaCombination}
	\end{align} Projecting the above equation orthogonally onto $\V_{1}$ we get
	\begin{align}
	\pi_{\V_{1}}(\nabla p_2|_{\x})=q_2(\x)\b_{12}+\cdots + q_c(\x)\b_{1c},\label{eq:NablaProjection}
	\end{align} which is zero by hypothesis. Since $\b_{12},\cdots,\b_{1c}$ are linearly independent vectors of $\V_{1}$ it must be the case that 
	$q_2(\x)=\cdots=q_c(\x)=0$. But this implies that $\nabla p_2|_{\x}=q_1(\x) \b_{11}$, which is a contradiction on the non-colinearity of $\nabla p_2|_{\x}$ with $\b_{11}$.	Hence it must be the case that 
	$0 \neq \pi_{\V_{1}}(\nabla p_2|_{\x})$. The fact that $\pi_{\V_{1}}(\nabla p_2|_{\x}) \perp \S_1$ follows from \eqref{eq:NablaProjection} and the fact that
	by definition $\b_{12},\dots,\b_{1c}$ are orthogonal to $\S_1$.	
\end{proof}

At this point, letting $\b_2:=\pi_{\V_{1}}(\nabla p_2|_{\x})$, we can define $\V_2 = \Span(\b_1,\b_2)^\perp$, which is a subspace of codimension $1$ inside $\V_1$ (and hence of codimension $2$ inside $\V_0:=\Re^D$). As before, we can define a subspace sub-arrangement $\A_2$ of $\A_1$ by intersecting $\A_1$ with $\V_2$. Once again, this intersection can be realized in the algebraic domain as $\A_2 = \mathcal{Z}(\I_{\X, m},\b_1^\transpose x, \b_2^\transpose x)$. Next, we have a similar result as in Lemmas \ref{lem:V=A-geometry} and \ref{lem:V=A-algebra}, which we now prove in general form:

\begin{lemma}\label{lem:FiltrationTermination}
	Let $\b_{1},\dots,\b_{s}$ be $s$ vectors orthogonal to $\S_1$ and define the intermediate ambient space
	$\V_{s} := \Span(\b_{1},\cdots,\b_{s})^\perp$. Let $\A_{s}$ be the subspace arrangement obtained by intersecting
	$\A$ with $\V_{s}$. Then the following are equivalent:
	\begin{enumerate}[(i)]
		\item $\V_{s}=\A_{s}$
		\item $\V_{s}=\S_1$
		\item $\S_1 = \Span(\b_{1},\dots,\b_s)^\perp$
		\item $\I_{\X, m} \subset \langle \b_{1}^\transpose x, \dots, \b_{s}^\transpose x \rangle_{m}$.		
	\end{enumerate}	 
\end{lemma}
\begin{proof}			
	$(i)\Rightarrow (ii):$ By taking vanishing ideals on both sides of
	$\V_s = \S_1 \bigcup_{i>1} (\S_i \cap \V_s)$ we get 
	$\I_{\V_s} = \I_{\S_1} \bigcap_{i>1} \I_{\S_i \cap \V_s}$. By using
	Proposition \ref{prp:ideals-intersection} in a similar fashion as in the proof of Lemma \ref{lem:V=A-geometry}, we conclude that $\V_s = \S_1$.
	$(ii)\Rightarrow (iii):$ This is obvious from the definition of $\V_s$.
	$(iii)\Rightarrow (iv):$ Let $h \in \I_{\X, m}$. Then $h$ vanishes on $\A$ and hence on $\S_1$ and by Proposition \ref{prp:VIS} we must have that
	$h \in \I_{\S_1}=\langle \b_{1}^\transpose x, \dots, \b_{s}^\transpose x \rangle$. $(iv)\Rightarrow (i):$ $\I_{\X, m} \subset \langle \b_{1}^\transpose x, \dots, \b_{s}^\transpose x \rangle_m$ can be written as $\I_{\X, m} \subset \I_{\V_s}$. By the general position assumption $\I_{\A, m} = \I_{\X, m}$ and so we have $\I_{\A, m} \subset \I_{\V_s}$. Taking zero sets on both sides we get $\A \supset \V_s$, and intersecting both sides of this relation with $\V_s$, we get $\A_s \supset \V_s$. Since $\A_s \subset \V_s$, this implies that $\V_s = \A_s$. 
\end{proof}

Similarly to Lemma \ref{lem:SinglePolynomialCheck} we have:

\begin{lemma} \label{lem:SinglePolynomialCheck-general}
	Let $\boldsymbol{V}_s=[\v_1,\dots,\v_{D-s}]$ be a basis for $\V_s$, and let $\sigma_{\boldsymbol{V}_s}:\V_s \rightarrow \Re^{D-s}$ be the linear map that takes a vector $\v = \alpha_1 \v_1 + \cdots + \alpha_{D-s} \v_{D-s}$ to its coordinate representation $(\alpha_1,\dots,\alpha_{D-s})^\transpose$. Suppose that
	$\X_1$ is in general position inside $\S_1$ with respect to degree $m$. Then
	$\V_s = \A_s$ if and only if the embedded data matrix $\nu_{m}(\sigma_{\boldsymbol{V}_s}(\X \cap \V_s))$ is full rank.
\end{lemma}

By Lemma \ref{lem:FiltrationTermination}, if $\I_{\X, m} \subset \langle \b_{1}^\transpose x, \b_{2}^\transpose x \rangle$, the algorithm terminates the filtration with output the orthogonal basis $\left\{\b_1,\b_2\right\}$ for the orthogonal complement of the irreducible component $\S_1$ of $\A$. If on the other hand $\I_{\X, m} \not\subset \langle \b_{1}^\transpose x, \b_{2}^\transpose x \rangle$, then the algorithm picks a basis element $p_3$ of $\I_{\X, m} $ such that $p_3 \not\in \I_{\V_2}$ and $\nabla p_3 |_{\x} \not\in \Span(\b_1,\b_2)$, and defines a subspace $\V_3$ of codimension $1$ inside $\V_2$ using $\pi_{\V_2}\left(\nabla p_3 |_{\x}\right)$.\footnote{The proof of existence of such a $p_3$ is similar to the proof of Lemma \ref{lem:exists-q} and is omitted.} Setting $\b_3:= \pi_{\V_2}\left(\nabla p_3 |_{\x}\right)$, the algorithm uses Lemma \ref{lem:FiltrationTermination} to determine whether to terminate the filtration or take one more step and so on. 

The principles established in the previous sections, formally lead us to the algebraic descending filtration Algorithm \ref{alg:ADF} and its Theorem \ref{thm:ADF} of correctness.	

\begin{algorithm} \caption{Algebraic Descending Filtration (ADF)}\label{alg:ADF} 
	\begin{algorithmic}[1] 
		\Procedure{ADF}{$p,\x,\X,m$}		
		\State $\mathfrak{B} \gets \nabla p|_{\x}$; 
		\While{$\I_{\X, m} \not\subset \langle \b^\transpose x: \b \in \mathfrak{B} \rangle $}		
		\State find $p \in \I_{\X,\le m} - \langle \b^\transpose x: \b \in \mathfrak{B} \rangle$ s.t. $\nabla p|_{\x} \not\in \Span(\mathfrak{B})$;	
		\State $\mathfrak{B} \gets \mathfrak{B} \cup \left\{{\tiny } \pi_{\Span(\mathfrak{B})^\perp}\left(\nabla p|_{\x}\right)\right\}$;							 			
		\EndWhile									
		\State \Return $\mathfrak{B}$;
		\EndProcedure 				
	\end{algorithmic} 
\end{algorithm}	

\begin{theorem}[Correctness of Algorithm \ref{alg:ADF}] \label{thm:ADF}
	Let $\X=\left\{\x_1,\dots,\x_N\right\}$ be a finite set of points in general position (Definition \ref{dfn:GeneralPosition}) with respect to degree $m$ inside a transversal (Definition \ref{dfn:transversal}) arrangement $\A$ of at most $m$ linear subspaces of $\Re^D$. Let $p$ be a polynomial of minimal degree that vanishes on $\X$. Then there always exists a nonsingular $\x \in \X$ such that $\nabla p|_{\x} \neq \0$, and for such an $\x$, the output $\mathfrak{B}$ of Algorithm \ref{alg:ADF} is an orthogonal basis for the orthogonal complement in $\Re^D$ of the irreducible component of $\A$ that contains $\x$.
\end{theorem}
 
\subsection{The FASC algorithm} \label{subsection:All}

In Sections \ref{subsection:FirstStepFiltration}-\ref{subsection:MultipleSteps} we established the theory of a single filtration, according to which one starts with a nonsingular point $\x_1:=\x \in \A \cap \X$ and obtains an orthogonal basis $\b_{11},\dots,\b_{1c_1}$ for the orthogonal complement of the irreducible component $\S_1$ of $\A$ that contains reference point $\x_1$. To obtain an orthogonal basis $\b_{21},\dots,\b_{2c_2}$ corresponding to a second irreducible component $\S_2$ of $\A$, our approach is the natural one: remove $\X_1$ from $\X$ and run a filtration on the set $\X^{(1)}:=\X - \X_1$. All we need for the theory of Sections \ref{subsection:FirstStepFiltration}-\ref{subsection:MultipleSteps} to be applicable to the set $\X^{(1)}$, is that  $\X^{(1)}$ be in general position inside the arrangement $\A^{(1)}:=\S_{2} \cup \cdots \cup \S_n$. But this has been proved in Lemma \ref{lem:GeneralPosition}.
With Lemma \ref{lem:GeneralPosition} establishing the correctness of recursive application of a single filtration, the correctness of the FASC Algorithm \ref{alg:AASC} follows at once, as in Theorem \ref{thm:AASC}. Note that in Algorithm \ref{alg:AASC}, $n$ is the number of subspaces, while $\mathfrak{D}$ and $\mathfrak{L}$ are ordered sets, such that, up to a permutation, the $i$-th element of $\mathfrak{D}$ is $d_i = \dim \S_i$, and the $i$-th element of $\mathfrak{L}$ is an orthogonal basis for $\S_i^\perp$.

\begin{algorithm} \caption{Filtrated Algebraic Subspace Clustering}\label{alg:AASC} \begin{algorithmic}[1] 
		\Procedure{FASC}{$\X \in \Re^{D \times N},m$}		
		\State 	$n \gets 0$; $\mathfrak{D} \gets \emptyset $; $\mathfrak{L} \gets \emptyset$; 					
		\While{$\X \neq \emptyset$}
		\State find polynomial $p$ of minimal degree that vanishes on $\X$;
		\State find $\x \in \X$ s.t. $\nabla p|_{\x} \neq 0$;
		\State $\mathfrak{B} \gets$ ADF$(p,\x,\X,m)$;
		\State $\mathfrak{L} \gets \mathfrak{L} \cup \left\{\mathfrak{B}\right\}$; 
		\State $\mathfrak{D} \gets \mathfrak{D} \cup \left\{D - \card(\mathfrak{B}) \right\}$;				
		\State $\X \gets \X - \Span(\mathfrak{B})^\perp$;	
		\State $n \gets n+1$; $m \gets m-1$;	
		\EndWhile						
		\State \Return $n,\mathfrak{D},\mathfrak{L}$;
		\EndProcedure 				
	\end{algorithmic} 
\end{algorithm}

\begin{theorem}[Correctness of Algorithm \ref{alg:AASC}] \label{thm:AASC}
	Let $\X=\left\{\x_1,\dots,\x_N\right\}$ be a set in general position 
	with respect to degree $m$ (Definition \ref{dfn:GeneralPosition}) inside
	a transversal (Definition \ref{dfn:transversal}) arrangement $\A$ of at most $m$ linear
	subspaces of $\Re^D$. For such an $\X$ and $m$, Algorithm \ref{alg:AASC} always terminates with output a set $\mathfrak{L}=\left\{\mathfrak{B}_1,\dots,\mathfrak{B}_n\right\}$, such that up to a permutation, $\mathfrak{B}_i$ is an orthogonal basis for the orthogonal complement of the $i^{th}$ irreducible component $\S_i$ of $\A$, i.e., $\S_i = \Span(\mathfrak{B}_i)^\perp, i=1,\dots,n$, and $\A = \bigcup_{i=1}^n \S_i$.
\end{theorem}

\section{Filtrated Spectral Algebraic Subspace Clustering} \label{section:FSASC}
In this section we show how FASC (Sections \ref{section:geometricAASC}-\ref{section:mfAASC}) can be adapted to a working subspace clustering algorithm that is robust to noise. As we will soon see, the success of such an algorithm depends on being able to 1) implement a single filtration in a robust fashion, and 2) combine multiple robust filtrations to obtain the clustering of the points. 

\subsection{Implementing robust filtrations} \label{subsection:FSASC-filtrations}
Recall that the filtration component ADF (Algorithm \ref{alg:ADF}) of the FASC Algorithm \ref{alg:AASC}, is based on computing a descending filtration of ambient 
spaces $\V_1 \supset \V_2 \supset \cdots$. Recall that $\V_1$ is obtained as the hyperplane of $\Re^D$ with normal vector $\nabla p|_{\x}$, where $\x$ is the reference point associated with the filtration, and $p$ a polynomial of minimal degree $k$ that vanishes on $\X$. In the absence of noise, the value of $k$ can be characterized as the smallest $\ell$ such that $\nu_{\ell}(\X)$ drops rank (see section \ref{subsection:Hyperplanes} for notation). In the presence of noise, and assuming that $\X$ has cardinality at least ${m+D-1 \choose m}$, there will be in general no vanishing polynomial of degree $\le m$, i.e., the embedded data matrix $\nu_{\ell}(\X)$ will have full column rank, for any $\ell \le m$. Hence, in the presence of noise we do not know a-priori what the minimal degree $k$ is. On the other hand, we do know that $m \ge n$, which implies that the underlying subspace arrangement $\A$ admits vanishing polynomials of degree $m$. Thus a reasonable choice for an approximate vanishing polynomial $p_1:=p$, is the polynomial whose coefficients are given by the right singular vector of $\nu_m(\X)$ that corresponds to the smallest singular value. Recall also that in the absence of noise we chose our reference point $\x \in \X$ such that $\nabla p_1|_{\x} \neq \0$. In the presence of noise this condition will be almost surely true every point $\x \in \X$; then one can select the point that gives the largest gradient, i.e., we can pick as reference point an $\x$ that maximizes the norm of the gradient $\left\|\nabla p_1|_{\x}\right\|_2$.

Moving on, ADF constructs the filtration of $\X$ by intersecting $\X$ with the intermediate ambient spaces $\V_1 \supset \V_2 \supset \cdots$. In the presence of noise in the dataset $\X$, such intersections will almost surely be empty. As it turns out, we can replace the operation of intersecting $\X$ with the intermediate spaces $\V_s,s=1,2,\dots$, by projecting $\X$ onto $\V_s$. \emph{In the absence of noise}, the norm of the points of $\X$ that lie in $\V_s$ will remain unchanged after projection, while points that lie outside $\V_s$ will witness a drop in their norm upon projection onto $\V_s$. Points whose norm is reduced can then be removed and the end result of this process is equivalent to intersecting $\X$ with $\V_s$.
\emph{In the presence of noise} one can choose a threshold $\delta>0$, such that if the distance of a point from subspace $\V_s$ is less than $\delta$, then the point is maintained after projection onto $\V_s$, otherwise it is removed. 
But how to choose $\delta$? One reasonable way to proceed, is to consider the polynomial $p$ that corresponds to the right singular vector of $\nu_m(\X)$ of smallest singular value,
and then consider the quantity 
\begin{align}
\beta(\X) := \frac{1}{N} \sum_{j=1}^N \frac{\left|\x_j^\transpose \nabla p|_{\x_j}\right|}{\left\|\x_j\right\|_2 \left\|\nabla p|_{\x_j}\right\|_2}.
\end{align} Notice that in the absence of noise $\dim \mathcal{N}(\nu_m(\X))>0$ and subsequently 
$\beta(\X)=0$. In the presence of noise however, $\beta(\X)$ represents the average distance of a point $\x$ in the dataset to the hyperplane that it produces by means of $\nabla p|_{\x}$ (in the absence of noise this distance is zero by Proposition \ref{prp:Grd}). Hence intuitively, $\delta$ should be of the same order of magnitude as $\beta(\X)$; a natural choice is to set $\delta:= \gamma \cdot \beta(\X)$, where $\gamma$ is a user-defined parameter taking values close to $1$. Having projected $\X$ onto $\V_1$ and removed points whose distance from $\V_1$ is larger than $\delta$, we obtain a second approximate polynomial $p_2$ from the right singular vector of smallest singular value of the embedded data matrix of the remaining projected points and so on.
 
It remains to devise a robust criterion for terminating the filtration. Recall that the criterion for terminating the filtration in ADF is $\I_{\X, m} \subset \langle \b_1^\transpose x,\dots,\b_s^\transpose x\rangle_m$, where $\V_s=\Span(\b_1,\dots,\b_s)^\perp$. Checking this criterion is equivalent to checking the inclusion $\I_{\X, m} \subset \langle \b_1^\transpose x,\dots,\b_s^\transpose x\rangle_{ m}$ of finite dimensional vector spaces. In principle, this requires computing a basis for the vector space $\I_{\X, m}$. Now recall from section \ref{subsection:SASC-A}, that it is precisely this computation that renders the classic polynomial differentiation algorithm unstable to noise; the main difficulty being the correct estimation of $\dim\left(\I_{\X, m}\right)$, and the dramatic dependence of the quality of clustering on this estimate. Consequently, for the purpose of obtaining a robust algorithm, it is imperative to avoid such a computation. But we know from Lemma \ref{lem:SinglePolynomialCheck-general} that, if $\X_i:=\X \cap \S_i$ is in general position inside $\S_i$ with respect to degree $m$ for every $i \in [n]$, then the criterion for terminating the filtration is equivalent to checking whether in the coordinate representation of $\V_s$ the points $\X \cap \V_s$ admit a vanishing polynomial of degree $ m$. But this is computationally equivalent to checking whether $\mathcal{N}\left(\nu_m\left(\sigma_{\boldsymbol{V}_s}(\X \cap \V_s)\right)\right) \neq 0$; see notation in Lemma \ref{lem:SinglePolynomialCheck-general}. This is a much easier problem than estimating $\dim\left(\I_{\X, m}\right)$, and we solve it implicitly as follows. Recall that in the absence of noise, the norm of the reference point remains unchanged as it passes through the filtration. Hence, it is natural to terminate the filtration at step $s$, if the distance from the projected reference point\footnote{Here by projected reference point we mean the image of the reference point under all projections up to step $s$.} to $\V_{s+1}$ is more than $\delta$, i.e., if the projected reference point is among the points that are being removed upon projection from $\V_s$ to $\V_{s+1}$. To guard against overestimating the number of steps in the filtration, we enhance the termination criterion by additionally deciding to terminate at step $s$ if the number of points that survived the projection from $\V_s$ to $\V_{s+1}$ is less than a pre-defined integer $L$, which is to be thought of as the minimum number of points in a cluster. 

\subsection{Combining multiple filtrations} \label{subsection:FSASC-MultipleFiltrations}
Having determined a robust algorithmic implementation for a single filtration, we face the following issue: In general, two points lying approximately in the same subspace $\S$ will produce different hyperplanes that approximately contain $\S$ with different levels of accuracy. In the noiseless case any point would be equally good. In the presence of noise though, the choice of the reference point $\x$ becomes significant. How should $\x$ be chosen? To deal with this problem in a robust fashion, it is once again natural to construct a single filtration for each point in $\X$ and define an affinity between points $j$ and $j'$ as
\begin{align} \label{eq:affinity}
\C_{jj',\text{FSASC}} = \begin{cases}
\| \pi_{s_j}^{(j)} \circ \cdots \circ \pi_1^{(j)} (\x_{j'}) \|  & \text{if $\x_{j'}$ remains}\\
0 &  \text{otherwise},
\end{cases}
\end{align} 
where $\pi_s^{(j)}$ is the projection from $\V_s$ to $\V_{s+1}$ associated to the filtration of point $\x_j$ and $s_j$ is the length of that filtration. This affinity captures the fact that if points $\x_j$ and $\x_{j'}$ are in the same subspace, then the norm of $\x_{j'}$ should not change from step $0$ to step $c$ of the filtration computed with reference point $\x_j$, where $c=D - \dim (\S)$ is the codimension of the irreducible component $\S$ associated to reference point $\x_j$. Otherwise, if $\x_j$ and $\x_{j'}$ are in different subspaces, the norm of $\x_{j'}$ is expected to be reduced by the time the filtration reaches step $c$. In the case of noiseless data, only the points in the correct subspace survive step $c$ and their norms are precisely equal to one. In the case of noisy data, the affinity defined above will only be approximate.

\subsection{The FSASC algorithm} \label{subsection:FSASC}
Having an affinity matrix as in eq. \eqref{eq:affinity}, standard spectral clustering techniques can be applied to obtain a clustering of $\X$ into $n$ groups. We emphasize that in contrast to the abstract case of Algorithm \ref{alg:AASC}, the number $n$ of clusters must be given as input to the algorithm. On the other hand, the algorithm does not require the subspace dimensions to be given: these are implicitly estimated by means of the filtrations. Finally, one may choose to implement the above scheme for $M$ distinct values of the parameter $\gamma$ and choose the affinity matrix that leads to the smallest $n^{th}$ eigengap. The above discussion leads to the \emph{Filtrated Spectral Algebraic Subspace Clustering (FSASC)} Algorithm \ref{alg:FSASC}, in which
\begin{itemize} 
\item $\textsc{Spectrum}\big(NL(\C+\C^{\transpose})\big)$ denotes the spectrum of the normalized Laplacian matrix of $\C + \C^{\transpose}$, \item $\textsc{SpecClust}\big(\C^*+(\C^*)^\transpose,n\big)$ denotes spectral clustering being applied to $\C^*+\C^{*\transpose}$ to obtain $n$ clusters,
\item $\textsc{Vanishing}\big(\nu_n(\X)\big)$ is the polynomial whose coefficients are the right singular vector of $\nu_n(\X)$ corresponding to the smallest singular value.
\item $\pi \gets \left[ \Re^d \rightarrow  \H \xrightarrow{\sim} \Re^{d-1} \right]$ is to be read as ``$\pi$ is assigned the composite linear transformation $\Re^d \rightarrow  \H \xrightarrow{\sim} \Re^{d-1}$, where the first arrow is the orthogonal projection of $\Re^{d}$ to hyperplane $\H$, and the second arrow is the linear isomorphism that maps a basis of $\H$ in $\Re^d$ to the standard coordinate basis of $\Re^{d-1}$".
\end{itemize}

\begin{algorithm} \caption{Filtrated Spectral Algebraic Subspace Clustering (FSASC)} \label{alg:FSASC} 
\begin{spacing}{0.8}
\begin{algorithmic} [1]
\Procedure{FSASC}{$\mathcal{X}, D, n,L,\{\gamma_m\}_{m=1}^M$}
\If {$N < \mathcal{M}_n(D)$}
\State \Return('Not enough points');
\Else
\State eigengap $\gets 0$; $\C^* \gets 0_{N \times N}$;
\State $\x_j \gets \x_j / ||\x_j||, \, \forall j \in [N]$;
\State $p \gets \Call{Vanishing}{\nu_n(\X)}$;
\State $\beta \gets  \frac{1}{N} \sum_{j=1}^N \big|\langle \x_j, \frac{\nabla p|_{\x_j}}{||\nabla p|_{\x_j}||}\rangle\big|$;
\For {$k = 1 : M$}
\State $\delta \gets  \beta \cdot \gamma_k, \, \C \gets 0_{N \times N}$;
\For {$j = 1 : N$}
\State $C_{j,:} \gets \Call{Filtration}{\X,\x_j,p,L,\delta,n}$;
\EndFor
\State $\{\lambda_s\}_{s=1}^N \gets \Call{Spectrum}{NL(\C+\C^\transpose)}$ ;
\If {(eigengap $< \lambda_{n+1} - \lambda_n$)}
\State eigengap $\gets \lambda_{n+1} - \lambda_n$; $\C^* \gets \C$;   
\EndIf     
\EndFor
\State $\left\{\Y_i\right\}_{i=1}^n \gets \Call{SpecClust}{\C^*+\C^{*\transpose},n}$; 
\State \Return $\left\{\Y_i\right\}_{i=1}^n$;
\EndIf
\EndProcedure 

\Statex
\Function{Filtration}{$\X,\x,p,L,\delta,n$}

\State $d \gets D, \, \J \gets [N],  q \gets p, \boldsymbol{c} \gets 0_{1 \times N}$;
\State  flag $\gets 1$;
\While {($d >1$) and ($\text{flag}=1$)}						
\State $\H \gets \langle \nabla q|_{\x} \rangle^\perp, \, \pi \gets \left[ \Re^d \rightarrow  \H \xrightarrow{\sim} \Re^{d-1} \right]$;
\If {$(||\x|| - || \pi(\x)||)/||\x|| > \delta$}
\If {$d = D$}
\State $\boldsymbol{c}(j') \gets || \pi(\x_j')||, \, \forall j' \in [N]$;						
\EndIf			
\State flag $\gets 0$;		
\Else
\State $\J \gets \left\{j' \in [N] : \frac{||\x_{j'}|| - || \pi(\x_{j'})||}{||\x_{j'}||} \le \delta \right\}$
\If {$\left| \J \right| < L$}
\State flag $\gets 0$;
\Else
\State $\boldsymbol{c}(j') \gets || \pi(\x_j')||, \, \forall j' \in \J$;
\State $\boldsymbol{c}(j') \gets 0, \, \forall j' \in [N]- \J$;
\If {$|\J| < \mathcal{M}_n(d)$}
\State flag $\gets 0$;
\Else
\State $d \gets d -1, \x \gets \pi(\x)$;
\State $\x_{j'} \gets \pi(\x_{j'}) \, \forall j' \in \J$;
\State $ \X \gets \left\{\x_{j'}: j' \in \J \right\}$;											
\State $q \gets \Call{Vanishing}{\nu_n(\X)}$;
\EndIf
\EndIf
\EndIf
\EndWhile

\State \Return($\boldsymbol{c}$);
\EndFunction

\end{algorithmic} 
\end{spacing}
\end{algorithm} 


\subsection{A distance-based affinity (SASC-D)} \label{subsection:SASC-D}

Observe that\footnote{We will henceforth be assuming that all points $\x_1,\dots,\x_N$ are normalized to unit $\ell_2$-norm.} the FSASC affinity \eqref{eq:affinity} between points $\x_j$ and $\x_{j'}$, can be interpreted as the distance of point $\x_{j'}$ to the 
orthogonal complement of the final ambient space $\V_{s_j}$ of the filtration corresponding to reference point $\x_j$. If all irreducible components of $\A$ were hyperplanes, then the optimal length of each filtration would be $1$. Inspired by this observation, we may define a simple \emph{distance-based} affinity, alternative to the \emph{angle-based} affinity of eq. \eqref{eq:ABA}, by
\begin{align}
\C_{jj',\text{dist}}: = 1 - \frac{\left|\x_{j'}^\transpose \nabla p|_{\x_j}\right|}{\left\|\nabla p|_{\x_j}\right\|_2}. \label{eq:DBA}
\end{align} The affinity of eq. \eqref{eq:DBA} is theoretically justified only for hyperplanes, as $\C_{jj',\text{angle}}$ is; yet as we will soon see in the experiments, $\C_{jj',\text{dist}}$ is much more robust than $\C_{jj',\text{angle}}$ in the case of subspaces of different dimensions. We attribute this phenomenon to the fact that, in the absence of noise, it is always the case that $\C_{jj',\text{dist}}=1$ whenever $\x_j,\x_{j'}$ lie in the same irreducible component; as mentioned in section \ref{subsection:SASC-A}, this need not be the case for $\C_{jj',\text{angle}}$. We will be referring to the Spectral ASC method that uses affinity \eqref{eq:DBA} as 
\emph{SASC-D}.

\subsection{Discussion on the computational complexity} \label{subsection:complexity}

As mentioned in section \ref{section:ASC}, the main object that needs to be computed in  algebraic subspace clustering is a vanishing polynomial $p$ in $D$ variables of degree $n$, where $D$ is the ambient dimension of the data and $n$ is the number of subspaces. 
This amounts to computing a right null-vector of the $N \times \mathcal{M}_n(D)$ embedded data matrix $\nu_n(\X)$, where $\mathcal{M}_n(D):={n+D-1 \choose n}$, and $N \ge \mathcal{M}_n(D)$. In practice, the data are noisy and there are usually no vanishing polynomials of degree $n$; instead one needs to compute the right singular vector of the embedded data matrix that corresponds to the smallest singular value. 
Approximate iterative methods for performing this task do exist  \cite{Schwetlick:LAA03,Liang:ETNA14,Stathopoulos:SIAM15}, and in this work we use the MATLAB function \texttt{svds.m}, which is based on an \emph{inverse-shift iteration} technique; see, e.g., the introduction of \cite{Liang:ETNA14}. Even though \texttt{svds.m} is in principle more efficient than computing the full SVD of $\nu_n(\X)$ via the MATLAB function \texttt{svd.m}, the complexity of both functions is of the same order 
\begin{align}
N \mathcal{M}_n(D)^2  = N { n+D-1 \choose n}^2, \label{eq:SVDcomplexity}
\end{align} which is the well-known complexity of SVD \cite{golub1996matrix} adapted to the dimensions of $\nu_n(\X)$. This is because \texttt{svds.m} requires at each iteration the solution to a linear system of equations whose coefficient matrix has size of the same order as the size of $\nu_n(\X)$.

Evidently, the complexity of \eqref{eq:SVDcomplexity} is prohibitive for large $D$ even for moderate values of $n$. 
If we discount the spectral clustering step, this is precisely the complexity of SASC-A of 
section \ref{subsection:SASC-A} as well as of SASC-D of section \ref{subsection:SASC-D}. 
On the other hand, FSASC (Algorithm \ref{alg:FSASC}) is even more computationally demanding, as it requires the computation of a vanishing polynomial
at each step of every filtration, and there are as many filtrations as the total number of points.  Assuming for simplicity that there is no noise and that the dimensions of all subspaces are equal to $d<D$, then the complexity of a single filtration in FSASC is of the order of
\begin{align}
\sum_{i=0}^{D-d} N \left(\mathcal{M}_n(D-i) \right)^2 = N \sum_{i=0}^{D-d} { n+D-i-1 \choose n}^2.
\end{align} Since FSASC computes a filtration for each and every point, its total complexity (discounting the spectral clustering step and assuming that we are using a single value for the parameter $\gamma$) is
\begin{align}
N \, \sum_{i=0}^{D-d} N \left(\mathcal{M}_n(D-i) \right)^2 = N^2 \sum_{i=0}^{D-d} { n+D-i-1 \choose n}^2.
\end{align} Even though the filtrations are independent of each other, and hence fully parallelizable, the complexity of FSASC is still prohibitive for large scale applications even after parallelization. Nevertheless, when the subspace dimensions are small, then FSASC is applicable after one reduces the dimensionality of the data by means of a projection, as will be done in section \ref{subsection:MotionSegmentation}. At any case, we hope that the complexity issue of FSASC will be addressed in future research. 

\section{Experiments}
\label{subsection:Experiments}
In this section we evaluate experimentally the proposed methods FSASC (Algorithm \ref{alg:FSASC}) and SASC-D (section \ref{subsection:SASC-D}) and compare them to other state-of-the-art subspace clustering methods, using synthetic data (section \ref{subsection:SyntheticExperiments}), as well as real motion segmentation data (section \ref{subsection:MotionSegmentation}).

\subsection{Experiments on synthetic data} \label{subsection:SyntheticExperiments}
We begin by randomly generating $n=3$ subspaces of various dimension configurations $(d_1,d_2,d_3)$ in $\Re^9$. The choice $D=9$ for the ambient dimension is motivated by applications in two-view geometry \cite{Hartley-Zisserman04,Vidal:IJCV06-multibody}. Once the subspaces are randomly generated, we use a zero-mean unit-variance Gaussian distribution with support on each subspace to randomly sample $N_i = 200$ points per subspace.
The points of each subspace are then corrupted by additive zero-mean Gaussian noise with standard deviation $\sigma\in\{0,0.01, 0.03, 0.05\}$ and support in the orthogonal complement of the subspace. All data points are subsequently normalized to have unit euclidean norm. 

\begin{table}
\centering
\caption{Mean subspace clustering error in $\%$ over $100$ independent trials for synthetic data randomly generated in three random subspaces of $\Re^{9}$ of dimensions $(d_1,d_2,d_3)$. The total number of points is $N=600$ with $200$ points associated to each subspace. We consider noiseless data $(\sigma=0)$ as well as data corrupted by zero-mean additive white noise of standard deviation $\sigma$ and support in the orthogonal complement of each subspace.}
\label{table:Synthetic_n3_error}
\ra{0.7}
\begin{tabular}{@{}l@{\, \, \,}c@{\, \,}c@{\, \,}c@{\, \,}c@{\, \,}c@{\, \,}c@{\, \,}c@{\, \,}c}\toprule[1pt] method &  $(2,3,4)$ & $(4,5,6)$ & $(6,7,8)$ &  $(2,5,8)$  & $(3,3,3)$ &  $(6,6,6)$ & $(7,7,7)$ & $(8,8,8)$\\ 
\midrule[0.5pt]
& & & & $\sigma=0$\\
\cmidrule{5-5}
FSASC & $\boldsymbol{0}$ & $\boldsymbol{0}$ & $\boldsymbol{0}$ & $\boldsymbol{0}$ & $\boldsymbol{0}$ & $\boldsymbol{0}$ & $\boldsymbol{0}$  & $\boldsymbol{0}$ \\
SASC-D & $\boldsymbol{0}$ & $\boldsymbol{0}$ & $\boldsymbol{0}$ & $\boldsymbol{0}$ & $\boldsymbol{0}$ & $\boldsymbol{0}$ & $\boldsymbol{0}$  & $\boldsymbol{0}$ \\
SASC-A & $42$ & $39$ & $6$ & $14$ & $37$ & $24$ & $12$  & $\boldsymbol{0}$ \\
SSC & $\boldsymbol{0}$ & $1$ & $18$ & $49$ & $\boldsymbol{0}$ & $3$ & $14$  & $55$ \\
LRR & $\boldsymbol{0}$ & $3$ & $39$ & $5$ & $\boldsymbol{0}$ & $9$ & $42$  & $51$ \\
LRR-H & $\boldsymbol{0}$ & $3$ & $36$ & $6$ & $\boldsymbol{0}$ & $8$ & $38$  & $51$ \\
LRSC & $\boldsymbol{0}$ & $3$ & $39$ & $5$ & $\boldsymbol{0}$ & $9$ & $42$  & $51$ \\
LSR & $\boldsymbol{0}$ & $3$ & $39$ & $5$ & $\boldsymbol{0}$ & $9$ & $42$  & $51$ \\
LSR-H & $\boldsymbol{0}$ & $3$ & $32$ & $6$ & $\boldsymbol{0}$ & $8$ & $38$  & $51$ \\

\midrule[0.1pt]

& & & & $\sigma=0.01$\\
\cmidrule{5-5}
FSASC & $\boldsymbol{0}$ & $\boldsymbol{0}$ & $\boldsymbol{0}$ & $\boldsymbol{1}$ & $\boldsymbol{0}$ & $\boldsymbol{0}$ & $\boldsymbol{0}$  & $5$ \\
SASC-D & $\boldsymbol{0}$ & $\boldsymbol{0}$ & $1$ & $\boldsymbol{1}$ & $\boldsymbol{0}$ & $\boldsymbol{0}$ & $\boldsymbol{0}$  & $\boldsymbol{3}$ \\
SASC-A & $54$ & $45$ & $8$ & $24$ & $57$ & $36$ & $13$  & $\boldsymbol{3}$ \\
SSC & $2$ & $2$ & $18$ & $49$ & $\boldsymbol{0}$ & $3$ & $13$  & $55$ \\
LRR & $\boldsymbol{0}$ & $3$ & $38$ & $5$ & $\boldsymbol{0}$ & $9$ & $42$  & $51$ \\
LRR-H & $\boldsymbol{0}$ & $3$ & $36$ & $7$ & $\boldsymbol{0}$ & $8$ & $38$  & $51$ \\
LRSC & $\boldsymbol{0}$ & $3$ & $38$ & $5$ & $\boldsymbol{0}$ & $9$ & $42$  & $51$ \\
LSR & $\boldsymbol{0}$ & $3$ & $39$ & $5$ & $\boldsymbol{0}$ & $9$ & $42$  & $51$ \\
LSR-H & $\boldsymbol{0}$ & $3$ & $32$ & $6$ & $\boldsymbol{0}$ & $8$ & $38$  & $51$ \\

\midrule[0.1pt]

& & & & $\sigma=0.03$\\
\cmidrule{5-5}
FSASC & $\boldsymbol{0}$ & $\boldsymbol{0}$ & $\boldsymbol{1}$ & $\boldsymbol{2}$ & $\boldsymbol{0}$ & $\boldsymbol{0}$ & $\boldsymbol{1}$  & $10$  \\
SASC-D & $\boldsymbol{0}$ & $\boldsymbol{0}$ & $4$ & $3$ & $\boldsymbol{0}$ & $1$ & $2$  & $\boldsymbol{6}$ \\
SASC-A & $57$ & $46$ & $13$ & $31$ & $58$ & $37$ & $15$  & $7$ \\
SSC & $\boldsymbol{0}$ & $1$ & $20$ & $48$ & $\boldsymbol{0}$ & $3$ & $13$  & $55$ \\

\midrule[0.1pt]

& & & & $\sigma=0.05$\\
\cmidrule{5-5}
FSASC & $1$ & $\boldsymbol{0}$ & $\boldsymbol{2}$ & $\boldsymbol{3}$ & $1$ & $\boldsymbol{0}$ & $\boldsymbol{2}$ & $14$   \\
SASC-D & $1$ & $1$ & $7$ & $5$ & $1$ & $2$ & $5$  & $\boldsymbol{10}$ \\
SASC-A & $58$ & $46$ & $17$ & $36$ & $60$ & $39$ & $17$  & $11$ \\
SSC & $\boldsymbol{0}$ & $2$ & $20$ & $49$ & $\boldsymbol{0}$ & $3$ & $15$  & $55$ \\
LRR & $1$ & $3$ & $39$ & $6$ & $\boldsymbol{0}$ & $10$ & $42$  & $51$ \\
LRR-H & $1$ & $3$ & $36$ & $13$ & $\boldsymbol{0}$ & $8$ & $38$  & $52$ \\
LRSC & $1$ & $3$ & $39$ & $6$ & $\boldsymbol{0}$ & $10$ & $42$  & $51$ \\
LSR & $1$ & $3$ & $39$ & $6$ & $\boldsymbol{0}$ & $10$ & $42$  & $51$ \\
LSR-H & $1$ & $3$ & $32$ & $7$ & $\boldsymbol{0}$ & $8$ & $38$  & $51$ \\

\bottomrule[1pt]
\end{tabular}
\end{table}

\begin{table}
\centering
\caption{Mean intra-cluster connectivity over $100$ independent trials for synthetic data randomly generated in three random subspaces of $\Re^{9}$ of dimensions $(d_1,d_2,d_3)$. There are $200$ points associated to each subspace, which are corrupted by zero-mean additive white noise of standard deviation $\sigma$ and support in the orthogonal complement of each subspace.}
\label{table:Synthetic_n3_intra}
\ra{0.7}
\begin{tabular}{@{}l@{\, \, \,}c@{\, \,}c@{\, \,}c@{\, \,}c@{\, \,}c@{\, \,}c@{\, \,}c@{\, \,}c}\toprule[1pt] method &  $(2,3,4)$ & $(4,5,6)$ & $(6,7,8)$ &  $(2,5,8)$  & $(3,3,3)$ &  $(6,6,6)$ & $(7,7,7)$ & $(8,8,8)$\\ 

\midrule[0.5pt]

& & & & $\sigma=0$\\
\cmidrule{5-5}
FSASC & $\boldsymbol{1}$ & $\boldsymbol{1}$ & $\boldsymbol{1}$ & $\boldsymbol{1}$ & $\boldsymbol{1}$ & $\boldsymbol{1}$ & $\boldsymbol{1}$ & $\boldsymbol{1}$   \\
SASC-D & $\boldsymbol{1}$ & $\boldsymbol{1}$ & $\boldsymbol{1}$ & $\boldsymbol{1}$ & $\boldsymbol{1}$ & $\boldsymbol{1}$ & $\boldsymbol{1}$ & $\boldsymbol{1}$  \\
SASC-A & $0.37$ & $0.37$ & $0.37$ & $0.39$ & $0.34$ & $0.41$ & $0.37$  & $\boldsymbol{1}$ \\
SSC & $10^{-3}$ & $0.01$ & $10^{-4}$ & $10^{-3} $ & $0.01$ & $0.02$ & $10^{-3}$  & $10^{-7}$ \\
LRR & $0.59$ & $0.37$ & $0.43$ & $0.31$ & $0.64$ & $0.41$ & $0.45$  & $0.50$ \\
LRR-H & $0.28$ & $0.23$ & $0.23$ & $0.19$ & $0.31$ & $0.24$ & $0.24$  & $0.26$ \\
LRSC & $0.59$ & $0.37$ & $0.43$ & $0.31$ & $0.64$ & $0.41$ & $0.45$  & $0.50$ \\
LSR & $0.59$ & $0.37$ & $0.42$ & $0.31$ & $0.64$ & $0.41$ & $0.45$  & $0.50$ \\
LSR-H & $0.28$ & $0.24$ & $0.24$ & $0.21$ & $0.31$ & $0.25$ & $0.25$  & $0.27$ \\

\midrule[0.5pt]

& & & & $\sigma=0.01$\\
\cmidrule{5-5}
FSASC & $0.05$ & $0.35$ & $0.43$ & $0.10$ & $0.09$ & $0.43$ & $0.42$ & $0.43$   \\
SASC-D & $\boldsymbol{0.91}$ & $\boldsymbol{0.93}$ & $\boldsymbol{0.85}$ & $\boldsymbol{0.84}$ & $\boldsymbol{0.94}$ & $\boldsymbol{0.91}$ & $\boldsymbol{0.87}$  & $\boldsymbol{0.85}$ \\
SASC-A & $0.32$ & $0.30$ & $0.12$ & $0.14$ & $0.30$ & $0.29$ & $0.24$  & $0.07$ \\
SSC & $10^{-3}$ & $0.01$ & $10^{-4}$ & $10^{-3} $ & $0.01$ & $0.02$ & $10^{-3}$  & $10^{-7}$ \\
LRR & $0.42$ & $0.37$ & $0.43$ & $0.31$ & $0.51$ & $0.41$ & $0.45$  & $0.50$ \\
LRR-H & $0.13$ & $0.23$ & $0.23$ & $0.17$ & $0.22$ & $0.24$ & $0.24$  & $0.26$ \\
LRSC & $0.42$ & $0.37$ & $0.43$ & $0.31$ & $0.52$ & $0.41$ & $0.45$  & $0.50$ \\
LSR & $0.41$ & $0.37$ & $0.42$ & $0.31$ & $0.51$ & $0.41$ & $0.45$  & $0.50$ \\
LSR-H & $0.11$ & $0.24$ & $0.24$ & $0.18$ & $0.21$ & $0.25$ & $0.25$  & $0.27$ \\

\bottomrule[1pt]
\end{tabular}
\end{table}

\begin{table}
\centering
\caption{Mean inter-cluster connectivity in $\%$ over $100$ independent trials for synthetic data randomly generated in three random subspaces of $\Re^{9}$ of dimensions $(d_1,d_2,d_3)$. There are $200$ points associated to each subspace, which are corrupted by zero-mean additive white noise of standard deviation $\sigma$ and support in the orthogonal complement of each subspace.}
\label{table:Synthetic_n3_inter}
\ra{0.7}
\begin{tabular}{@{}l@{\, \, \,}c@{\, \,}c@{\, \,}c@{\, \,}c@{\, \,}c@{\, \,}c@{\, \,}c@{\, \,}c}\toprule[1pt] method &  $(2,3,4)$ & $(4,5,6)$ & $(6,7,8)$ &  $(2,5,8)$  & $(3,3,3)$ &  $(6,6,6)$ & $(7,7,7)$ & $(8,8,8)$\\ 

\midrule[0.5pt]
& & & & $\sigma=0$\\
\cmidrule{5-5}
FSASC & $\boldsymbol{0}$ & $\boldsymbol{0}$ & $\boldsymbol{1}$ & $\boldsymbol{1}$ & $\boldsymbol{0}$ & $\boldsymbol{0}$ & $\boldsymbol{0}$ & $\boldsymbol{2}$   \\
SASC-D & $60$ & $60$ & $60$ & $60$ & $60$ & $60$ & $60$  & $60$ \\
SASC-A & $55$ & $55$ & $38$ & $43$ & $55$ & $50$ & $42$  & $35$ \\
SSC & $\boldsymbol{0}$ & $2$ & $22$ & $2$ & $\boldsymbol{0}$ & $7$ & $23$  & $46$ \\
LRR & $1$ & $49$ & $60$ & $45$ & $\boldsymbol{0}$ & $55$ & $60$  & $63$ \\
LRR-H & $\boldsymbol{0}$ & $18$ & $43$ & $9$ & $\boldsymbol{0}$ & $32$ & $44$  & $55$ \\
LRSC & $2$ & $49$ & $60$ & $45$ & $2$ & $55$ & $60$  & $63$ \\
LSR & $2$ & $49$ & $60$ & $43$ & $2$ & $56$ & $60$  & $64$ \\
LSR-H & $\boldsymbol{0}$ & $11$ & $24$ & $6$ & $\boldsymbol{0}$ & $19$ & $25$  & $30$ \\

\midrule[0.5pt]

& & & & $\sigma=0.01$\\
\cmidrule{5-5}
FSASC & $2$ & $4$ & $\boldsymbol{22}$ & $18$ & $2$ & $\boldsymbol{6}$ & $\boldsymbol{15}$ & $35$   \\
SASC-D & $62$ & $61$ & $60$ & $61$ & $62$ & $60$ & $60$  & $60$ \\
SASC-A & $63$ & $58$ & $46$ & $51$ & $64$ & $55$ & $47$  & $39$ \\
SSC & $\boldsymbol{0.1}$ & $\boldsymbol{1}$ & $23$ & $\boldsymbol{3} $ & $\boldsymbol{0.1}$ & $7$ & $23$  & $46$ \\
LRR & $17$ & $49$ & $60$ & $45$ & $16$ & $55$ & $60$  & $63$ \\
LRR-H & $1$ & $18$ & $43$ & $9$ & $1$ & $32$ & $44$  & $55$ \\
LRSC & $17$ & $49$ & $60$ & $45$ & $16$ & $55$ & $60$  & $63$ \\
LSR & $17$ & $49$ & $60$ & $46$ & $16$ & $55$ & $60$  & $64$ \\
LSR-H & $\boldsymbol{0.1}$ & $11$ & $24$ & $6$ & $\boldsymbol{0.1}$ & $19$ & $25$  & $\boldsymbol{30}$ \\

\bottomrule[1pt]
\end{tabular}
\end{table}

\begin{table}
\centering
\caption{Mean running time of each method in seconds over $100$ independent trials for synthetic data randomly generated in three random subspaces of $\Re^{9}$ of dimensions $(d_1,d_2,d_3)$. There are $200$ points associated to each subspace, which are corrupted by zero-mean additive white noise of standard deviation $\sigma=0.01$ and support in the orthogonal complement of each subspace. The reported running time is the time required to compute the affinity matrix, and it does not include the spectral clustering step. The experiment is run in MATLAB on a standard Macbook-Pro with a dual core 2.5GHz Processor and a total of $4$GB Cache memory.}
\label{table:Synthetic_n3_rt}
\ra{0.7}
\begin{tabular}{@{}l@{\, \, \,}r@{\, \,}r@{\, \,}r@{\, \,}r@{\, \,}r@{\, \,}r@{\, \,}r@{\, \,}r}\toprule[1pt] method &  $(2,3,4)$ & $(4,5,6)$ & $(6,7,8)$ &  $(2,5,8)$  & $(3,3,3)$ &  $(6,6,6)$ & $(7,7,7)$ & $(8,8,8)$\\ 
\midrule[0.5pt]
& & & & $\sigma=0.01$\\
\cmidrule{5-5}
FSASC & $13.57$ & $12.11$ & $8.34$ & $13.90$ & $13.69$ & $10.67$ & $8.55$ & $6.01$   \\
SASC-D & $0.03$ & $0.03$ & $0.03$ & $0.03$ & $0.03$ & $0.03$ & $0.03$  & $0.03$ \\
SASC-A & $0.03$ & $0.03$ & $0.03$ & $0.03$ & $0.03$ & $0.03$ & $0.03$  & $0.03$ \\
SSC & $5.01$ & $4.84$ & $5.06$ & $6.59 $ & $4.90$ & $4.71$ & $4.80$  & $5.03$ \\
LRR & $0.54$ & $0.36$ & $0.34$ & $0.45$ & $0.53$ & $0.34$ & $0.34$  & $0.34$ \\
LRR-H & $0.65$ & $0.48$ & $0.45$ & $0.61$ & $0.65$ & $0.46$ & $0.46$  & $0.45$ \\
LRSC & $\boldsymbol{0.01}$ & $\boldsymbol{0.01}$ & $\boldsymbol{0.01}$ & $\boldsymbol{0.01}$ & $\boldsymbol{0.01}$ & $\boldsymbol{0.01}$ & $\boldsymbol{0.01}$  & $\boldsymbol{0.01}$ \\
LSR & $0.05$ & $0.05$ & $0.05$ & $0.07$ & $0.05$ & $0.05$ & $0.05$  & $0.05$ \\
LSR-H & $0.25$ & $0.25$ & $0.24$ & $0.32$ & $0.24$ & $0.24$ & $0.24$  & $0.24$ \\

\bottomrule[1pt]
\end{tabular}
\end{table}

Using data as above, we compare the proposed methods FSASC (Algorithm \ref{alg:FSASC}) and SASC-D (section \ref{subsection:SASC-D}) to the state-of-the-art SASC-A (section \ref{subsection:SASC-A}) from algebraic subspace clustering methods, as well as to state-of-the-art \emph{self-expressiveness}-based methods, such as Sparse Subspace Clustering (SSC) \cite{Elhamifar:TPAMI13}, Low-Rank Representation (LRR) \cite{Liu:TPAMI13,Liu:ICML10}, Low-Rank Subspace Clustering (LRSC) \cite{Vidal:PRL14} and Least-Squares Regression subspace clustering (LSR) \cite{Lu:ECCV12}. 
For FSASC we use $L=10$ and $\gamma=0.1$. For SSC we use the Lasso version with $\alpha_z=20$, where $\alpha_z$ is defined above equation (14) in \cite{Elhamifar:TPAMI13}, and $\rho=0.7$, where $\rho$ is the thresholding parameter of the SSC affinity (see MATLAB function \texttt{thrC.m} provided by the authors of \cite{Elhamifar:TPAMI13}). For LRR we use the ADMM version provided by its first author with $\lambda=4$ in equation (7) of \cite{Liu:PAMI12}. For LRSC we use the ADMM method proposed by its authors with $\tau=420$ and $\alpha=4000$, where $\alpha$ and $\tau$ are defined at problem $(P)$ of page $2$ in \cite{Vidal:PRL14}. Finally, for LSR we use equation (16) in \cite{Lu:ECCV12} with $\lambda=0.0048$. For both LRR and LSR we also report results with the heuristic post-processing of the affinity matrix proposed by the first author of \cite{Liu:PAMI12} in their MATLAB function \texttt{lrr\_motion\_seg.m}; we denote these versions of LRR and LSR by LRR-H and LSR-H respectively.

Notice that all compared methods are spectral methods, i.e., they produce a pairwise affinity matrix $\C$ upon which spectral clustering is applied. To evaluate the quality of the produced affinity, besides reporting the standard subspace clustering error, which is the percentage of misclassified points, we also report the \emph{intra-cluster} and \emph{inter-cluster connectivities} of the affinity matrices $\C$. As an intra-cluster connectivity we use the minimum algebraic connectivity among the subgraphs corresponding to the ground truth clusters. The algebraic connectivity of a subgraph is the second smallest eigenvalue of its normalized Laplacian, and measures how well connected the graph is. In particular, values close to $1$ indicate that the subgraph is indeed well-connected (single connected component), while values close to $0$ indicate that the subgraph tends to split to at least two connected components. Clearly, from a clustering point of view, the latter situation is undesirable, since it may lead to over-segmentation.  Finally, as inter-cluster connectivity we use the percentage of the $\ell_1$-norm of the affinity matrix $\C$ that corresponds to erroneous connections, i.e., the quantity $\sum_{\x_j \in \S_i, \x_{j'} \in \S_{i'}, i \neq i'} |\C_{j,j'}|/||\C||_1$.
The smaller the inter-cluster connectivity is, the fewer erroneous connections the affinity contains. To summarize, a high-quality affinity matrix is characterized by high intra-cluster and low inter-cluster connectivity, which is then expected to lead to small spectral clustering error.

Tables \ref{table:Synthetic_n3_error}-\ref{table:Synthetic_n3_inter} show the clustering error, and the intra-cluster and inter-cluster connectivities associated with each method, averaged over $100$ independent experiments. Inspection of Table \ref{table:Synthetic_n3_error} reveals that, in the absence of noise ($\sigma=0$), FSASC gives exactly zero error across all dimension configurations. This is in agreement with the theoretical results of section \ref{section:mfAASC}, which guarantee that, in the absence of noise, the only points that survive the filtration associated with some reference point are precisely the points lying in the same subspace as the reference point. Indeed, notice that in Table \ref{table:Synthetic_n3_intra} and for $\sigma=0$ the connectivity attains its maximum value $1$, indicating that the subgraphs corresponding to the ground truth clusters are fully connected. Moreover in Table \ref{table:Synthetic_n3_inter} we see that for $\sigma=0$ the erroneous connections are either zero or negligible. This practically means that each point is connected to each and every other point from same subspace, while not connected to any other points, which is the ideal structure that an affinity matrix should have.

Remarkably, the proposed SASC-D, which is much simpler than FSASC, also gives zero error for zero noise. Table \ref{table:Synthetic_n3_intra} shows that SASC-D achieves perfect intra-cluster connectivity, while Table \ref{table:Synthetic_n3_inter} shows that 
the inter-cluster connectivity associated with SASC-D is very large. This is clearly an undesirable feature, which nevertheless seems not to be affecting the clustering error in this experiment, perhaps because the intra-cluster connectivity is very high. As we will see though later (section \ref{subsection:MotionSegmentation}), the situation is different for real data, for which SASC-D performs inferior to FSASC. 

Going back to Table \ref{table:Synthetic_n3_error} and $\sigma=0$, we see that the improvement in performance of the proposed FSASC and SASC-D over the existing SASC-A 
is dramatic: indeed, SASC-A succeeds only in the case of hyperplanes, i.e., when $d_1=d_2=d_3=8$. This is theoretically expected, since in the case of hyperplanes there is only one normal direction per subspace, and the gradient of the vanishing polynomial at a point in the hyperplane is guaranteed to recover this direction. However, when the subspaces have lower-dimensions, as is the case, e.g., for the dimension configuration $(4,5,6)$, then there are infinitely many orthogonal directions to each subspace. Hence a priori, the gradient of a vanishing polynomial may recover any such direction, and such directions could be dramatically different even for points in the same subspace (e.g., they could be orthogonal), thus leading to a clustering error of $39\%$.
 
As far as the rest of the self-expressiveness methods are concerned, Table \ref{table:Synthetic_n3_error} ($\sigma=0$) shows what we expect: the methods 
give a perfect clustering when the subspace dimensions are small, e.g., for dimension configurations $(2,3,4)$ and $(3,3,3)$, they start to degrade as the subspace dimensions increase ($(4,5,6)$, $(6,6,6)$), and eventually they fail when 
the subspace dimensions become large enough ($(6,7,8)$,$(7,7,7)$,$(8,8,8)$). To examine
the effect of the subspace dimension on the connectivity, let us consider SSC and the dimension configurations $(2,3,4)$ and $(2,5,8)$: Table \ref{table:Synthetic_n3_intra} ($\sigma=0$) shows
that for both of these configurations the intra-cluster connectivity has a small value of $10^{-3}$. This is expected, since SSC computes sparse affinities and it is known to produce weakly connected clusters. Now, Table \ref{table:Synthetic_n3_inter} ($\sigma=0$)
shows that the inter-cluster connectivity of SSC for $(2,3,4)$ is zero, i.e., there are no erroneous connections, and so, even though the intra-cluster connectivity is as small as $10^{-3}$, spectral clustering can still give a zero clustering error. On the other hand, for the case $(2,5,8)$ the inter-cluster connectivity is $2\%$, which, even though small, when coupled with the small intra-cluster connectivity of $10^{-3}$, leads to a spectral clustering error of $49\%$. Finally, notice that for the case of $(8,8,8)$ the intra-cluster connectivity is $10^{-7}$ and the inter-cluster connectivity is $46\%$, indicating that the quality of the produced affinity is very poor, thus explaining the corresponding clustering error of $55\%$.

\begin{table}[t!]
\centering
\caption{Mean subspace clustering error in $\%$ over $100$ independent trials for synthetic data randomly generated in four random subspaces of $\Re^{9}$ of dimensions $(8,8,5,3)$. There are $200$ points associated to each subspace, which are corrupted by zero-mean additive white noise of standard deviation $\sigma=0, 0.01, 0.03, 0.05$ and support in the orthogonal complement of each subspace.}
\label{table:Synthetic_n3_nv4}
\ra{0.7}
\begin{tabular}{@{}l@{\, \, \, \,}r@{\, \, \, \, }r@{\, \, \, \, }r@{\, \, \, \, }r@{\, \, \, \, }r@{\, \, \, \, }r@{\, \, \, \, }r@{\, \, \, \, }r}\toprule[1pt] method / $\sigma$  &  $0$ & $0.01$ & $0.03$ &  $0.05$  \\ 
\midrule[0.5pt]
FSASC & $\boldsymbol{0}$ & $\boldsymbol{2.19}$ & $\boldsymbol{5.08}$ & $\boldsymbol{7.65}$   \\
SASC-D & $22.88$ & $17.83$ & $15.93$ & $17.44$ \\
SASC-A & $22.88$ & $27.21$ & $31.43$ & $36.36$  \\
SSC & $64.39$ & $64.17$ & $64.36$ & $64.13 $ \\
LRR & $42.86$ & $42.88$ & $43.04$ & $42.91$  \\
LRR-H & $42.08$ & $42.06$ & $42.23$ & $42.21$  \\
LRSC & $42.85$ & $42.88$ & $43.05$ & $42.90$  \\
LSR & $42.84$ & $42.85$ & $43.00$ & $42.93$ \\
LSR-H & $38.72$ & $38.74$ & $38.96$ & $39.86$  \\

\bottomrule[1pt]
\end{tabular}
\end{table}

When the data are corrupted by noise ($\sigma=0.01,0.03,0.05$), the rest of the Tables 
\ref{table:Synthetic_n3_error}-\ref{table:Synthetic_n3_inter} show that FSASC is the best 
method, with the exception of the case of hyperplanes. In this latter case, i.e., when $d_1=d_2=d_3=8$, the best method is SASC-D with a clustering error of $6 \%$ when $\sigma=0.03$, as opposed to $10\%$ for FSASC. This is expected, since for the case of codimension-$1$ subspaces the length of each filtration should be precisely $1$, since in theory, the length of the filtration is equal to the codimension of the subspace associated to the reference point. Since FSASC automatically determines this length based on the data and the value of the parameter $\gamma$, it is expected that when the data are noisy, errors will be made in the estimation of the filtration length. On the other hand, SASC-D is equivalent to FSASC with an a priori configured filtration length equal to $1$, thus performing superior to FSASC. Certainly, giving as input to FSASC more than one values for $\gamma$, as shown in Algorithm \ref{alg:FSASC}, is expected to address this issue, but also increase the running time of FSASC (see Table \ref{table:Synthetic_n3_rt} for average running times of the methods in the current experiment).  

We conclude this section by demonstrating the interesting property of FSASC of being able to give the correct clustering by using vanishing polynomials of degree strictly less than the true number of subspaces. Towards that end, we consider a similar situation as above, except that now we have $n=4$ subspaces of dimensions $(8,8,5,3)$. Contrary to 
SASC-D and SASC-A, for which the theory requires degree-$4$ polynomials, 
FSASC is still applicable if one works with polynomials of degree $3$: the crucial observation is that for the dimension configuration $(8,8,5,3)$, the corresponding subspace arrangement always admits vanishing polynomials of degree $3$, and the same is true for every intermediate arrangement occurring in a filtration. For example, if one lets $\b_1$ be a normal vector to one of the $8$-dimensional subspaces, and $\b_2$ a normal vector to the other, and $\b_3$ a normal vector to the $8$-dimensional subspace spanned by both the $5$-dimensional and $3$-dimensional subspace, then the polynomial $p(x) = (\b_1^\transpose x) (\b_2^\transpose x) (\b_3^\transpose x)$ has degree $3$ and vanishes on the entire arrangement of the four subspaces. Interestingly, Table \ref{table:Synthetic_n3_nv4} shows that FSASC gives 
zero error in the absence of noise and $7.65\%$ error for the worst case $\sigma =0.05$, while all other methods fail. In particular, the other two algebraic methods, i.e., SASC-D and SASC-A, are not able to cluster the data using a single vanishing polynomial of degree $3$.

\subsection{Experiments on real motion sequences} \label{subsection:MotionSegmentation}

We evaluate different methods on the Hopkins155 motion segmentation data set \cite{Tron:CVPR07}, which contains 155 videos of $n=2,3$ moving objects, each one with $N=100$-$500$ feature point trajectories of dimension $D=56$-$80$. While SSC, LRR, LRSC and LSR can operate directly on the raw data, algebraic methods require $\mathcal{M}_n(D) \leq N$. Hence, for algebraic methods, we project the raw data onto the subspace spanned by their $D$ principal components, where $D$ is the largest integer $\le 8$ such that  $\mathcal{M}_n(D) \leq N$, and then normalize each point to have unit norm. We apply SSC to i) the raw data (SSC-raw) and ii) the raw points projected onto their first $8$ principal components and normalized to unit norm (SSC-proj). For FSASC we use $L=10$ and $\gamma=0.001,0.005,0.01,0.05,0.1, 0.5,1,5,10 $. LRR, LRSC and LSR use the same parameters as in section \ref{subsection:SyntheticExperiments}, while for SSC the parameters are $\alpha = 800$ and $\rho = 0.7$. 

The clustering errors and the intra/inter-cluster connectivities are reported in Table \ref{table:Hopkins155} and Fig. \ref{figure:Hopkins155-ordered-error}. Notice the clustering errors of about 5\% and 37\% for SASC-A for two and three motions respectively. Notice how changing the angle-based by the distance-based affinity, SASC-D already gives errors of around 5.5\% and 14\%. But most dramatically, notice how FSASC further reduces those errors to 0.8\% and 2.48\%. Moreover, even though the dimensions of the subspaces ($d_i \in \{1,2,3,4\}$ for motion segmentation) are low relative to the ambient space dimension ($D=56$-$80$) - a case that is specifically suited for SSC, LRR, LRSC, LSR - projecting the data to $D\leq 8$, which makes the subspace dimensions comparable to the ambient dimension, is sufficient for FSASC to get superior performance relative to the best performing algorithms on Hopkins 155. We believe that this is because, overall, FSASC produces a much higher intra-cluster connectivity, without increasing the inter-cluster connectivity too much.

\setlength{\tabcolsep}{0.2em}
\begin{table}
	\centering
	\ra{0.7}
	\caption{Mean clustering error ($E$) in $\%$, intra-cluster connectivity ($C_1$), and inter-cluster connectivity ($C_2$) in $\%$ for the Hopkins155 data set.} \label{table:Hopkins155}
	\begin{tabular}
		{@{}l@{\, \, \, }@{\, \,  }r@{\, \, }r@{\, \, }r@{\, \, }r@{\, \, }r@{\, \, }r@{\, \, }r@{\, \, }r@{\, \, }r@{\, \,}r@{\, \, }r@{\, \, }r@{\, \, }r}\toprule[1pt]
		\phantom{abc}&& \multicolumn{3}{c}{$2$ motions} & \phantom{ab}& \multicolumn{3}{c}{$3$ motions} & 
		\phantom{ab} & \multicolumn{3}{c}{all motions} & \\
		\cmidrule{3-5} \cmidrule{7-9} \cmidrule{11-13}
		method &&  $E$ &  $C_1$ & $C_2$ &&  $E$ &  $C_1$ & $C_2$ && $E$ &  $C_1$ & $C_2$  \\ \midrule
		FSASC && $\boldsymbol{0.80}$ &    $0.18$ & $4$    && $\boldsymbol{2.48}$ &  $0.10$ &    $10$     
		&&  $\boldsymbol{1.18}$ &  $0.16$ &   $5$         \\
		SASC-D   && $5.65$ &      $0.82$ & $26$ && $14.0$ &   $0.80$ & $46$  && $7.59$  &   $0.81$ &  $31$      \\
		SASC-A    && $4.99$ &       $0.35$ & $5$ && $36.8$ &  $0.09$ & $35$ && $12.2$ &   $0.29$ &  $12$      \\
		SSC-raw    && $1.53$ &     $0.05$ &  $2$   && $4.40$ &  $0.04$ &  $3$  &&  $2.18$ &   $0.05$ &  $2$       \\
	SSC-proj   &&  $5.87$ & $0.04$ & $3$ && $5.70$ &  $0.03$ & 3 &&   $5.83$ &   $0.03$ & $3$     \\
LRR         && $4.26$ &       $0.25$ & $19$   && $7.78$ &  $0.25$ & $28$   &&   $5.05$ &   $0.25$ &   $21$      \\
LRR-H  && $2.25$ &  $0.05$ & $2$ && $3.40$ &  $0.04$ & $3$    &&   $2.51$ &   $0.05$ & $2$ \\
		LRSC  && $3.38$ &  $0.25$ & $19$ && $7.42$ &  $0.24$ & $28$    &&   $4.29$ &   $0.25$ & $21$ \\
LSR         && $3.60$ & $0.24$ &  $18$  && $7.77$ &  $0.23$ &  $28$   &&  $ 4.54$ &    $0.23$ &   $21$     \\
LSR-H && $2.73$ &   $0.04$ &  $1$   && $2.60$ &  $0.03$ &  $2$   &&  $2.70$ &   $0.04$ &   $1$     \\
		\bottomrule[1pt]
	\end{tabular}
\end{table}

\begin{figure}[t!] 
	\centering
	\includegraphics[trim=20 0 50 50,clip,width=0.8\linewidth]{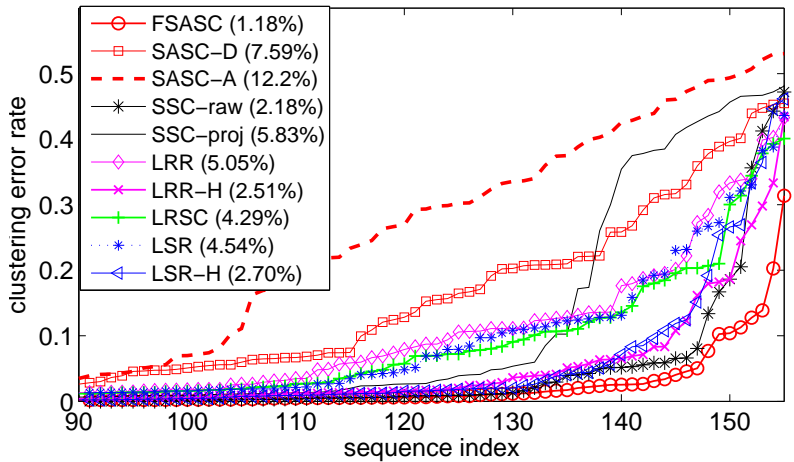}
	\caption{Clustering error ratios for both $2$ and $3$ motions in Hopkins155, ordered increasingly for each method. Errors start from the $90$-th smallest error of each method.}\label{figure:Hopkins155-ordered-error}
\end{figure}

\section{Conclusions and Future Research}
We presented a novel family of subspace clustering algorithms, termed 
\emph{Filtrated Algebraic Subspace Clustering} (FASC). The common theme
of these algorithms is the notion of a filtration of subspace arrangements. 
The first algorithm
of the family, termed \emph{Filtrated Algebraic Subspace Clustering}
(FASC) receives as input a finite point set in general position inside 
a subspace arrangement, together with an upper bound on the number of subspaces
in the arrangement. Then FASC provably returns the number of the subspaces, their dimensions, as 
well as a basis for the orthogonal complement of each subspace. The second algorithm
of the family, termed \emph{Filtrated Spectral Algebraic Subspace Clustering} (FSASC) is an adaptation of FASC to a working algorithm that is robust to noise.
In fact, by experiments on synthetic and real data we showed that FSASC is superior to state-of-the-art subspace clustering algorithms on several occasions. 

Due to the power of the machinery of filtrations, FSASC is unique among other
subspace clustering algorithms in that it can handle robustly subspaces of 
potentially very different dimensions, which can be arbitrarily close or far from the dimension of the ambient space. This is an important distinctive feature of FSASC from state-of-the-art Sparse and Low-Rank methods, which are in principle applicable only when the subspace dimensions are sufficiently small relative to the ambient dimension. However, this advantage of FSASC comes at the cost of a large computational complexity. Future research will address the problem of reducing this complexity with the aim of making FSASC applicable to large scale datasets.  Additional challenges to be undertaken include making FSASC robust to missing entries and outliers. 


\appendix



\section{Notions From Commutative Algebra} \label{appendix:CA}
A central concept in the theory of polynomial algebra is that of an \emph{ideal}:
\begin{definition}[Ideal]\label{dfn:ideal}
A subset $\I$ of the ring $\Re[x]:=\Re[x_1,\dots,x_D]$ of polynomials 
 is called an \emph{ideal} if for every $p,q \in \I$ and every $r \in \Re[x]$ we have that $p+q \in \I$ and $r p \in \I$. If $p_1,\hdots,p_n$ are elements of $\Re[x]$, then the \emph{ideal generated} by these elements is the set of all linear combinations of the $p_i$ with coefficients in $\Re[x]$.
\end{definition}

A polynomial $f \in \Re[x]$ is called homogeneous of degree $r$, if all the monomials that appear in $f$ have degree $r$. An ideal $\I$ is called homogeneous, if it is generated by homogeneous elements, i.e., $\I = \langle f_1,\dots,f_s\rangle$ where $f_i$ is a homogeneous polynomial of degree $r_i$. The reader can check that an ideal $\I$ is homogeneous if and only if $\I = \oplus_{k \ge 0} \I_k$, where $\I_k = \I \cap \Re[x]_k$. It is not hard to see that the intersection and the sum of two (homogeneous) ideals is a (homogeneous) ideal.
In performing algebraic operations with ideals it is also
useful to have a notion of product of ideals:
\begin{definition}[Product of ideals] Let $\I_1, \I_2$ be ideals of $\Re[x]$. The \emph{product} $\I_1 \I_2$ of $\I_1, \I_2$ is defined to be the set of all elements of the form $p_1 q_1 +\cdots + p_m q_m$ for any $m \in \mathbb{N}, p_i \in \I_1, q_i \in \I_2$. 
\end{definition}
\noindent The notion of a prime ideal is a natural generalization of the notion of a prime number. Prime ideals play a fundamental role in the study of the structure of general ideals, in analogy to the role that prime numbers have in the structure of integers.
\begin{definition}[Prime ideal]
An ideal $\mathfrak{p}$ of $\Re[x]$ is called \emph{prime}, if whenever $p q \in \mathfrak{p}$ for some $p,q \in \Re[x]$, then either $p \in \mathfrak{p}$ or $q \in \mathfrak{p}$. 
\end{definition}
We note that if $\mathfrak{p}$ is a homogeneous ideal, then in order to check whether $\mathfrak{p}$ is prime, it is enough
to consider $f,g$ homogeneous polynomials in the above definition.
\begin{proposition} \label{prp:ideals-intersection}
Let $\mathfrak{p}, {\I}_1, \dots, {\I}_n$ be ideals of $\Re[x]$ with $\mathfrak{p}$ being prime. If $\mathfrak{p} \supset {\I}_1 \cap \cdots \cap {\I}_n$, then $\mathfrak p \supset {\I}_i$ for some $i \in [n]$.
\end{proposition}
\begin{proof} 
Suppose $\mathfrak{p} \not\supset {\I}_i$ for all $i$. Then for every $i$ there exists $x_i \in {I}_i - \mathfrak{p}$. But then $\prod_{i=1}^s x_i \in \cap_{i=1}^s {\I}_i \subset \mathfrak{p}$ and since $\mathfrak{p}$ is prime, some $x_j \in \mathfrak{p}$, contradiction.
\end{proof}
A final notion that we need is that of a radical ideal:
\begin{definition}
	An ideal $\I$ of $\Re[x]$ is called \emph{radical}, if whenever some $p \in \Re[x]$ satisfies $p^\ell \in \I$ for some $\ell$, then it must be the case that $p \in \I$. 
\end{definition} 
Radical ideals have a very nice structure:
\begin{theorem} \label{thm:Radical}
	Every radical ideal $\I$ of $\Re[x]$ can be written uniquely as the finite intersection of prime ideals. Conversely, the intersection of a finite number of prime ideals is always a radical ideal.
\end{theorem}
For further information on commutative algebra we refer the reader to \cite{AtiyahMacDonald-1994} and
\cite{Eisenbud-2004} or to the more advanced treatment of \cite{Matsumura-2006}.


\section{Notions From Algebraic Geometry} \label{appendix:AG}
The central object of algebraic geometry is that of an \emph{algebraic variety}:
\begin{definition}[Algebraic variety]
A subset $\Y$ of $\Re^D$ is called an \emph{algebraic variety} or \emph{algebraic set} if it is the zero-locus of some ideal  $\mathfrak a$ of $\Re[x]$, i.e., $\Y = \left\{\y \in \Re^D : p(\y)=0, \, \forall p \in \mathfrak a \right\}$. A standard notation is to write $\Y = \Z (\mathfrak a)$ where the operator $\Z(\cdot)$ denotes \emph{zero set}.
\end{definition}
If $\Y = \Z (\mathfrak a)$ is an algebraic variety, then certainly every polynomial of $\mathfrak a$
vanishes on the entire $\Y$ (by definition). However, there may be more polynomials with that property,
 and they have a special name:
\begin{definition}[Vanishing ideal] The vanishing ideal of a subset $\Y$ of $\Re^D$, denoted $\I_{\Y}$, is the set of all polynomials of $\Re[x]$ that vanish on every point of $\Y$, i.e., 
$\I_\Y = \left\{p \in \Re[x]: p(\y)=0, \, \forall \y \in \Y \right\}$.
\end{definition}
It can be shown that the algebraic varieties induce a topology on $\Re^D$:
\begin{definition}[Zariski topology] \label{dfn:Zariski}
The \emph{Zariski Topology} on $\Re^D$ is the topology generated by defining the closed sets to be all the algebraic varieties. 
\end{definition}
Applying the definition of an \emph{irreducible topological space} in the context of the Zariski topology, we obtain:
\begin{definition}[Irreducible algebraic variety]
An algebraic variety $\Y$ is called irreducible if it can not be written as the union of two proper subsets of $\Y$ that are closed in the subspace topology of $\Y$.\footnote{We note that certain authors (e.g. \cite{Hartshorne-1977}) reserve the term \emph{algebraic variety} to refer to an \emph{irreducible closed set}.}
\end{definition}
The following Theorem is one of many interesting connections between geometry and algebra:
\begin{theorem}
An algebraic variety $\Y = \Z(\mathfrak a)$ is irreducible if and only if its vanishing ideal $\I_\Y$ is prime.
\end{theorem}
Perhaps not surprisingly, irreducible varieties are the fundamental building blocks of general varieties:
\begin{theorem}[Irreducible decomposition] \label{thm:varieties-decomposition}
Every algebraic variety $\Y$ of $\Re^D$ can be uniquely written as $\Y = \Y_1 \cup \cdots \cup \Y_n$, where $\Y_i$ are irreducible varieties and there are no inclusions $\Y_i \subset \Y_j$ for $i \neq j$. The varieties $\Y_i$ are referred to as the irreducible components of $\Y$.
\end{theorem}

\begin{proposition} \label{prp: varieties-inclusion}
If $\Y_1=\Z(\mathfrak a_1), \Y_2=\Z(\mathfrak a_2)$ are algebraic varieties such that $\mathfrak a_1 \subset \mathfrak a_2$, then $\Y_1 \supset \Y_2$.
\end{proposition}
\begin{theorem} \label{thm:InclusionVI}
It two subsets $\Y_1,\Y_2$ of $\Re^D$ satisfy the inclusion $\Y_1 \supset \Y_2$, then their vanishing ideals
will satisfy the reverse inclusion $\I_{\Y_1} \subset \I_{\Y_2}$.
\end{theorem}
\begin{proposition} \label{prp:VarietiesIntersection}
	Let $\Y_1=\Z(\mathfrak{a}_1),\Y_2=\Z(\mathfrak{a}_2)$ be varieties of $\Re^D$. Then $\Y_1 \cap \Y_2 = \Z(\mathfrak{a}_1+\mathfrak{a}_2)$. 
\end{proposition}
The final theorem that we present characterizes the set of all points that arise as the zero set of the vanishing ideal
of an arbitrary subset $\Y$ of $\Re^D$.
\begin{proposition} \label{prp:closure}
Let $\Y$ be a subset of $\Re^D$ and $\I_\Y$ its vanishing ideal. Then $\Z(\I_\Y) = Y^{cl}$, where $Y^{cl}$ is
the topological closure of $Y$ in the Zariski topology.
\end{proposition}
Finally, it should be noted that most of classic and modern algebraic geometry \cite{Hartshorne-1977} assume that 
the underlying algebraic field (in this paper $\Re$) is \emph{algebraically closed} \cite{Lang-2005}. An example of
an algebraically closed field is the complex numbers $\mathbb{C}$. Consequently, one should be careful when
using results such as \emph{Hilbert's Nullstellensatz} in real polynomial rings. 


\section{Subspace Arrangements and their Vanishing Ideals} \label{appendix:SA}
We begin by defining the main mathematical object of interest in this paper.
\begin{definition} [Subspace arrangement]
A union $\A=\bigcup_{i=1}^n \S_i$ of linear subspaces $\S_1,\hdots,\S_n$ of $\Re^D$, with $D \ge 1, n\ge1$ is called a subspace arrangement.
\end{definition}
It is often technically convenient to work with subspace arrangements that are as general as possible. One way to capture this notion is by the following definition.
\begin{definition}[Transversal subspace arrangement \cite{Derksen:JPAA07}]\label{dfn:transversalAppendix}
	A subspace arrangement $\A = \bigcup_{i=1}^n \S_i \subset \Re^D$ is called transversal, if for any subset $\mathfrak{I}$ of $[n]$, the codimension of $\bigcap_{i \in \mathfrak{I}} \S_i$ is the minimum between $D$ and the sum of the codimensions of all $\S_i, \, i \in \mathfrak{I}$, i.e., 
	\begin{align}
	\codim\left(\bigcap_{i \in \mathfrak{I}} \S_i\right) = \min \left\{D, \sum_{i\in \mathfrak{I}}c_i \right\},
	\end{align} where $c_i = \codim \S_i$.
\end{definition}
Transversality is a geometric condition on the subspaces $\S_1,\dots,\S_n$, that requires all possible intersections among the subspaces to be as small as possible, as allowed by the dimensions of the subspaces. To see this, let $\mathfrak{I}$ be a subset of $[n]$, which without loss of generality can be taken to be $\mathfrak{I}=\left\{1,2,\dots,\ell\right\}=[\ell]$, where $\ell \le n$. For every $i \in \mathfrak{I}$ let $\B_i$ be a $D \times c_i$ matrix, whose columns form a basis for $\S_i^\perp$, where $c_i = \codim S_i := D - \dim \S_i$, and let $\B = [\B_1 \dots \B_{\ell}]$. Then the intersection 
$\bigcap_{i \in \mathfrak{I}} \S_i$ can be described algebraically as 
\begin{align}
\x \in \bigcap_{i \in \mathfrak{I}} \S_i \Leftrightarrow \B^\transpose \x = 0. \label{eq:TransversalityAlgebraic}
\end{align} From \eqref{eq:TransversalityAlgebraic} it is clear that the dimension of $\bigcap_{i \in \mathfrak{I}} \S_i$ is equal to the dimension of the right nullspace of $\B$, or equivalently 
\begin{align}
\codim \left(\bigcap_{i \in \mathfrak{I}}\S_i\right)  = \rank(\B). \label{eq:CodimRankB}
\end{align}Now, $\B$ is a $D \times \left(\sum_{i \in \mathfrak{I}} c_i\right)$ matrix and so its rank will satisfy
\begin{align}
\rank(\B) \le \min\left\{D,\sum_{i \in \mathfrak{I}} c_i\right\}, \label{eq:RankB}
\end{align} which in conjunction with \eqref{eq:CodimRankB} justifies the geometric interpretation of Definition \ref{dfn:transversal}. 
In fact, if $\A$ is not transversal, then there exists some subset $\mathfrak{I} \subset [n]$, for which $\B$ is rank-deficient, which shows that certain algebraic relations must be satisfied among the parametrizations $\B_1,\dots,\B_{n}$ of the subspaces $\S_1,\dots,\S_n$. This is essentially the argument behind the proof of the next Proposition, which shows that transversality is not a strong condition, rather it will be satisfied almost surely.
\begin{proposition}
	Let $\A$ be a subspace arrangement consisting of $n$ linear subspaces of $\Re^D$ of dimensions $d_1,\dots,d_n$. If $\A$ is chosen uniformly at random, then $\A$ will be transversal with probability $1$.
\end{proposition} 
\begin{eg}
	An arrangement $\A = \S_1 \cup \S_2 \cup S_3 \subset \Re^D$ such that $\S_1 \subset \S_2$ is non-transversal, since 
	$\codim \S_1 \cap \S_2 = \codim \S_1  = c_1 < \min \left\{D,c_1+c_2\right\}$. Note that when choosing 
	$\S_1,\S_2,\S_3$ uniformly at random, the event $\S_1 \subset \S_2$ has probability zero. 
\end{eg} 
\begin{eg}
	An arrangement of three planes $\A = \H_1 \cup \H_2 \cup \H_3$ of $\Re^3$ that intersect on a line is non-transversal, because $\codim \H_1 \cap \H_2 \cap \H_3 = 2 < \min\left\{3,1+1+1\right\}$. When $\H_1,\H_2,\H_3$ are chosen uniformly at random, which is equivalent to choosing their normal vectors $\b_1,\b_2,\b_3$ uniformly at random, the three planes intersect on a line only if  $\b_1,\b_2,\b_3$ are linearly dependent, which is a probability zero event.
\end{eg}
Another notion of subspace arrangements in general position that is closely related to transversal arrangements, is that of \emph{linearly general} subspaces. 
\begin{definition}[Linearly general subspace arrangement \cite{Conca:CM03}] 
	A subspace arrangement $\A = \bigcup_{i=1}^n \S_i$ is called linearly general, if for every subset $\mathfrak{I} \subset [n]$ we have
	\begin{align}
	\dim \left(\sum_{i \in \mathfrak{I}} \S_i\right) = \min\left\{D,\sum_{i \in \mathfrak{I}} d_i \right\},
	\end{align} where $d_i = \dim \S_i$.
\end{definition} 
As the reader may suspect, the notion of transversal and linearly general are dual to each other in the following sense.
\begin{proposition}
	A subspace arrangement $\bigcup_{i=1}^n \S_i$ is transversal if and only if the subspace arrangement $\bigcup_{i=1}^n \S_i^\perp$ is linearly general.
\end{proposition}
\begin{proof}
	This follows by noting that with reference to the matrix $\B$ constructed below Definition \ref{dfn:transversal}, we have
	\begin{align}
	\codim \left(\bigcap_{i \in \mathfrak{I}} \S_i\right) = \rank(\B) = \dim \left(\sum_{i \in \mathfrak{I}} \S_i^\perp\right),
	\end{align} and that $\codim \S_i = \dim \S_i^\perp$.
\end{proof}	

In order to understand some important properties of subspace arrangements, it is necessary to examine the 
algebraic-geometric properties of a single subspace $\S$ of $\Re^D$ of dimension $d$. 
Let $\b_1,\dots,\b_c$ be a basis for the orthogonal complement of $\S$, where $c = D - d$ and define the 
polynomials $p_i(x) = \b_i^\transpose x, \, i=1,\dots,c$. Notice that $p_i(x)$ is homogeneous of degree $1$
and is thus also referred to as \emph{linear form}. If a point $\x$ belongs to $\S$, then $p_i(\x)=0, \, \forall i$.
Conversely, if a point $\x \in \Re^D$ satisfies $p_i(\x)=0, \, \forall i$, then $\x \in \S$. This shows that 
$\S=\Z(p_1,\dots,p_c)$, i.e., $\S$ is an algebraic variety. Notice that the set of linear forms that vanish on $\S$
is a vector space and the polynomials $p_i, \, i=1,\dots,c$ form a basis.

The Proposition that follows asserts that the vanishing ideal of $\S$, i.e., the set of all polynomials that vanish at every
point of $\S$, is in fact generated by the polynomials $p_i(x), \, i=1,\dots,c$.
\begin{proposition}[Vanishing Ideal of a Subspace]  \label{prp:VIS}
Let $\S=\Span (\b_1,\dots,\b_c)^{\perp}$ be a subspace of $\Re^D$ defined as the orthogonal complement of the space spanned by $\left\{\b_1,\hdots,\b_c \right\}$ over $\Re$. Then 
$\I_{\S}$ is generated over $\Re[x]$ by the linear forms $\b_1^{\transpose} x, \hdots, \b_c^{\transpose}x$. 
\end{proposition}
\begin{proof}
Let $\left\{\b_1,\hdots,\b_c \right\}$ be a basis for the orthogonal complement of $\S$ and augment it to a basis
$\left\{\b_1,\hdots,\b_c,\h_1,\dots,\h_{D-c} \right\}$ of $\Re^D$, where $\h_1,\dots,\h_{D-c}$ is a basis for $\S$. Now define a change of basis transformation
$\phi: \Re^D \rightarrow \Re^D$, which maps the basis $\left\{\b_1,\hdots,\b_c,\h_1,\dots,\h_{D-c} \right\}$
to the canonical basis $\left\{\e_1,\dots,\e_D\right\}$ of $\Re^D$, where $\e_i$ is the $i$-th column of the $D \times D$ identity matrix. Notice that $\b_i$ is mapped to $\e_i$ and as a consequence $\S$ is mapped to the orthogonal complement of
the vectors $\e_1,\dots,\e_c$. Since $\phi$ is a vector space isomorphism, we do not loose generality if we assume from
the beginning that $\S = \Span( \e_1,\dots, \e_c)^\perp=\Span( \e_{c+1},\dots, \e_D)$ and the vector space of linear forms that vanish on $\S$ is $x_1,\dots,x_c$. Notice that $\x \in \S$ if and only if the first $c$ coordinates of $\x$ are zero.

Now let $g \in \I_\S$.
We can write $g(x) = \bar{g}(x_{c+1},\hdots,x_D) + \sum_{i=1}^c x_i g_i(x)$. By hypothesis we have 
$g(0,\dots,0,a_{c+1}\hdots,a_D)=0$ for any real numbers $a_{c+1},\dots,a_D$, which implies that $\bar{g}(a_{c+1},\hdots,a_D)=0, \forall a_{c+1}, \dots, a_D \in \Re$. This in turn implies that $\bar{g}$ is the zero polynomial \footnote{We can prove by induction on $d$ that if $\mathbb{F}$ is an infinite field and $g(x_1,\hdots,x_d)=0, \forall x_1,\cdots,x_d \in \mathbb{F}$, then $g=0$.}. 
{\tiny }
Hence $g(x) =  \sum_{i=1}^c x_i g_i(x)$, which shows that $g$ is inside the ideal generated by the linear forms that vanish on $\S$.
\end{proof}

In algebraic-geometric notation, the above proposition can be concisely stated as $\I_{\Z(\b_1^\transpose x,\dots, \b_c^\transpose x)} = \langle\b_1^\transpose x,\dots, \b_c^\transpose x\rangle$. Interestingly, the vanishing ideal of a subspace
is a prime ideal:
\begin{proposition} \label{prp:ISprime}
Let $\S$ be a subspace of $\Re^D$. Then $\S$ is irreducible in the Zariski topology of $\Re^D$ or equivalently, $\I_\S$ is a prime ideal of $\Re[x]$.
\end{proposition}
\begin{proof}
As in the proof of Proposition \ref{prp:VIS} we can assume that $(x_1,\dots,x_c)$ is a basis for the linear forms of $\Re[x]$ that vanish on $\S$. Then $\I_\S = \langle x_1,\dots,x_c\rangle$ and our task is to show that $\I_\S$ is prime. So let $f,g$ be homogeneous polynomials such that $fg \in  \I_\S$ and suppose that $f \not\in \I_\S$. We will show that $g \in \I_\S$. We can write $f = f_1 + f_2$, where $f_1,f_2$ are polynomials such that $f_1 \in \I_\S$ and $f_2 \in \Re[x_{c+1},\dots,x_D]$. Similarly 
$g= g_1 + g_2$, with $g_1 \in \I_\S$ and $g_2 \in \Re[x_{c+1},\dots,x_D]$. Since by hypothesis $f \not\in \I_\S$, it must be the case that $f_2 \neq 0$. To show that $g \in \I_\S$, it is enough to show that $g_2=0$. 

Towards that end, notice that $fg = (f g_1 + f_1 g_2) + f_2 g_2$, where $f g_1 + f_1 g_2 \in \I_\S$. Since by hypothesis $fg \in \I_\S$, we also have that $f_2 g_2 \in \I_\S$. This means that 
there exist polynomials $h_1,\dots,h_c \in \Re[x_1,\dots,x_D]$, such that 
$f_2 g_2 = x_1 h_1 + \cdots +x_c h_c$. However, none of the variables $x_1,\dots,x_c$ appear on the left hand side of this equation, and so this equation is true only when
both sides are equal to zero. Since by hypothesis $f_2 \neq 0$, this implies that $g_2 = 0$, and so $g \in \I_\S$.

\emph{Alternative Proof}: A more direct proof exists if we assume familiarity of the reader with quotient rings. In particular, it is known that an ideal $\I$ of a commutative ring $R$ is prime if and only if the quotient ring $R/I$ has no zero-divisors \cite{AtiyahMacDonald-1994}. By noticing that $\Re[x_1,\dots,x_D] / \langle x_1,\dots,x_c\rangle \cong \Re[x_{c+1},\dots,x_D]$ we immediately see that $\langle x_1,\dots,x_c\rangle $ is prime. \end{proof}

Returning to the subspace arrangements, we see that a subspace arrangement $\A = \S_1 \cup \cdots \cup \S_n$ is the union of irreducible algebraic varieties $\S_1,\dots,\S_n$. This immediately suggests that the subspace arrangment itself is an algebraic variety. This was established
in \cite{Ma:SIAM08} via an alternative argument. Additionally, in view of Theorem \ref{thm:varieties-decomposition},  
the irreducible components of $\A$ are precisely its constituent subspaces $\S_1,\dots,\S_n$, which also proves
that a subspace arrangement can be uniquely written as the union of subspaces among which there are no inclusions.
We summarize these observations in the following theorem:
\begin{theorem} \label{thm:SA-Variety}
Let $\S_1,\dots,\S_n$ be subspaces of $\Re^D$ such that no inclusions exist between any two subspaces. Then the arrangement $\A = \S_1 \cup \cdots \cup \S_n$ is an algebraic variety and its irreducible components are $\S{\tiny }_1,\dots,\S_n$.
\end{theorem}

The vanishing ideal of a subspace arrangement $\A=\bigcup_{i=1}^n \S_i$ is readily seen to relate to the vanishing ideals of its irreducible components via the formula
\begin{align}
\I_{\A} = \I_{\S_1} \cap \cdots \cap \I_{\S_n} \label{e:VIA}.
\end{align} Since $\I_{\S_i}$ is a prime ideal, Theorem \ref{thm:Radical} implies that $\I_\A$ is radical and that $\A$ uniquely determines the ideals $\I_{\S_1}, \dots, \I_{\S_n}$, assuming that there are no inclusions between the subspaces. Hence, retrieving the irreducible components of a subspace arrangement is equivalent to computing the prime factors of its vanishing ideal $\I_\A$.

Since the ideal of a single subspace $\S_1$ is generated by linear forms, i.e., it is generated in degree $1$, one may be tempted to conjecture that the ideal $\I_{\A}$ of a union of $n$ subspaces is generated in degree less or equal than $n$. In fact, this is true:
\begin{proposition} \label{prp:Regularity}
	Let $\A$ be an arrangement of $n$ linear subspaces of $\Re^D$. Then its vanishing ideal $\I_{\A}$ is generated in degree $\le n$.
\end{proposition} 
\begin{proof}
By \cite{Sidman:AIM02} the \emph{Castelnuovo-Mumford regularity}\footnote{Please see \cite{Eisenbud-2004}, \cite{Conca:CM03}, \cite{Sidman:AIM02} or \cite{Derksen:JPAA07} for the definition of Castelnuovo-Mumford regularity.} of $\I_{\A}$ is bounded above by $n$. But by definition, the CM-regularity of an ideal bounds from above the maximal degree of a generator of the ideal.
\end{proof}

A crucial property of a subspace arrangement $\A$ in relation to the theory of Algebraic Subspace Clustering is that for any non-zero vanishing polynomial $p$ on $\A$, the orthogonal complement of the space spanned by the gradient of $p$ at some point 
$\x \in \A$ contains the subspace to which $\x$ belongs.
\begin{proposition} \label{prp:Grd}
Let $\A=\bigcup_{i=1}^n \S_i$ be a subspace arrangement of $\Re^D$, $p \in \I_{\A}$ and $\x \in \A$, 
 say $\x \in \S_i$ for some $i \in [n]$. Then $\nabla p|_{\x} \perp  \S_i$.
\end{proposition}
\begin{proof}
Take $p \in \I_{\A}$. From $\I_{\A} = \I_{\S_1} \cap \cdots \cap \I_{\S_n}$ we have that $\I_{\A} \subset \I_{\S_i}$. Hence $p \in \I_{\S_i}$. Now, from Proposition
\ref{prp:VIS} we know that $\I_{\S_i}$ is generated by a basis among all linear forms that vanish on $\S_i$, i.e., by a basis of $\I_{\S_i,1}$.
If $(\b_{i1},\dots,\b_{ic_i})$ is an  $\Re$-basis for $\S_i^{\perp}$ then $(\b_{i1}^{\transpose}x,\dots,\b_{ic_i}^{\transpose}x)$ is an $\Re$-basis for $\I_{\S_i,1}$ and a set of generators for $\I_{\S_i}$ over $\Re[x]$. Hence we can write 
$p(x) = \sum_{j=1}^{c_i} (\b_{i j}^{\transpose}x) g_{j}(x)$ where $g_{j}(x) \in \Re[x]$.
Taking the gradient of both sides of the above equation we get 
$\nabla p = \sum_{j=1}^{c_i} g_{j}(x) \b_{i j} + \sum_{j=1}^{c_i} (\b_{i j}^{\transpose}x) \nabla g_{j}$. Now let $\x \in \S_i$ be any point of $\S_i$. Evaluating both sides at $\x$ we have
$\nabla p|_{\x}  = \sum_{j=1}^{c_i} g_{j}(\x) \b_{i j} + \sum_{j=1}^{c_i} (\b_{i j}^{\transpose} \x) \nabla g_{j}|_{\x}.$ By hypothesis we have $\b_{ij}^{\transpose} \x=0, \, \forall j$ and so we obtain
$\nabla p|_{\x}  = \sum_{j=1}^{c_i} g_{j}(\x) \b_{i j} \in \S_i^{\perp}.$
\end{proof}

One may wonder when it is the case that the gradient of a vanishing polynomial
on a subspace arrangement $\A$ is zero at every point of $\A$. This is answered by

\begin{proposition} \label{prp:NablaZero}
	Let $\A = \bigcup_{i=1}^n \S_i$ be a subspace arrangement of $\Re^D$ and let $p \in \I_{\A}$. Then
	$\nabla p |_{\x} = 0,\, \forall \x \in \A$ if and only if $p \in \bigcap_{i=1}^n \I_{\S_i}^2$.
\end{proposition}
\begin{proof}
	$(\Rightarrow)$ Suppose  that $p \in \I_{\A}$, such that $\nabla p |_{\x} = 0, \forall \x \in \A$. Since $\I_{\A} \subset \I_{\S_i}, \, \forall i \in [n]$, by Proposition \ref{prp:VIS}, $p(x)$ can be written as
	\begin{align}
	p(x) = \sum_{j=1}^{c_i} g_{i,j}(x) (\b_{i,j}^{\transpose}x), \label{e:one}
	\end{align} where $c_i$ is the codimension of $\S_i$, $(\b_{i,1},\hdots,\b_{i,c_i})$ is a basis for $\S_i^{\perp}$ and $g_{i,j}(x)$ are polynomials. Now the hypothesis $\nabla p |_{\x} = 0, \forall \x \in \A$ implies that $\partial p / \partial x_k|_{\x}=0, \forall \x \in \A, \, \forall k \in [D]$. Thus $\partial p / \partial x_k \in \I_{\A}$ and so $\partial p / \partial x_k \in \I_{\S_i}$. Hence, again by Proposition \ref{prp:VIS}, $\partial p/ \partial x_k$ can be written as 
	\begin{align}
	\partial p / \partial x_k = \sum_{j=1}^{c_i} h_{i,j,k}(x) (\b_{i,j}^{\transpose}x) \label{e:two}.
	\end{align} Differentiating equation (\ref{e:one}) with respect to $x_k$ gives 
	\begin{align}
	\partial p / \partial x_k = \sum_{j=1}^{c_i} \left(\partial g_{i,j} / \partial x_k \right) (\b_{i,j}^{\transpose} x) + \sum_{j=1}^{c_i} g_{i,j}(x) \b_{i,j}(k) \label{e:three}.
	\end{align} From equations (\ref{e:two}), (\ref{e:three}) we obtain 
	\begin{align}
	\sum_{j=1}^{c_i} g_{i,j}(x) \b_{i,j}(k) = \sum_{j=1}^{c_i} \left(h_{i,j,k}(x)-\partial g_{i,j} / \partial x_k\right)  (\b_{i,j}^{\transpose}x)
	\end{align} which can equivalently be written as 
	\begin{align}
	\sum_{j=1}^{c_i} \b_{i,j}(k) g_{i,j}(x) =
	\sum_{j=1}^{c_i} q_{i,j,k}(x) (\b_{i,j}^{\transpose}x) \label{e:four}
	\end{align} where $q_{i,j,k}(x) := h_{i,j,k}(x)-\partial g_{i,j} / \partial x_k$. Note that equation (\ref{e:four}) is true for every $k \in [D]$. We can write these $D$ equations in matrix form
	\begin{align}
	\left[ \begin{array}{cccc}
		\b_{i,1} & \b_{i,2} & \cdots & \b_{i,c_i}
	\end{array}\right]
	\left[ \begin{array}{c}
		g_{i,1}(x) \\
		g_{i,2}(x) \\
		\vdots \\
		g_{i,c_i}(x)
	\end{array} \right] = 
	\Q(x) \left[ \begin{array}{c}
		\b_{i,1}^{\transpose}x \\
		\b_{i,2}^{\transpose}x \\
		\vdots \\
		\b_{i,c_i}^{\transpose}x
	\end{array} \right], \label{e:five}
	\end{align} where $\Q(x)$ is a $D \times c_i$ polynomial matrix with entries in $\Re[x]$. 
	We can view equation (\ref{e:five}) as a linear system of equations over the field $\Re(x)$.
	Define $\B_i: = \left[ \begin{array}{cccc}
	\b_{i,1} & \b_{i,2} & \cdots & \b_{i,c_i}.
	\end{array}\right]$ The columns of $\B_i$ form a basis of $\S_i^{\perp}$, and so they will be linearly independent over $\Re$. Consequently, the square matrix $\B_i^{\transpose} \B_i$ will be invertible over $\Re$ and its inverse will also be the inverse of $\B_i^{\transpose} \B_i$ over the larger field $\Re(x)$. Multiplying both sides of equation (\ref{e:five}) from the left with $(\B_i^{\transpose} \B_i)^{-1} \B_i^{\transpose}$, we obtain
	\begin{align}
	\left[ \begin{array}{c}
		g_{i,1}(x) \\
		g_{i,2}(x) \\
		\vdots \\
		g_{i,c_i}(x)
	\end{array} \right] = 
	(\B_i^{\transpose} \B_i)^{-1}
	\B_i^{\transpose} \Q(x) 
	\left[ \begin{array}{c}
		\b_{i,1}^{\transpose}x \\
		\b_{i,2}^{\transpose}x \\
		\vdots \\
		\b_{i,c_i}^{\transpose}x
	\end{array} \right]. \label{e:seven}
	\end{align} Note that $(\B_i^{\transpose} \B_i)^{-1}
	\B_i^{\transpose} \Q(x)  \in \left(\Re[x]\right)^{c_i \times c_i}$ and so equation (\ref{e:seven}) gives that $g_{i,j}(x) \in \I_{\S_i}, \, \forall j \in [c_i]$. Returning back to equation (\ref{e:one}), we readily see that $p \in \I_{\S_i}^2, \, \forall i \in [n]$, which implies that $p \in \cap_{i=1}^n \I_{\S_i}^2$.
	
	$(\Leftarrow)$ Suppose that $p \in \cap_{i=1}^n \I_{\S_i}^2$. Since 
	$\cap_{i=1}^n \I_{\S_i}^2 \subset \cap_{i=1}^n \I_{\S_i}=\I_{\A}$, we see that $p$ must be a vanishing polynomial. Since $p \in \I_{\S_i}^2$, by Proposition \ref{prp:VIS} we can write 
	$p(x) = \sum_{j,j'=1}^{c_i} g_{j,j'}(x) (\b_{i,j}^{\transpose}x)(\b_{i,j'}^{\transpose}x)$ from which it follows that $\nabla p|_{\x_i}=0, \, \forall \x_i \in \S_i$. Since this holds for any $i \in [n]$, we get that $\nabla p|_{\x}=0, \, \forall \x \in \A$.
	\end{proof}
We conclude with a theorem lying at the heart of Algebraic Subspace Clustering.   
\begin{theorem} \label{thm:I=J}
Let $\A = \bigcup_{i=1}^n \S_i$ be a transversal subspace arrangement of $\Re^D$
with vanishing ideal $\I_\A$. Let $\J_\A$ be the product ideal $\J_\A = \I_{\S_1}\cdots \I_{\S_n}$. Then the two ideals are equal at degrees $\ell \ge n$, i.e., $\I_{\A,\ell} = \J_{\A,\ell}, \forall \ell \ge n$.
\end{theorem}
Theorem \ref{thm:I=J} implies that every polynomial of degree $n$ that vanishes on a transversal subspace arrangement $\A$ of $n$ subspaces is a linear combination of products of linear forms vanishing on $\A$, a fundamental fact that is used repeatedly in the main text of the paper. Theorem \ref{thm:I=J} was first proved in Proposition 3.4 of \cite{Conca:CM03}, in the context of the \emph{Castelnuovo-Mumford regularity} of products of ideals generated by linear forms. It was later reproved in \cite{Derksen:JPAA07} using a Hilbert series argument and the result from \cite{Sidman:AIM02} on the Castelnuovo-Mumford regularity of a subspace arrangement.

\bibliography{FASC-ArXiv17.bbl}
\bibliographystyle{siamplain}

\end{document}